%% file: arxiv_latest.tex
\documentclass{article}
\usepackage[utf8]{inputenc}
\usepackage[english]{babel}
\usepackage[margin=1in]{geometry}
\usepackage{hyperref}
\usepackage{url}
\usepackage[square,numbers,sort&compress]{natbib}

\usepackage{microtype}
\usepackage{graphicx}
\usepackage{subfigure}
\usepackage{booktabs} 
\usepackage{tikz}
\usepackage{hyperref}

\usepackage[square,numbers,sort&compress]{natbib}

\usepackage{amsmath}
\usepackage{amssymb}
\usepackage{mathtools}
\usepackage{amsthm}
\usepackage{thmtools,thm-restate}
\usepackage{cancel}
\usepackage{titletoc}

\usepackage[capitalize,noabbrev]{cleveref}

\usepackage[textsize=tiny]{todonotes}

\include{macro}

    \title{Sequential Kernelized Independence Testing}
\author{Aleksandr Podkopaev$^{1}$\thanks{A large fraction of this work was completed while AP was an intern at Amazon in Summer 2022.}, Patrick Blöbaum$^2$, \\
    Shiva Prasad Kasiviswanathan$^2$, Aaditya Ramdas$^{1,2}$\\ 
	Carnegie Mellon University$^1$\\
	Amazon Web Services$^2$
}
\date{}

\begin{document}

\maketitle

\begin{abstract}
Independence testing is a classical statistical problem that has been extensively studied in the batch setting when one fixes the sample size before collecting data. However, practitioners often prefer procedures that adapt to the complexity of a problem at hand instead of setting sample size in advance. Ideally, such procedures should (a) stop earlier on easy tasks (and later on harder tasks), hence making better use of available resources, and (b) continuously monitor the data and efficiently incorporate statistical evidence after collecting new data, while controlling the false alarm rate. Classical batch tests are not tailored for streaming data: valid inference after data peeking requires correcting for multiple testing which results in low power. Following the principle of testing by betting, we design sequential kernelized independence tests that overcome such shortcomings. We exemplify our broad framework using bets inspired by kernelized dependence measures, e.g., the Hilbert-Schmidt independence criterion. Our test is also valid under non-i.i.d.\ time-varying settings. We demonstrate the power of our approaches on both simulated and real data.
\end{abstract}

\tableofcontents

\section{Introduction}\label{sec:intro}

Independence testing is a fundamental statistical problem that has also been studied within information theory and machine learning. Given paired observations $(X,Y)$ sampled from some (unknown) joint distribution $P_{XY}$, the goal is to test the null hypothesis that $X$ and $Y$ are independent. The literature on independence testing is vast as there is no unique way to measure dependence, and different measures give rise to different tests. Traditional measures of dependence, such as Pearson's $r$, Spearman's $\rho$, and Kendall's $\tau$, are limited to the case of univariate random variables. Kernel tests~\citep{bach2003kernel_cca,gretton2005kernel_coco,gretton2005measuring} are amongst the most prominent modern tools for nonparametric independence testing that work for general $\calX,\calY$ spaces. 

In the literature, heavy emphasis has been placed on \emph{batch} testing when one has access to a sample whose size is specified in advance. However, 
even if random variables are dependent, the sample size that suffices to detect dependence is never known a priori. If the results of a test are promising yet non-conclusive (e.g., a p-value is slightly larger than a chosen significance level), one may want to collect more data and re-conduct the study. This is not allowed by traditional batch tests. We focus on sequential tests that allow peeking at observed data to decide whether to stop and reject the null or to continue collecting data. 

\paragraph{Problem Setup.} Suppose that one observes a stream of data: $\roundbrack{Z_t}_{t\geq 1}$, where $Z_t=(X_t, Y_t)\simiid P_{XY}$. We design sequential tests for the following pair of hypotheses:
\begin{subequations}\label{eq:ind_testing}
\begin{align}
    H_0: & \ Z_t\simiid P_{XY}, \ t\geq 1 \ \text{and} \ P_{XY}=P_X\times P_Y, \label{eq:null_iid}\\
    H_1: & \ Z_t\simiid P_{XY}, \ t\geq 1 \ \text{and} \ P_{XY}\neq P_X\times P_Y.\label{eq:alt_iid}
\end{align}
\end{subequations}
Following the framework of ``tests of power one''~\citep{darling1968some}, we define a level-$\alpha$ sequential test  as a mapping $\Phi: \cup_{t=1}^\infty (\calX\times\calY)^t\to \curlybrack{0,1}$ that satisfies
\begin{equation*}
    \Prob_{H_0}\roundbrack{\exists t\geq 1:\Phi(Z_1,\dots,Z_t)=1} \leq \alpha.
\end{equation*} 
As is standard, $0$ stands for ``do not reject the null yet'' and $1$ stands for ``reject the null and stop''. Defining the stopping time $\tau := \inf\curlybrack{t\geq 1: \Phi(Z_1,\dots, Z_t)=1}$ as the first time that the test outputs 1, a sequential test must satisfy
\begin{align*}
    &\Prob_{H_0}\roundbrack{\tau<\infty}\leq \alpha. 
\end{align*}
We work in a very general composite nonparametric setting: $H_0$ and $H_1$ consist of huge classes of distributions (discrete/continuous) for which there may not be a common reference measure, making it impossible to define densities and thus ruling out likelihood-ratio based methods. 

\paragraph{Our Contributions.} Following the principle of testing by betting, we design consistent sequential nonparametric independence tests. Our bets are inspired by popular kernelized dependence measures: Hilbert-Schmidt independence criterion (HSIC)~\citep{gretton2005measuring}, constrained covariance criterion (COCO)~\citep{gretton2005kernel_coco}, and kernelized canonical correlation (KCC)~\citep{bach2003kernel_cca}. We provide theoretical guarantees on \emph{time-uniform} type I error control for these tests --- the type I error is controlled even if the test is continuously monitored and adaptively stopped --- and further establish consistency and asymptotic rates for our sequential HSIC under the i.i.d.\ setting. Our tests also remain valid even if the underlying distribution changes over time. Additionally, while the initial construction of our tests requires bounded kernels, we also develop variants based on symmetry-based betting that overcome this requirement. This strategy can be readily used with a linear kernel to construct a sequential linear correlation test. We justify the practicality of our tests through a detailed empirical study.

We start by highlighting two major shortcomings of existing tests that our new tests overcome.

\paragraph{(i) Limitations of Corrected Batch tests and Reduction to Two-sample Testing.} Batch tests (without corrections for multiple testing) have an inflated false alarm rate under continuous monitoring (see Appendix~\ref{appsubsec:batch_failure}). Na\"ive Bonferroni corrections restore type I error control, but generally result in tests with low power. This motivates a direct design of sequential tests (not by correcting batch tests). It is tempting to reduce sequential independence testing to sequential two-sample testing, for which a powerful solution has been recently designed~\citep{shekhar2022one_two_sample}. This can be achieved by splitting a single data stream into two and permuting the $X$ data in one of the streams (see Appendix~\ref{appsubsec:two_sample_reduction}). Still, the splitting results in inefficient use of data and thus low power, compared to our new direct approach (Figure~\ref{fig:intro_fig}).

\begin{figure}[!htb]
\centering
        \includegraphics[width=0.5\linewidth]{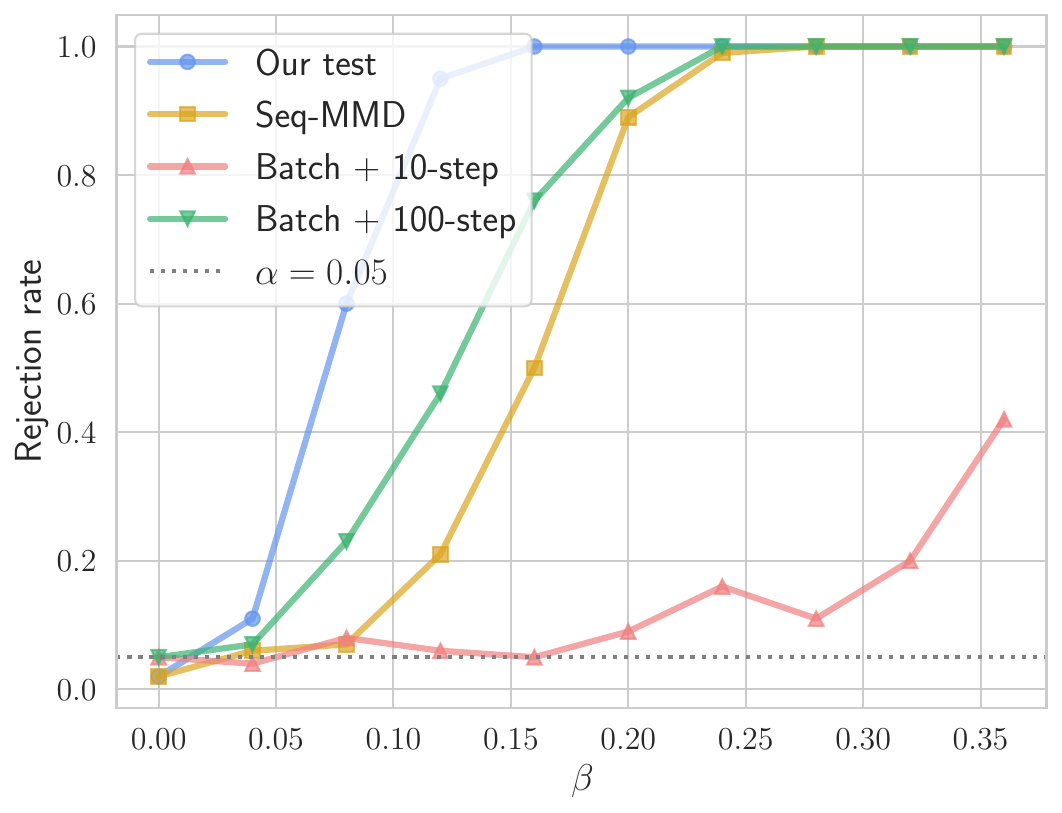}
        \caption{Valid sequential independence tests for: $Y_t=X_t\beta+\varepsilon_t$, $X_t,\varepsilon_t\sim \mathcal{N}(0,1)$. Batch + $n$-step is batch HSIC with Bonferroni correction applied every $n$ steps (allowing monitoring only at those steps). Seq-MMD refers to the reduction to two-sample testing (Appendix~\ref{appsubsec:two_sample_reduction}). Our test outperforms other tests.}
        \label{fig:intro_fig}
\end{figure}

\paragraph{(ii) Time-varying Independence Testing: Beyond the i.i.d.\ Setting.} A common practice of using a permutation p-value for batch independence testing requires $(X, Y)$-pairs to be i.i.d.\ (more generally, exchangeable). If data distribution drifts, the resulting test is no longer valid, and even under mild changes, an inflated false alarm rate is observed empirically. Our tests handle more general non-stationary settings. For a stream of independent data: $(Z_t)_{t\geq 1}$, where $Z_t\sim P_{XY}^{(t)}$, consider the following pair of hypotheses:
\begin{subequations}\label{eq:ind_test_generalization_indep}
\begin{align}
    H_0: & \ P_{XY}^{(t)}= P_{X}^{(t)}\times P_{Y}^{(t)}, \ \forall t, \label{eq:null_noniid_drift}\\
    H_1: & \ \exists t': P_{XY}^{(t')}\neq P_{X}^{(t')}\times P_{Y}^{(t')}. \label{eq:alt_noniid_drift}
\end{align}
\end{subequations}
Suppose that under $H_0$ in~\eqref{eq:null_noniid_drift}, it holds that either $P_{X}^{(t-1)}=P_{X}^{(t)}$ or $P_{Y}^{(t-1)}=P_{Y}^{(t)}$ for each $t\geq 1$, meaning that either the distribution of $X$ may change or that of $Y$ may change, but not both simultaneously. In this case, our tests control the type I error, whereas batch independence tests fail to. 

\begin{example}\label{ex:dist_drift_ex}
Let $\roundbrack{(W_{t}, V_{t})}_{t\geq 1}$ be a sequence of i.i.d.\ jointly Gaussian random variables with zero mean and covariance matrix with ones on the diagonal and $\rho$ off the diagonal. For $t=1,2,\dots$ and $i\in\curlybrack{0,1}$, consider the following stream:
\begin{equation}\label{eq:dist_drift_example}
\begin{cases}
    X_{2t-i} = 2c \sin(t)+W_{2t-1},\\
    Y_{2t-i} = 3 c\sin(t) + V_{2t-1},
\end{cases}
\end{equation}
Setting $\rho=0$ falls into the null case~\eqref{eq:null_noniid_drift}, whereas any $\rho\neq 0$ implies dependence as per~\eqref{eq:alt_noniid_drift}. Visually, it is hard to distinguish between $H_0$ and $H_1$: the drift makes data seem dependent (see Appendix~\ref{appsubsec:inst_dep}). In Figure~\ref{subfig:testing_distribution_drift_null}, we show that our test controls type I error, whereas batch test fails\footnote{This is also related to Yule's nonsense correlation~\citep{yule1926why,ernst2017yule}, which would not pose a problem for our method.}.
\end{example}

\begin{figure}[!htb]
\begin{center}
        \subfigure[The null is true, meaning that $W \indep V$ in Example~\ref{ex:dist_drift_ex}. (Batch) HSIC: dashed lines, SKIT: solid lines.]{
            \includegraphics[width=0.45\linewidth]{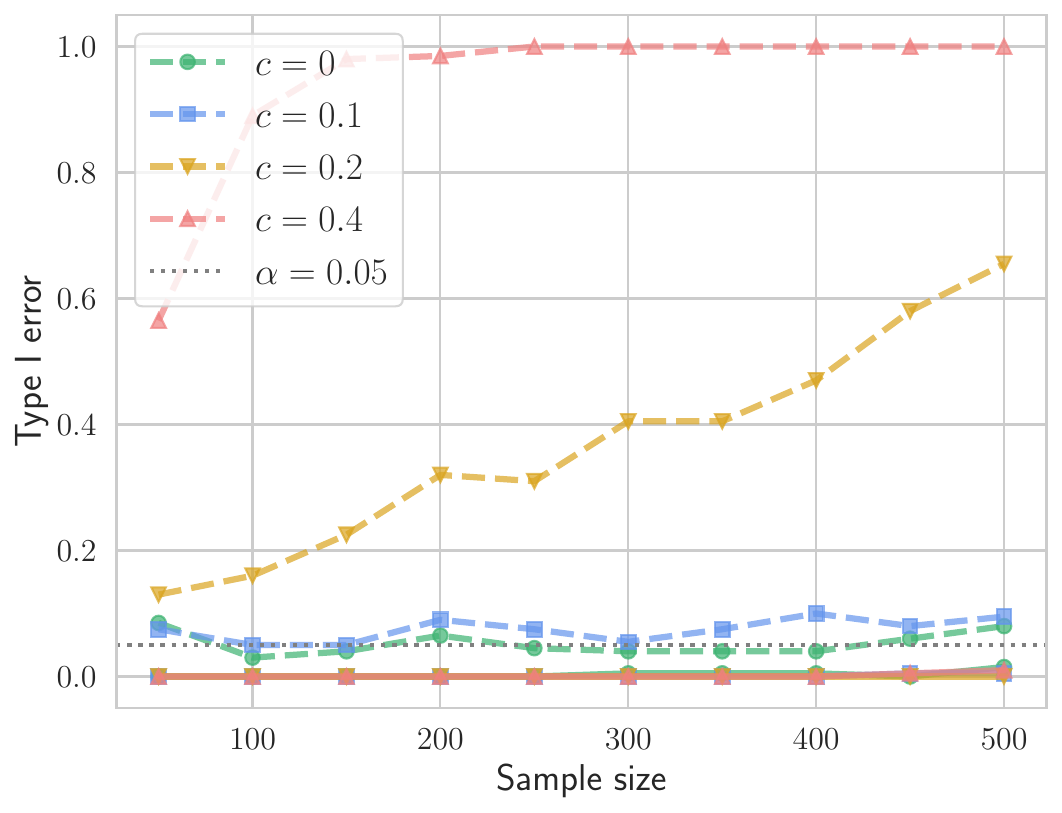}
            \label{subfig:testing_distribution_drift_null}}
        \subfigure[The alternative is true with $\rho=1/2$.]{
            \includegraphics[width=0.45\linewidth]{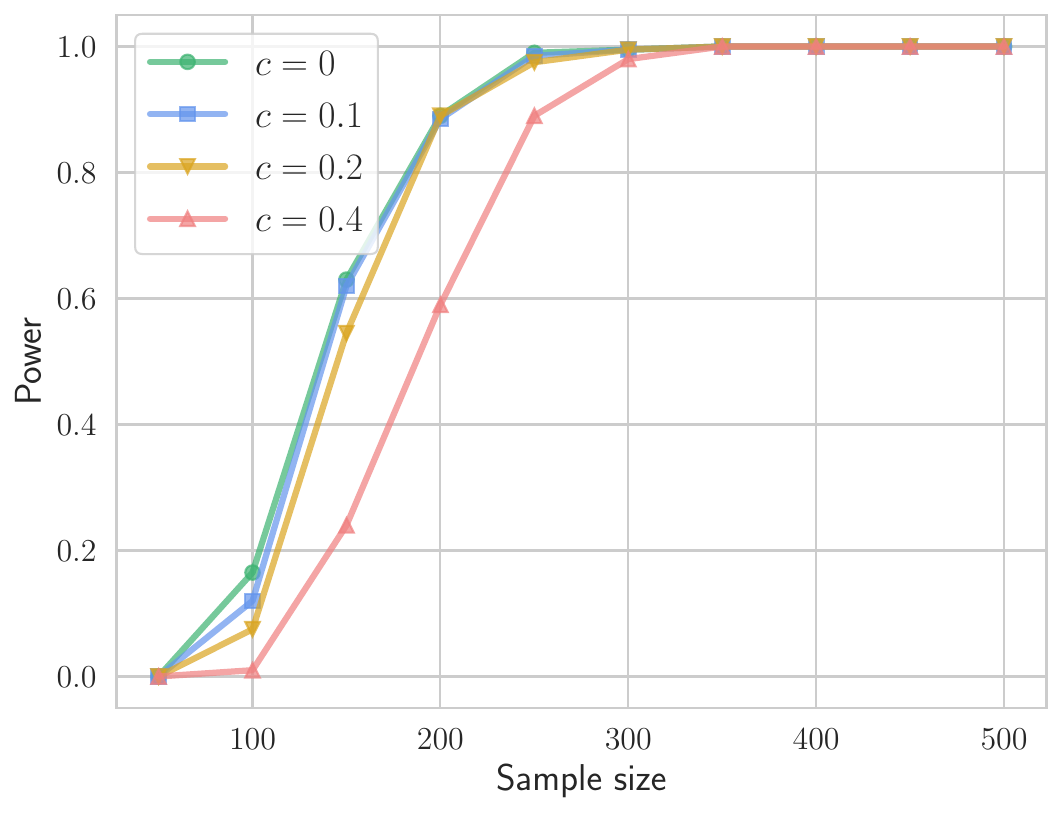}
            \label{subfig:testing_distribution_drift_alt}}
        \caption{Under distribution drift~\eqref{eq:dist_drift_example}, SKIT controls type I error under $H_0$ and has high power under $H_1$. Batch HSIC fails to control type I error under $H_0$ (hence we do not plot its power).}
        \label{fig:dist_drift_fig}
\end{center}
\end{figure}

\paragraph{Related Work.} In addition to the aforementioned papers on batch independence testing, our work is also related to methods for ``safe, anytime-valid inference'', e.g., confidence sequences~\citep[and references therein]{ws2020est_means} and e-processes~\citep{grunwald2020safe,ramdas2022testing_exch}. Sequential nonparametric two-sample tests of~\citet{balsubramani2016seq}, based on linear-time test statistics and empirical Bernstein inequality for random walks, are amongst the first results in this area. While such tests are valid in the same sense as ours, betting-based tests are much better empirically~\citep{shekhar2022one_two_sample}. 

The roots of the principle of testing by betting can be traced back to Ville's 1939 doctoral thesis~\citep{ville1939etude} and was recently popularized
by~\citet{shafer2021testing_by_betting}. The latter work considered it mainly in the context of parametric and simple hypotheses, far from our setting. The most closely related works to the current paper are~\citep{shekhar2022one_two_sample,shaer2022seq_crt,grunwald2022anytime} which also handle composite and nonparametric hypotheses. \citet{shekhar2022one_two_sample} use testing by betting to design sequential nonparametric two-sample tests, including a state-of-the-art sequential kernel maximum mean discrepancy test. Two recent works by \citet{shaer2022seq_crt,grunwald2022anytime}, developed in parallel to the current paper, extend these ideas to the setting of sequential conditional independence tests ($H_0: X\indep Y\mid Z$) under the model-X assumption, i.e., when the distribution $X\mid Z$ is assumed to be known. Our methods are very different from the aforementioned papers because when $Z=\emptyset$, the model-X assumption reduces to assuming $P_X$ is known, which we of course avoid. The current paper can be seen as extending the ideas from~\citep{shekhar2022one_two_sample} to nonparametric independence testing.


\section{Sequential Kernel Independence Test}\label{sec:seq_indep_testing}

We begin with a brief summary of the principle of testing by betting~\citep{shafer2021testing_by_betting, shafer2019game}. Suppose that one observes a sequence of random variables $\roundbrack{Z_t}_{t\geq 1}$, where $Z_t\in \calZ$. A player begins with initial capital $\calK_0=1$. At round $t$ of the game, she selects a payoff function $f_t:\calZ\to [-1,\infty)$ that satisfies $\Exp{Z\sim P_Z}{f_t(Z)\mid \calF_{t-1}}=0$ for all $P_Z\in H_0$, where $\calF_{t-1}=\sigma(Z_1,\dots,Z_{t-1})$, and bets a fraction of her wealth $\lambda_t\calK_{t-1}$ for an $\calF_{t-1}$-measurable $\lambda_t\in [0,1]$. Once $Z_t$ is revealed, her wealth is updated as
\begin{equation}\label{eq:wealth_update}
\calK_t 
= \calK_{t-1}(1+\lambda_t f_t(Z_t)).    
\end{equation}
A level-$\alpha$ sequential test is obtained using the following stopping rule: $\Phi(Z_1,\dots,Z_t)=\indicator{\calK_t\geq 1/\alpha}$, i.e., the null is rejected once the player's capital exceeds $1/\alpha$.
If the null is true, imposed constraints on sequences of payoffs $\roundbrack{f_t}_{t\geq 1}$ and betting fractions $\roundbrack{\lambda_t}_{t\geq 1}$ prevent the player from making money. Formally, the wealth process $\roundbrack{\calK_t}_{t\geq 0}$ is a nonnegative martingale. The validity of the resulting test then follows from Ville's inequality~\citep{ville1939etude}.

To ensure that the resulting test has power under the alternative, payoffs and betting fractions have to be chosen carefully. Inspired by sequential two-sample tests of~\citet{shekhar2022one_two_sample}, our construction relies on dependence measures: $m(P_{XY};\calC)$, which admit a variational representation:
\begin{equation}\label{eq:dep_var_rep}
    \sup_{c\in\calC} \squarebrack{\mathbb{E}_{P_{XY}}c(X,Y)-\mathbb{E}_{P_X\times P_Y}c(X,Y)},
\end{equation}
for some class $\calC$ of bounded functions $c:\calX\times \calY\to\Real$. 
The supremum above is often achieved at some $c^* \in \calC$, and in this case, $c^*$ is called the ``witness function''. In what follows, we use sufficiently rich functional classes $\calC$ for which the following characteristic condition holds:
\begin{equation}\label{eq:charac_cond}
    \begin{cases}
        m(P_{XY};\calC)=0,& \text{under $H_0$},\\
        m(P_{XY};\calC)>0,& \text{under $H_1$},
    \end{cases}
\end{equation}
for $H_0$ and $H_1$ defined in~\eqref{eq:ind_testing}. To proceed, we bet on \emph{pairs} of points from $P_{XY}$. Swapping $Y$-components in a pair of points from $P_{XY}$: $Z_{2t-1}$ and $Z_{2t}$, gives two points from $P_X\times P_Y$: $\tilde{Z}_{2t-1} = (X_{2t-1},Y_{2t})$ and $\tilde{Z}_{2t} = (X_{2t},Y_{2t-1})$. We consider payoffs $f(Z_{2t-1},Z_{2t})$ of the form:
\begin{equation}\label{eq:general_payoff}
    s\cdot\roundbrack{(c(Z_{2t-1})+c(Z_{2t})) - (c(\tilde{Z}_{2t-1})+c(\tilde{Z}_{2t}))},
\end{equation}
where the scaling factor $s>0$ ensures that $f(z,z')\in [-1,1]$ for any $z,z'\in\calX\times\calY$. When the witness function $c^*$ is used in the above, we denote the resulting function as the ``oracle payoff'' $f^*$. Let the oracle wealth process $\roundbrack{\calK^*_t}_{t\geq 0}$ be defined by using $f^*$ along with the betting fraction 
\begin{equation}\label{eq:oracle_closed_form}
    \lambda^\star = \frac{\Exp{}{f^\star(Z_{1},Z_{2})}}{\Exp{}{f^\star(Z_{1},Z_{2})}+\Exp{}{(f^\star(Z_{1},Z_{2}))^2}}.
\end{equation}
We have the following result regarding the above quantities, whose proof is presented in Appendix~\ref{appsubsubsec:sec_2_main}. 

\begin{restatable}{theorem}{thmoracletest}\label{thm:oracle_test}
Let $\calC$ denote a class of functions $c:\calX\times\calY\to\Real$ for measuring dependence as per~\eqref{eq:dep_var_rep}. 
\begin{enumerate}
    \item Under $H_0$ in~\eqref{eq:null_iid} and~\eqref{eq:null_noniid_drift}, any payoff $f$ of the form~\eqref{eq:general_payoff} satisfies $\Exp{H_0}{f(Z_{1},Z_{2})}=0$.
\item Suppose that $\calC$ satisfies~\eqref{eq:charac_cond}. Under $H_1$ in~\eqref{eq:alt_iid}, the oracle payoff $f^*$ based on the witness function $c^*$ satisfies $\Exp{H_1}{f^\star(Z_{1},Z_{2})}>0$. Further, for $\lambda^\star$ defined in~\eqref{eq:oracle_closed_form}, it holds that $\Exp{H_1}{\log(1+\lambda^\star f^\star(Z_{1},Z_{2})}>0$. Hence, $\calK_t^\star\convas+\infty$, which implies that the oracle test is consistent: $\Prob_{H_1}(\tau^\star<\infty)=1$,
where $\tau^\star = \inf\curlybrack{t\geq 1: \calK_t^\star\geq 1/\alpha}$.
\end{enumerate}
\end{restatable}

\begin{remark}\label{rmk:kelly_bet}
    While the betting fraction~\eqref{eq:oracle_closed_form} suffices to guarantee the consistency of the corresponding test, the fastest growth rate of the wealth process is ensured by considering
    \begin{equation*}
         \lambda^\star_{\mathrm{K}} \in \argmax_{\lambda\in (0,1)} \Exp{}{\log(1+\lambda f^\star(Z_{1},Z_{2})}.
    \end{equation*}
    \emph{Overshooting} with the betting fraction may, however, result in the wealth tending to zero almost surely.
\end{remark}

\begin{example}\label{ex:overshoot}
    Consider a sequence $\roundbrack{W_t}_{t\geq 1}$, where
\begin{equation*}
    W_t = \begin{cases}
        1, & \text{with probability $3/5$},\\
        -1, & \text{with probability $2/5$}.
    \end{cases}
\end{equation*}
In this case, we have $\lambda^\star_{\mathrm{K}} = 1/5$ and $\Exp{}{\log(1+\lambda^\star W_t)}>0$, implying that $\calK_t\convas+\infty$. On the other hand, it is easy to check that for $\tilde{\lambda}=2\lambda^\star_{\mathrm{K}}$ we have: $\mathbb{E}[\log(1+\tilde{\lambda} W_t)]<0$. As a consequence, for the wealth process $\calK_t$ corresponding to the pair $(f^*,\tilde{\lambda})$ it holds that $\calK_t\convas 0$.
\end{example}

To construct a practical test, we select an appropriate class $\calC$ for which the condition~\eqref{eq:charac_cond} holds and replace the oracle $f^\star$ and $\lambda^\star$ with predictable estimates $\roundbrack{f_t}_{t\geq 1}$ and $\roundbrack{\lambda_t}_{t\geq 1}$, meaning that those are computed using data observed prior to a given round of the game. We begin with a particular dependence measure, namely HSIC~\citep{gretton2005measuring}, and defer extensions to other measures to Section~\ref{sec:alt_measures}. 

\paragraph{HSIC-based Sequential Kernel Independence Test (SKIT).} Let $\calG$ be a separable RKHS\footnote{Recall that an RKHS is a Hilbert space $\calG$ of real-valued functions over $\calX$, for which the evaluation functional $\delta_x:\calG\to\Real$, which maps $g\in\calG$ to $g(x)$, is a continuous map, and this fact must hold for every $x\in\calX$. Each RKHS is associated with a unique positive-definite kernel $k: \calX\times\calX\to\Real$, which can be viewed as a generalized inner product on $\calX$. We refer the reader to~\citep{muander2017kernel} for an extensive recent survey of kernel methods.} with positive-definite kernel $k(\cdot,\cdot)$ and feature map $\varphi(\cdot)$ on $\calX$. Let $\calH$ be a separable RKHS with positive-definite kernel $l(\cdot,\cdot)$ and feature map $\psi(\cdot)$ on $\calY$. 

\begin{assumption}\label{assump:kernels}
Suppose that:
\begin{enumerate}[label=(A\arabic*),itemsep=0em]
    \item Kernels $k$ and $l$ are nonnegative and bounded by one: $\sup_{x\in\calX} k(x,x)\leq 1$ and $\sup_{y\in\calY} l(y,y)\leq 1$.
    \item The product kernel $k\otimes l:(\calX\times\calY)^2\to\Real$, defined as $(k\otimes l)((x,y),(x',y')):=k(x,x')l(y,y')$, 
    is a characteristic kernel on the joint domain.
\end{enumerate}
\end{assumption}
Assumption (A1) is used to justify that the mean embeddings introduced later are well-defined elements of RKHSs, and the particular bounds are used to simplify the constants. 
Assumption (A2) is a sufficient condition for the characteristic condition~\eqref{eq:charac_cond} to hold~\citep{fukumizu2007kernel_measures}, and we use it to argue about the consistency of our test. Under mild assumptions, it can be further relaxed to characteristic property of the kernels on the respective domains~\citep{gretton2015simpler_cond}. We note that the most common kernels on $\Real^d$: Gaussian (RBF) and Laplace, satisfy both (A1) and (A2). Define mean embeddings of the joint and marginal distributions:
\begin{equation}\label{eq:mean_emb_pop}
\begin{gathered}
    \mu_{XY} := \Exp{P_{XY}}{\varphi(X)\otimes \psi(Y)}, \\ 
    \mu_X := \Exp{P_X}{\varphi(X)},\quad \mu_Y := \Exp{P_Y}{\psi(Y)}.
\end{gathered}
\end{equation}
The cross-covariance operator $C_{XY}:\calH\to \calG$ associated with the joint measure $P_{XY}$ is defined as
\begin{equation*}\label{eq:cross_cov_def}
\begin{aligned}
    C_{XY} := \mu_{XY}  - \mu_X\otimes \mu_Y,
\end{aligned}
\end{equation*}
where $\otimes$ is the outer product operation. This operator generalizes the covariance matrix.
\emph{Hilbert-Schmidt independence criterion} (HSIC) is a criterion defined as Hilbert-Schmidt norm, a generalization of Frobenius norm for matrices, of the cross-covariance operator~\citep{gretton2005measuring}:
\begin{equation}\label{eq:hsic_def}
\begin{aligned}
    \HSIC(P_{XY};\calG,\calH)&:= \norm{\mathrm{HS}}{C_{XY}}^2.
\end{aligned}
\end{equation}
HSIC is the squared kernel maximum mean discrepancy (MMD) between mean embeddings of $P_{XY}$ and $P_X\times P_Y$ in the \emph{product RKHS} $\calG\otimes\calH$ on $\calX\times\calY$, defined by a product kernel $k\otimes l$. 
We can rewrite~\eqref{eq:hsic_def} as
\begin{equation*}
\begin{aligned}
\bigg(\sup_{\substack{g\in\calG\otimes\calH\\ \norm{\calG\otimes\calH}{g}\leq 1}}\Exp{P_{XY}}{g(X,Y)}-\Exp{P_X\times P_Y}{g(X',Y')}\bigg)^2,
\end{aligned}
\end{equation*}
which matches the form~\eqref{eq:dep_var_rep}. The witness function for HSIC admits a closed form (see Appendix~\ref{appsec:ommited_der}):
\begin{align}
    g^\star &= \frac{\mu_{XY}-\mu_X\otimes\mu_Y}{\norm{\calG\otimes\calH}{\mu_{XY}-\mu_X\otimes\mu_Y}},\label{eq:hsic_oracle_wf}
\end{align}
where $\mu_{XY}$, $\mu_{X}$ and $\mu_{Y}$ are defined in~\eqref{eq:mean_emb_pop}. The oracle payoff based on HSIC: $f^\star(Z_{2t-1},Z_{2t})$, is given by
\begin{equation}\label{eq:hsic_payoff_fn_oracle}
   \frac{1}{2}\roundbrack{g^\star(Z_{2t-1})+g^\star(Z_{2t})- g^\star(\tilde{Z}_{2t-1})-g^\star(\tilde{Z}_{2t})},
\end{equation}
which has the form~\eqref{eq:general_payoff} with $s=1/2$. To construct the test, we use estimators $\roundbrack{f_t}_{t\geq 1}$ of the oracle payoff $f^\star$ obtained by replacing $g^\star$ in~\eqref{eq:hsic_payoff_fn_oracle} with the plug-in estimator:
\begin{align}
    \hat{g}_t &= \frac{\hat{\mu}_{XY}-\hat{\mu}_X\otimes\hat{\mu}_Y}{\norm{\calG\otimes\calH}{\hat{\mu}_{XY}-\hat{\mu}_X\otimes\hat{\mu}_Y}},\label{eq:hsic_plugin_wf}
\end{align}
where $\hat{\mu}_{XY}, \hat{\mu}_{X},\hat{\mu}_{Y}$ denote the empirical mean embeddings, computed at round $t$ as\footnote{At round $t$, evaluating HSIC-based payoff requires a number of operations that is linear in $t$ (see Appendix~\ref{appsubsec:hsic_updates}). Thus after $T$ steps, we have expended a total of $O(T^2)$ computation, the same as batch HSIC. However, our test threshold is $1/\alpha$, but batch HSIC requires permutations to determine the right threshold, requiring recomputing HSIC hundreds of times. Thus, our test is actually more computationally feasible than batch HSIC.}
\begin{equation}\label{eq:emp_mean_emb}
    \begin{gathered}
        \hat{\mu}_{XY} = \tfrac{1}{2(t-1)} \sum_{i=1}^{2(t-1)} \varphi(X_{i})\otimes \psi(Y_{i}),\\
        \hat{\mu}_{X} = \tfrac{1}{2(t-1)} \sum_{i=1}^{2(t-1)} \varphi(X_{i}), \quad
        \hat{\mu}_{Y} = \tfrac{1}{2(t-1)} \sum_{i=1}^{2(t-1)}  \psi(Y_{i}).
    \end{gathered}
\end{equation}

Note that in~\eqref{eq:hsic_plugin_wf} the witness function is defined as an operator. We clarify this
point in Appendix~\ref{appsec:ommited_der}. To select betting fractions, we follow~\citet{cutkosky2018bb_reductions} who state the problem of choosing the optimal betting fraction for coin betting as an online optimization problem with exp-concave losses and propose a strategy based on online Newton step (ONS)~\citep{hazan2007log_regret} as a solution. ONS betting fractions are inexpensive to compute while being supported by strong theoretical guarantees. We also consider other strategies for selecting betting fractions and defer a detailed discussion to Appendix~\ref{appsec:betting_fractions}. 
We conclude with formal guarantees on time-uniform type I error control and consistency of HSIC-based SKIT. In fact, we establish a stronger result: we show that the wealth process grows exponentially and characterize the rate of the growth of wealth in terms of the true HSIC score. The proof is deferred to Appendix~\ref{appsubsubsec:sec_2_main}.

\begin{algorithm}[!htb]
\caption{Online Newton step (ONS) strategy for selecting betting fractions}\label{alg:ons}
\begin{algorithmic}
\State \textbf{Input:} sequence of payoffs $\roundbrack{f_t(Z_{2t-1},Z_{2t})}_{t\geq 1}$, $\lambda^{\mathrm{ONS}}_1=0$, $a_0=1$. 
\For {$t=1,2,\dots$}
\State Observe $f_t(Z_{2t-1},Z_{2t})$;
\State Set $z_t=f_t(Z_{2t-1},Z_{2t})/(1+\lambda_{t}^{\mathrm{ONS}}f_t(Z_{2t-1},Z_{2t}))$;
\State Set $a_t=a_{t-1}+z_t^2$;
\State Set $\lambda^{\mathrm{ONS}}_{t+1}:=\frac{1}{2}\wedge\roundbrack{0 \vee \roundbrack{\lambda^{\mathrm{ONS}}_t+\frac{2}{2-\log 3} \cdot \frac{z_t}{a_t}}}$;
\EndFor
\end{algorithmic}
\end{algorithm}

\begin{algorithm}[!htb]
\caption{HSIC-based SKIT}\label{alg:ind_test_framework}
\begin{algorithmic}
\State \textbf{Input:} significance level $\alpha\in (0,1)$, data stream $\roundbrack{Z_i}_{i\geq 1}$, where $Z_i=(X_i,Y_i)\sim P_{XY}$, $\lambda_1^{\mathrm{ONS}}=0$.
\For {$t=1,2,\dots$}
\State Use $Z_1,\dots, Z_{2(t-1)}$ to compute $\hat{g}_{t}$ as in~\eqref{eq:hsic_plugin_wf};
\State Compute HSIC payoff $f_t(Z_{2t-1}, Z_{2t})$;
\State Update the wealth process $\calK_t$ as in~\eqref{eq:wealth_update};
\If{$\calK_t\geq 1/\alpha$}
\State Reject $H_0$ and stop;
\Else 
\State Compute $\lambda_{t+1}^{\mathrm{ONS}}$ (Algorithm~\ref{alg:ons});
\EndIf
\EndFor
\end{algorithmic}
\end{algorithm}

\begin{restatable}{theorem}{thmhsicskitconsistency}\label{thm:hsic_skit}
Suppose that Assumption~\ref{assump:kernels} is satisfied. The following claims hold for HSIC-based SKIT (Algorithm~\ref{alg:ind_test_framework}):
\begin{enumerate}
    \item Under $H_0$ in~\eqref{eq:null_iid} or~\eqref{eq:null_noniid_drift}, SKIT ever stops with probability at most $\alpha$: $\Prob_{H_0}\roundbrack{\tau<\infty}\leq \alpha$.
\item Suppose that $H_1$ in~\eqref{eq:alt_iid} is true. Then it holds that $\calK_t\overset{\mathrm{a.s.}}{\longrightarrow}+\infty$, and thus SKIT is consistent: $\Prob_{H_1}(\tau<\infty) = 1$.
Further, the wealth grows exponentially, and the corresponding growth rate satisfies
    \begin{equation}\label{eq:log_wealth_lower_bound}
        \liminf_{t\to\infty} \tfrac{\log \calK_t}{t} \overset{\mathrm{a.s.}}{\geq} \tfrac{\Exp{}{f^\star(Z_{1},Z_{2})}}{4}\cdot \roundbrack{\tfrac{\Exp{}{f^\star(Z_{1},Z_{2})}}{\Exp{}{(f^\star(Z_{1},Z_{2}))^2}} \wedge 1},
    \end{equation}
    where $f^\star$ is the oracle payoff defined in~\eqref{eq:hsic_payoff_fn_oracle}.
\end{enumerate}
\end{restatable}

\noindent Since $\Exp{}{f^\star(Z_{1},Z_{2})}= \sqrt{\mathrm{HSIC}(P_{XY};\calG,\calH)}$ and $\Exp{}{(f^\star(Z_{1},Z_{2}))^2}\leq 1$, Theorem~\ref{thm:hsic_skit} implies that:
\begin{equation*}
    \liminf_{t\to\infty} \Big(\tfrac{1}{t}\log \calK_t\Big) \overset{\mathrm{a.s.}}{\geq} \tfrac{1}{4}\cdot \mathrm{HSIC}(P_{XY};\calG,\calH).
\end{equation*}
However, the lower bound~\eqref{eq:log_wealth_lower_bound} is never worse. Further, if the variance of the oracle payoffs: $\sigma^2 =\Var{}{f^\star(Z_1,Z_2)}$, is small, i.e., $\sigma^2\leq \Exp{}{f^\star(Z_{1},Z_{2})}(1-\Exp{}{f^\star(Z_{1},Z_{2})})$, we get a faster rate: $\sqrt{\mathrm{HSIC}(P_{XY};\calG,\calH)}/4$, reminiscent of an empirical Bernstein type adaptation. Up to some small constants, we show that this is the best possible exponent that adapts automatically between the low- and high-variance regimes. We do this by considering the oracle test, i.e., assuming that the oracle HSIC payoff $f^\star$ is known. Amongst the betting fractions that are constrained to lie in $[-0.5, 0.5]$, like ONS bets, the optimal growth rate is ensured by taking
\begin{equation}\label{eq:oracle_const_bet}
    \lambda^\star = \argmax_{\lambda\in [-0.5,0.5]}\Exp{}{\log(1+\lambda f^\star(Z_{1},Z_{2}))}.
\end{equation}
We have the following result about the growth rate of the oracle test, whose proof is deferred to Appendix~\ref{appsubsubsec:sec_2_main}. 

\begin{restatable}{proposition}{proporaclebounds}\label{prop:oracle_bounds}
    The optimal log-wealth $S^\star := \Exp{}{\log(1+\lambda^\star f^\star(Z_{1},Z_{2}))}$ --- that can be achieved by an oracle betting scheme~\eqref{eq:oracle_const_bet} which knows $f^\star$ from ~\eqref{eq:hsic_payoff_fn_oracle} and the underlying distribution --- satisfies:
    \begin{equation}\label{eq:oracle_lw_bounds}
        S^\star \leq  \frac{\Exp{}{f^\star(Z_{1},Z_{2})}}{2} \roundbrack{ \frac{8\Exp{}{f^\star(Z_{1},Z_{2})}}{3\Exp{}{(f^\star(Z_{1},Z_{2}))^2}}\wedge 1}.
    \end{equation}
\end{restatable}

\begin{remark}[Minibatching]
    While our test processes the data stream in pairs, it is possible to use larger batches of points from the joint distribution $P_{XY}$. For a batch size is $b\geq 2$, at round $t$, the bet is placed on $\curlybrack{(X_{b(t-1)+1},Y_{b(t-1)+1}),\dots, (X_{bt},Y_{bt})}$. In this case, the empirical mean embeddings are computed analogous to~\eqref{eq:emp_mean_emb} but using $\{(X_{i}, Y_{i})\}_{i\leq b(t-1)}$. We defer the details to Appendix~\ref{appsec:ommited_der}. Such payoff function satisfies the necessary conditions for the wealth process to be a nonnegative martingale, and hence, the resulting sequential test has time-uniform type I error control. The same argument as in the proof of Theorem~\ref{thm:hsic_skit} can be used to show that the resulting test is consistent. The main downside of minibatching is that monitoring of the test (and hence, optional stopping) is allowed only after processing every $b$ points from $P_{XY}$. 
\end{remark}

\paragraph{Distribution Drift.} As discussed in Section~\ref{sec:intro}, batch independence tests have an inflated false alarm rate even under mild changes in distribution. In contrast, SKIT remains valid even when the data distribution drifts over time. For a stream of independent points, we claimed that our test controls the type I error as long as only one of the marginal distributions changes at each round. In Appendix~\ref{appsec:ommited_der}, we provide an example that shows that this assumption is necessary for the validity of our tests.
Our tests can also be used to test instantaneous independence between two streams. Formally, define $\calD_t:= \curlybrack{(X_i, Y_i)}_{i\leq 2t}$ and consider:
\begin{subequations}\label{eq:ind_test_generalization}
\begin{align}
     H_0:& \ \forall t, \ X_{2t-1} \indep Y_{2t-1} \mid \calD_{t-1} \text{ and } X_{2t} \indep Y_{2t}\mid \calD_{t-1}, \label{eq:null_noniid}\\
     H_1:& \ \exists t':  X_{2t'-1} \ \cancel{\indep} \ Y_{2t'-1} \mid \calD_{t-1} \text{ or } X_{2t'} \ \cancel{\indep} \ Y_{2t'}\mid \calD_{t-1}. \label{eq:alt_noniid}
\end{align}
\end{subequations}

\begin{assumption}\label{assump:noniid_setting}
    Suppose that under $H_0$ in~\eqref{eq:null_noniid}, it also holds that:
    \begin{enumerate}[label=(\alph*),itemsep=0em]
    \item The cross-links between $X$ and $Y$ streams are not allowed, meaning that for all $t\geq 1$,
    \begin{equation}\label{eq:cross_links}
        \begin{aligned}
        Y_{t}\indep X_{t-1} & \mid Y_{t-1}, \{(X_i, Y_i)\}_{i\leq t-2},\\
        X_{t}\indep Y_{t-1} & \mid X_{t-1},\{(X_i, Y_i)\}_{i\leq t-2}.
        \end{aligned}
    \end{equation}
    \item For all $t\geq 1$, either $(X_t, X_{t-1})$ or $(Y_t, Y_{t-1})$ are exchangeable conditional on $\curlybrack{(X_i, Y_i)}_{i\leq t-2}$. 
\end{enumerate}
\end{assumption}

In the above, (a) relaxes the independence assumption within each pair, and (b) generalizes the assumption about allowed changes in the marginal distributions of $X$ and $Y$. Under the above setting, we deduce that our test retains type-1 error control, and the proof is deferred to Appendix~\ref{appsubsubsec:sec_2_main}. 

\begin{restatable}{theorem}{thmnoniid}\label{thm:non_iid}
Suppose that $H_0$ in~\eqref{eq:null_noniid} is true. Further, assume that Assumption~\ref{assump:noniid_setting} holds. Then HSIC-based SKIT (Algorithm~\ref{alg:ind_test_framework}) satisfies: $\Prob_{H_0}\roundbrack{\tau<\infty}\leq \alpha$.
\end{restatable}
\citet{chwialkowski2014kernel_ind_test} considered a related (at a high level) problem of testing instantaneous independence between a pair of time series. Similar to distribution drift, HSIC fails to test independence between innovations in time series since naively permuting one series destroys the underlying structure. \citet{chwialkowski2014kernel_ind_test} used a subset of permutations --- rotations by circular shifting (allowed by their assumption of strict stationarity) of one series for preserving the structure --- to design a p-value and used the assumption of mixing (decreasing memory of a process) to justify the asymptotic validity. The setting we consider is very different, and we make no assumptions of mixing or stationarity anywhere. Related works on independence testing for time series also include~\citep{chwialkowski2014wild,besserve2013stat_analysis}. 

\section{Alternative Dependence Measures}\label{sec:alt_measures}

Let $\calC_1$ and $\calC_2$ denote some classes of bounded functions $c_1:\calX\to\Real$ and $c_2:\calY\to\Real$ respectively. For a class $\calC$ of functions $c:\calX\times\calY\to\Real$ that factorize into the product: $c(x,y) = c_1(x)c_2(y)$ for some $c_1\in\calC_1$ and $c_2\in\calC_2$, the general form of dependence measures~\eqref{eq:dep_var_rep} reduces to
\begin{equation*}
\begin{aligned}
    m(P_{XY};\calC_1,\calC_2)=\sup_{\substack{c_1\in\calC_1, c_2\in\calC_2}} \Cov{c_1(X),c_2(Y)}.
\end{aligned}
\end{equation*}
Next, we develop SKITs based on two dependence measures of this form: COCO and KCC. While the corresponding witness functions do not admit a closed form, efficient algorithms for computing the plug-in estimates are available.

\paragraph{Witness Functions for COCO.}
\emph{Constrained covariance} (COCO) is a criterion for measuring dependence based on covariance between smooth functions of random variables:
\begin{equation}\label{eq:coco_cross_cov}
\begin{aligned}
    \sup_{\substack{g,h: \\ \norm{\calG}{g}\leq 1, \norm{\calH}{h}\leq 1} } \Cov{g(X),h(Y)} =\sup_{\substack{g,h: \\ \norm{\calG}{g}\leq 1,  \norm{\calH}{h}\leq 1} } \langle g, C_{XY}h\rangle_{\calG},
\end{aligned}
\end{equation}
where the supremum is taken over the unit balls in the respective RKHSs~\citep{gretton2005kernel_coco,gretton05a}. At round $t$, we are interested in empirical witness functions computed from $\curlybrack{(X_i, Y_i)}_{i\leq 2(t-1)}$. The key observation is that maximizing the objective function in~\eqref{eq:coco_cross_cov} with the plug-in estimator of the cross-covariance operator requires considering only functions in $\calG$ and $\calH$ that lie in the span of the data:
\begin{equation}\label{eq:coco_witness_fns}
    \begin{aligned}
        \hat{g}_t &= \sum_{i=1}^{2(t-1)} \alpha_i \bigg(\varphi(X_i)-\frac{1}{2(t-1)}\sum_{j=1}^{2(t-1)} \varphi(X_j)\bigg),\\ \hat{h}_t &= \sum_{i=1}^{2(t-1)} \beta_i \bigg(\psi(Y_i)-\frac{1}{2(t-1)}\sum_{j=1}^{2(t-1)} \psi(Y_j)\bigg).
    \end{aligned}
\end{equation}
Coefficients $\alpha$ and $\beta$ that solve the maximization problem~\eqref{eq:coco_cross_cov} define the leading eigenvector of the following generalized eigenvalue problem (see Appendix~\ref{appsec:ommited_der}):
\begin{align}
     \begin{pmatrix}
    0 & \frac{1}{2(t-1)}\tilde{K}\tilde{L} \\ 
    \frac{1}{2(t-1)}\tilde{L}\tilde{K} & 0
    \end{pmatrix}
    \begin{pmatrix}
    \alpha \\ \beta
    \end{pmatrix} &= \gamma \begin{pmatrix}
    \tilde{K} & 0 \\ 
    0 & \tilde{L}
    \end{pmatrix} \begin{pmatrix}
    \alpha \\ \beta
    \end{pmatrix},\label{eq:gen_eig_coco}
\end{align}
where $\tilde{K}=HKH$, $\tilde{L}=HLH$, and $H=\textbf{I}_{2(t-1)}-(1/(2(t-1))\textbf{1}\textbf{1}^\top$ is centering projection matrix. Computing the leading eigenvector for~\eqref{eq:gen_eig_coco} is computationally demanding for moderately large $t$. A common practice is to use low-rank approximations of $K$ and $L$ with fast-decaying spectrum~\citep{bach2003kernel_cca}. We present an approach based on incomplete Cholesky decomposition in Appendix~\ref{appsubsec:coco_updates}. 

\paragraph{Witness Functions for KCC.}
\emph{Kernelized canonical correlation} (KCC) relies on the regularized correlation between smooth functions of random variables:
\begin{equation}\label{eq:kcc_def}
\begin{aligned}
    \sup_{\substack{g\in \calG, \\ h \in \calH}} \frac{\Cov{g(X),h(Y)}}{\sqrt{\Var{}{g(X)}+\kappa_1\norm{\calG}{g}^2}\cdot \sqrt{\Var{}{h(Y)}+\kappa_2\norm{\calH}{h}^2}},
\end{aligned}
\end{equation}
where regularization is necessary for obtaining meaningful estimates of correlation~\citep{bach2003kernel_cca,fukumizu2007stat_consistency}. Witness functions for KCC have the same form as for COCO~\eqref{eq:coco_witness_fns}, but $\alpha$ and $\beta$ define the leading eigenvector of a modified problem (Appendix~\ref{appsec:ommited_der}).

\paragraph{SKIT based on COCO or KCC.} Given a pair of the witness functions $g^\star$ and $h^\star$ for COCO (or KCC) criterion, the corresponding oracle payoff: $f^\star (Z_{2t-1},Z_{2t})$, is given by
\begin{equation}\label{eq:coco_payoff_fn_oracle}
\begin{aligned}
     \frac{1}{2}\roundbrack{g^\star(X_{2t})-g^\star(X_{2t-1})}\roundbrack{h^\star(Y_{2t})-h^\star(Y_{2t-1})},
\end{aligned}
\end{equation}
which has the form~\eqref{eq:general_payoff} with $s=1/2$. To construct the test, we rely on estimates $(f_t)_{t\geq 1}$ of the oracle payoff $f^\star$ obtained by using $\hat{g}_t$ and $\hat{h}_t$, defined in~\eqref{eq:coco_witness_fns}, in~\eqref{eq:coco_payoff_fn_oracle}. We assume that $\alpha$ and $\beta$ in~\eqref{eq:gen_eig_coco} are normalized: $\alpha^\top\tilde{K}\alpha=1$ and $\beta^\top\tilde{L}\beta=1$. We conclude with a guarantee on time-uniform false alarm rate control of SKITs based on COCO (Algorithm~\ref{alg:ind_test_coco_kcc}), whose proof is deferred to Appendix~\ref{appsubsec:proofs_sec_3}.

\begin{algorithm}[!htb]
\caption{SKIT based on COCO (or KCC)}\label{alg:ind_test_coco_kcc}
\begin{algorithmic}
\State \textbf{Input:} significance level $\alpha\in (0,1)$, data stream $(Z_i)_{i\geq 1}$, where $Z_i=(X_i,Y_i)\sim P_{XY}$, $\lambda_1^{\mathrm{ONS}}=0$. 
\For {$t=1,2,\dots$}
\State Use $Z_1,\dots, Z_{2(t-1)}$ to compute $\hat{g}_{t}$ and $\hat{h}_{t}$ as in~\eqref{eq:coco_witness_fns};
\State Compute COCO payoff $f_t(Z_{2t-1}, Z_{2t})$;
\State Update the wealth process $\calK_t$ as in~\eqref{eq:wealth_update};
\If{$\calK_t\geq 1/\alpha$}
\State Reject $H_0$ and stop;
\Else 
\State Compute $\lambda_{t+1}^{\mathrm{ONS}}$ (Algorithm~\ref{alg:ons});
\EndIf
\EndFor
\end{algorithmic}
\end{algorithm}

\begin{restatable}{theorem}{thmvaliditycocokcc}\label{thm:validity_coco_kcc}
Suppose that (A1) in Assumption~\ref{assump:kernels} is satisfied. Then, under $H_0$ in~\eqref{eq:null_iid} and~\eqref{eq:null_noniid}, COCO/KCC-based SKIT (Algorithm~\ref{alg:ind_test_coco_kcc}) satisfies: $\Prob_{H_0}\roundbrack{\tau<\infty}\leq \alpha$.
\end{restatable}

\begin{remark} The above result does not contain a claim regarding the consistency of the corresponding tests. If (A2) in Assumption~\ref{assump:kernels} holds, the same argument as in the proof of Theorem~\ref{thm:hsic_skit} can be used to deduce that SKITs based on the \emph{oracle} payoffs (with oracle witness functions $g^\star$ and $h^\star$) are consistent. In contrast to HSIC, for which the oracle witness function is closed-form and the respective plug-in estimator is amenable for the analysis, to argue about the consistency of SKITs based on COCO/KCC, it is necessary to place additional assumptions, especially since low-rank approximations of kernel matrices are involved. We note that a sufficient condition for consistency is that the payoffs are positive on average:
        $\liminf_{t\to\infty}\frac{1}{t}\sum_{i=1}^t f_i(Z_{2i-1}, Z_{2i}) \overset{\mathrm{a.s.}}{>} 0$.
\end{remark}

\paragraph{Synthetic Experiments.} To compare SKITs based on HSIC, COCO, and KCC payoffs, we use RBF kernel with hyperparameters taken to be inversely proportional to the second moment of the underlying variables; we observed no substantial difference when such selection is data-driven (median heuristic). We consider settings where the complexity of a task is controlled through a single univariate parameter:
\begin{enumerate}[label=(\alph*),itemsep=0em]
    \item \emph{Gaussian model.} For $t\geq 1$, we consider $Y_t =  X_t\beta + \varepsilon_t$, where $X_t,\varepsilon_t \sim \mathcal{N}(0,1)$.
We have that $\beta\neq 0$ implies nonzero linear correlation (hence dependence). We consider 20 values for $\beta$, spaced uniformly in [0,0.3], and use $\lambda_X = 1/4$ and $\lambda_Y = 1/(4(1+\beta^2))$ as kernel hyperparameters.

\item \emph{Spherical model.} We generate a sequence of dependent but linearly uncorrelated random variables by taking $(X_t, Y_t) = ((U_t)_{(1)},(U_t)_{(2)})$, where $U_t\simiid \mathrm{Unif}(\mathbb{S}^d),$
for $t\geq 1$. $\mathbb{S}^d$ denotes a unit sphere in $\Real^d$ and $u_{(i)}$ is the $i$-th coordinate of $u$. We consider $d\in\curlybrack{3,\dots,15}$, and use $\lambda_X=\lambda_Y = d/4$ as kernel hyperparameters.
\end{enumerate}

We stop monitoring after 20000 points from $P_{XY}$ (if SKIT does not stop by that time, we retain the null) and aggregate the results over 200 runs for each value of $\beta$ and $d$. In Figure~\ref{fig:sim_data_power}, we confirm that SKITs control the type I error and adapt to the complexity of a task. In settings with a very low signal-to-noise ratio (small $\beta$ or large $d$), SKIT's power drops, but in such cases, both sequential and batch independence tests inevitably require a lot of data to reject the null. We defer additional experiments to Appendix~\ref{appsubsec:hard_to_det_vis}. 

\begin{figure}[!htb]
\begin{center}
        \subfigure[Gaussian model.]{\includegraphics[width=0.45\linewidth]{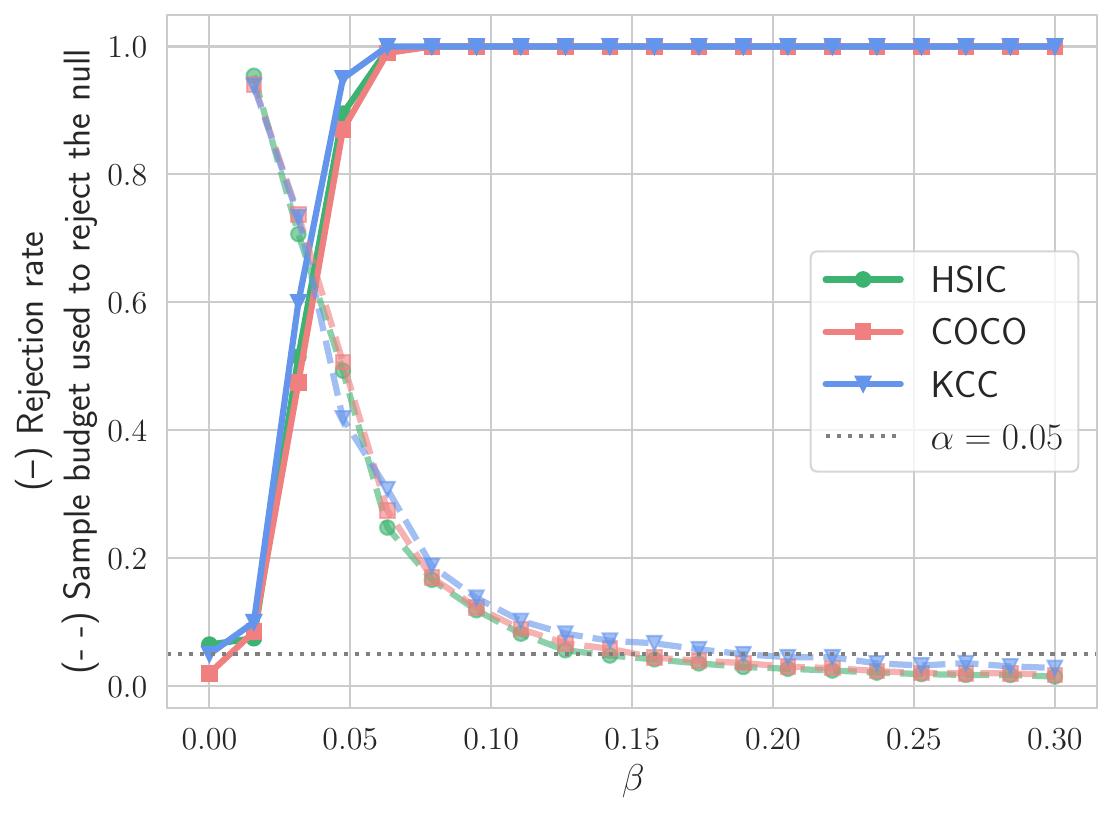}
                \label{subfig:gaus_sim_comp}}
        \subfigure[Spherical model.]{
                \includegraphics[width=0.45\linewidth]{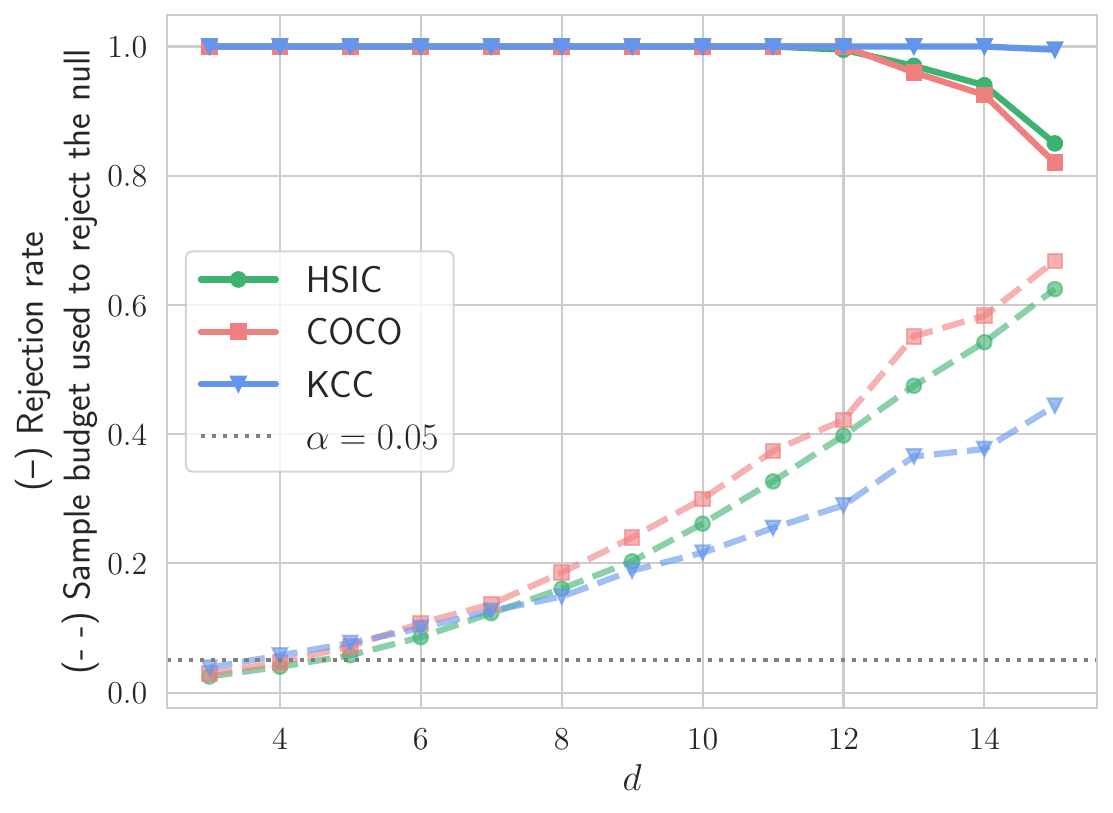}
                \label{subfig:sphere_sim_comp}}
        \caption{Rejection rate and scaled sample size used to reject the null hypothesis for synthetic data. Inspecting the rejection rate for $\beta=0$ (independence holds) confirms that the type I error is controlled. Further, we confirm that SKITs are adaptive to the complexity (smaller $\beta$ and larger $d$ correspond to harder settings).}
        \label{fig:sim_data_power}
\end{center}
\end{figure}

\section{Symmetry-based Betting Strategies}\label{sec:alt_bet}
In this section, we develop a betting strategy that relies on symmetry properties, whose advantage is that it overcomes the kernel boundedness assumption that underlined the SKIT construction. For example, using this betting strategy with a linear kernel: $k(x,y)=l(x,y)=\langle x,y\rangle$ readily implies a valid sequential linear correlation test. Consider
\begin{align}
    W_{t} &= \hat{g}_t(Z_{2t-1})+\hat{g}_t(Z_{2t}) - \hat{g}_t(\tilde{Z}_{2t-1})-\hat{g}_t(\tilde{Z}_{2t}), \label{eq:sym_rv}
\end{align}
where $\hat{g}_t = \hat{\mu}_{XY}-\hat{\mu}_X\otimes\hat{\mu}_Y$ is the \emph{unnormalized} plug-in witness function computed from $\curlybrack{Z_i}_{i\leq 2(t-1)}$. Symmetry-based betting strategies rely on the following key fact.
\begin{proposition}\label{prop:symmetry}
    Under any distribution in $H_0$, $W_{t}$ is symmetric around zero, conditional on $\calF_{t-1}$.
\end{proposition}
By construction, we expect the sign and magnitude of $W_t$ to be positively correlated under the alternative. We consider three payoff functions that aim to exploit this fact.

\begin{enumerate}
    \item \emph{Composition with an odd function.} This approach is based on the idea from sequential symmetry testing~\citep{ramdas2022admissible} that composition with an odd function of a symmetric around zero random variable is mean-zero. Absent knowledge regarding the scale of considered random variables, it is natural to standardize $\curlybrack{W_{i}}_{i\geq 1}$ in a predictable way. We consider
    \begin{equation}\label{eq:tanh_bet}
        f_t^{\mathrm{odd}}(W_{t}) = \tanh{\roundbrack{W_{t}/N_{t-1}}},
    \end{equation}
    where $N_{t}= \quant_{0.9}(\curlybrack{\abs{W_i}}_{i\leq t}) - \quant_{0.1}(\curlybrack{\abs{W_i}}_{i\leq t})$, and $\quant_\alpha(\curlybrack{\abs{W_i}}_{i\leq t})$ is the $\alpha$-quantile of the empirical distribution of the absolute values of $\curlybrack{W_i}_{i\leq t}$. (The choices of 0.1 and 0.9 are heuristic, and can be replaced by other constants without violating the validity of the test.) 
    The composition approach has demonstrated promising empirical performance for the betting-based two-sample tests of~\citet{shekhar2022one_two_sample} and conditional independence tests of~\citet{shaer2022seq_crt}.

\item \emph{Rank-based approach.} Inspired by sequential signed-rank test of symmetry around zero of~\citet{reynolds1975seq_signed_rank}, we consider the following payoff function:
\begin{align}
    f_t^{\mathrm{rank}}(W_{t}) &= \mathrm{sign}(W_{t}) \cdot \frac{\mathrm{rk}(\abs{W_{t}})}{t}, \label{eq:rank_bet}
\end{align}
where $\mathrm{rk}(\abs{W_{t}}) = \sum_{i=1}^{t} \indicator{|W_{i}|\leq |W_{t}|}$.
\item \emph{Predictive approach.} At round $t$, we fit a probabilistic predictor $p_t:\Real_+ \to [0,1]$, e.g., logistic regression, using $\curlybrack{\abs{W_i},\sign{W_i}}_{i\leq t-1}$ as feature-label pairs. We consider the following payoff function:
\begin{equation}\label{eq:pred_bet}
    f_t^{\mathrm{pred}}(W_{t})=
    (2p_t(\abs{W_{t}})-1)_+\cdot \roundbrack{1- 2\ell_t(W_{t})},
\end{equation}
where $\roundbrack{\cdot}_+=\max\curlybrack{\cdot,0}$ and $\ell_t(\abs{W_{t}},\sign{W_t})$ is the misclassification loss of the predictor $p_t$. 
\end{enumerate}

Next, we show that symmetry-based SKITs are valid; the proof is deferred to Appendix~\ref{appsubsec:proofs_alt_bet}.

\begin{algorithm}[!htb]
\caption{SKIT with symmetry-based betting}\label{alg:skit_symmetry}
\begin{algorithmic}
\State \textbf{Input:} significance level $\alpha\in (0,1)$, data stream $(Z_i)_{i\geq 1}$, where $Z_i=(X_i,Y_i)\sim P_{XY}$, $\lambda_1^{\mathrm{ONS}}=0$. 
\For {$t=1,2,\dots$}
\State After observing $Z_{2t-1}, Z_{2t}$, compute $W_{t}$ as in~\eqref{eq:sym_rv} and $f_t^{\mathrm{odd}}(W_t)$ as in~\eqref{eq:tanh_bet};
\State Update the wealth process $\calK_t$ as in~\eqref{eq:wealth_update};
\If{$\calK_t\geq 1/\alpha$}
\State Reject $H_0$ and stop;
\Else 
\State Compute $\lambda_{t+1}^{\mathrm{ONS}}$ (Algorithm~\ref{alg:ons});
\EndIf
\EndFor
\end{algorithmic}
\end{algorithm}

\begin{restatable}{theorem}{thmvaliditysym}\label{thm:validity_sym}
Under $H_0$ in~\eqref{eq:null_iid} or~\eqref{eq:null_noniid}, the symmetry-based SKIT (Algorithm~\ref{alg:skit_symmetry}) satisfies: $\Prob_{H_0}\roundbrack{\tau<\infty}\leq \alpha$.
\end{restatable}

\paragraph{Synthetic Experiments.} To compare the symmetry-based payoffs, we consider the Gaussian model along with aGRAPA betting fractions. For visualization purposes, we complete monitoring after observing 2000 points from the joint distribution. In Figure~\ref{subfig:sym_large_scale}, we observe that the resulting SKITs demonstrate similar performance. In Figure~\ref{fig:linear_kernel}, we demonstrate that SKIT with a linear kernel has high power under the Gaussian model, whereas its false alarm rate does not exceed $\alpha$ under the spherical model. Additional synthetic experiments can be found in Appendix~\ref{appsubsec:sym_bet}.

\begin{figure}[!htb]
\begin{center}
        \subfigure[]{              
                \includegraphics[width=0.45\linewidth]{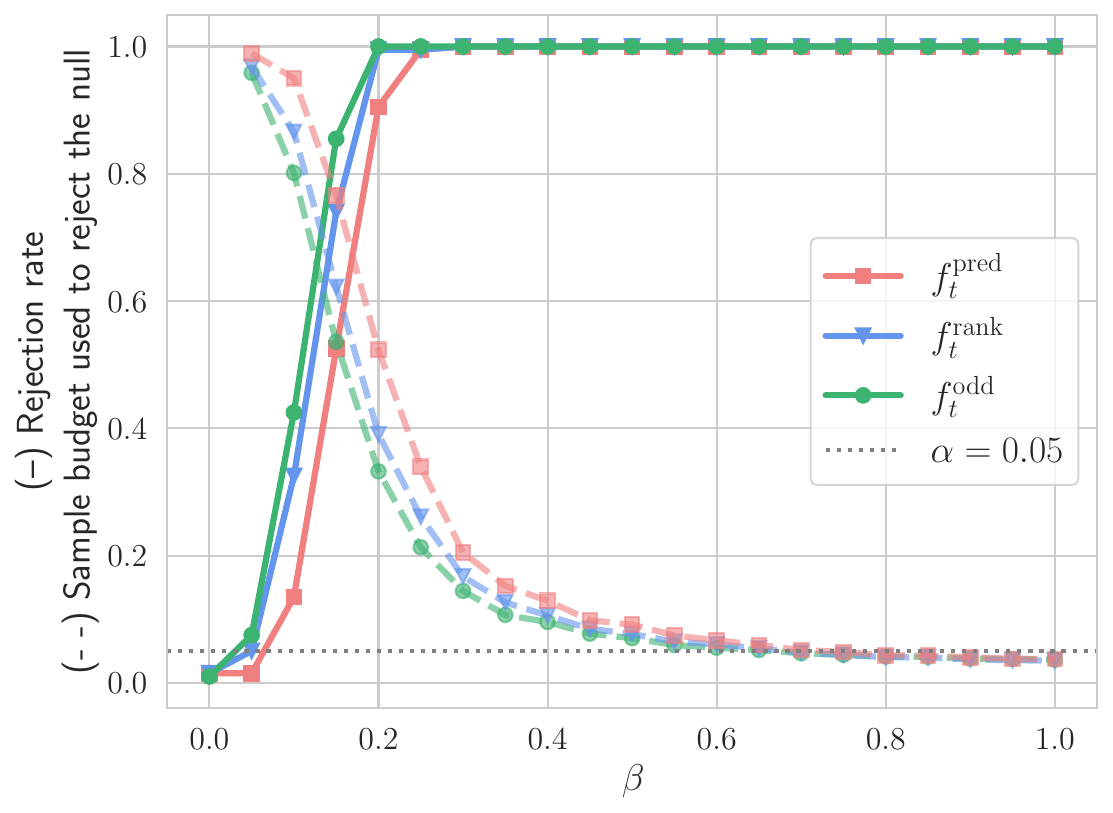}
                \label{subfig:sym_large_scale}}
        \subfigure[]{
                \includegraphics[width=0.45\linewidth]{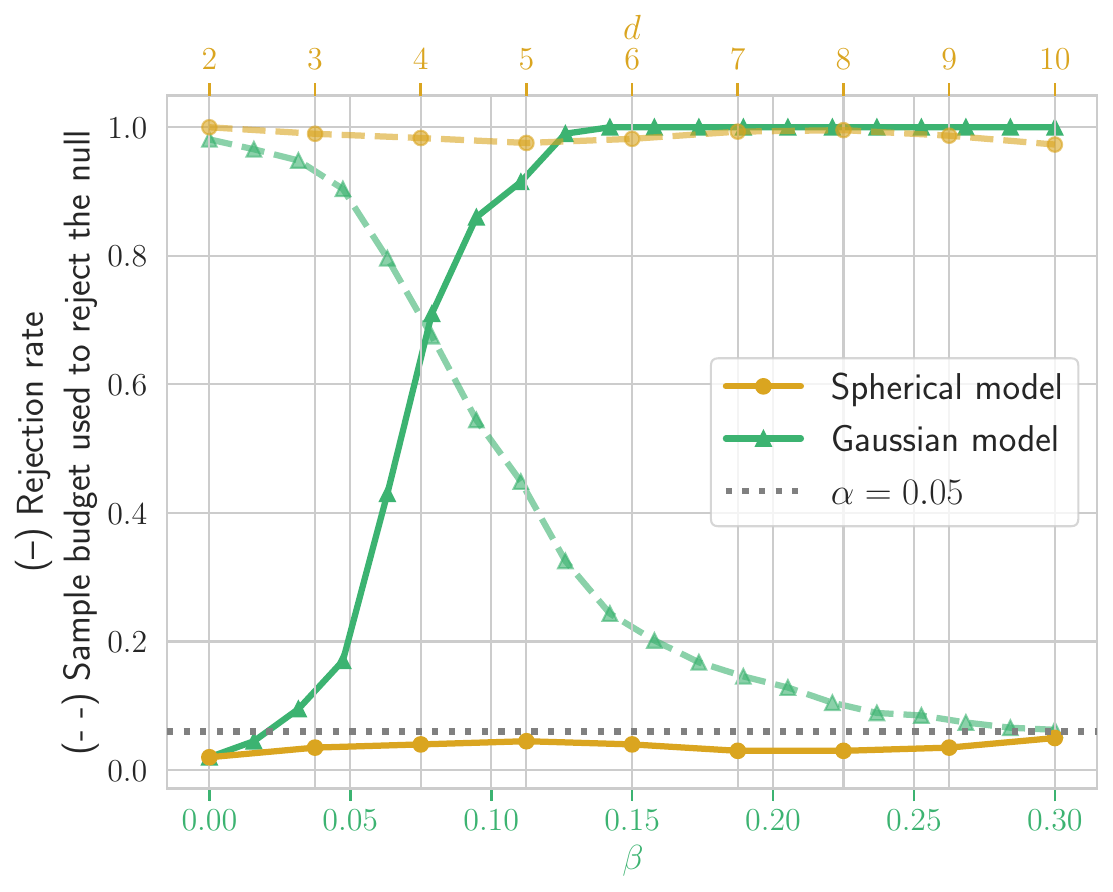}
                \label{fig:linear_kernel}}
        \caption{(a) SKITs with symmetry-based payoffs have high power under the Gaussian model. (b) SKIT with linear kernel has high power under the Gaussian model ($X$ and $Y$ are linearly correlated for $\beta\neq 0$), and its false alarm rate is controlled under the spherical model ($X$ and $Y$ are linearly uncorrelated but dependent).}
        \label{fig:ons_grapa_power}
\end{center}
\end{figure}

\paragraph{Real Data Experiments.} We analyze average daily temperatures\footnote{data source: https://www.wunderground.com} in four European cities: London, Amsterdam, Zurich, and Nice, from January 1, 2017, to May 31, 2022. The processes underlying temperature formation are complex and act both on macro (e.g., solar phase) and micro (e.g., local winds) levels. While average daily temperatures in selected cities share similar cyclic patterns, one may still expect the temperature fluctuations occurring in nearby locations to be dependent. We use SKIT for testing instantaneous independence (as per~\eqref{eq:ind_test_generalization}) between fluctuations (assuming that the conditions that underlie our test hold).  

We run SKIT with the rank-based payoff and ONS betting fractions for each pair of cities using $6/\alpha$ as a rejection threshold (accounting for multiple testing). We select the kernel hyperparameters via the median heuristic using recordings for the first 20 days. In Figure~\ref{subfig:europe_map}, we illustrate 
that SKIT supports our conjecture that temperature fluctuations are dependent in nearby locations. We also run this experiment for four cities in South Africa (see Appendix~\ref{appsubsec:real_data}). 

In addition, we analyze the performance of SKIT on MNIST data; the details are deferred to Appendix~\ref{appubsec:mnist_exp}.

\begin{figure}[!htb]
\begin{center}              
    \includegraphics[width=0.5\linewidth]{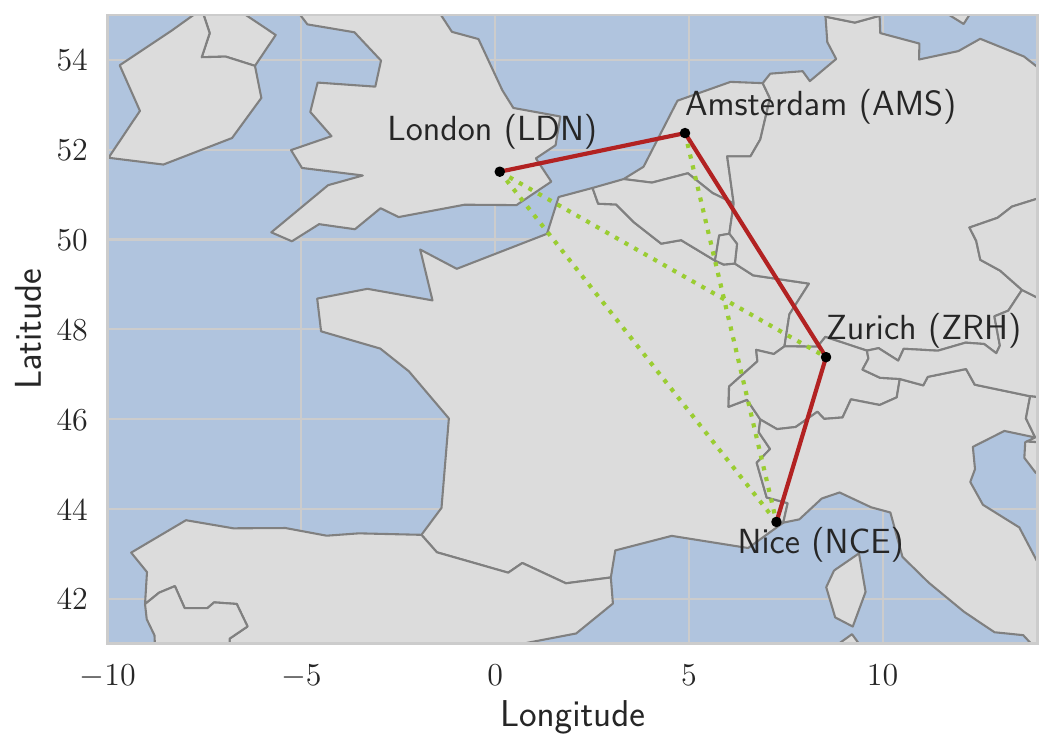}
    \caption{Solid lines connect cities for which the null is rejected. SKIT supports the conjecture regarding dependent temperature fluctuations in nearby locations. }
    \label{subfig:europe_map}
\end{center}
\end{figure}

\section{Conclusion}
A key advantage of sequential tests is that they can be continuously monitored, allowing an analyst to adaptively decide whether to stop and reject the null hypothesis or to continue collecting data, without inflating the false positive rate. In this paper, we design consistent sequential kernel independence tests (SKITs) following the principle of testing by betting. 
SKITs are also valid beyond the i.i.d.\ setting, allowing the data distribution to drift over time. Experiments on synthetic and real data confirm the power of SKITs.

\paragraph{Acknowledgements.} The authors thank Ian Waudby-Smith, Tudor Manole, and the anonymous reviewers for constructive feedback. The authors also acknowledge Lucas Janson and Will Hartog for thoughtful questions and comments at the International Seminar for Selective Inference.

\bibliographystyle{plainnat}
\bibliography{refs}

\newpage
\appendix
\onecolumn
\noindent\rule{\textwidth}{1pt}
\begin{center}
\vspace{7pt}
{\Large  Appendix}
\end{center}
\noindent\rule{\textwidth}{1pt}

\section{Independence Testing for Streaming Data}

In Section~\ref{appsubsec:batch_failure}, we describe a permutation-based approach for conducting batch HSIC and show that continuous monitoring of batch HSIC (without corrections for multiple testing) leads to an inflated false alarm rate. In Section~\ref{appsubsec:two_sample_reduction}, we introduce the sequential two-sample testing (2ST) problem and describe a reduction of sequential independence testing to sequential 2ST. In Section~\ref{appsubsec:batch_comp}, we compare our test to HSIC in the batch setting.

\subsection{Failure of Batch HSIC under Continuous Monitoring}\label{appsubsec:batch_failure}

To conduct independence testing using batch HSIC, we use permutation p-value (with $M=1000$ random permutations):
   $P = \tfrac{1}{M+1}(1+\sum_{m=1}^M\indicator{T_m\geq T})$,
where $T_m$ is the value of HS-norm computed from the $m$-th permutation and $T$ is HS-norm value on the original data. In other words, suppose that we are given a sample $Z_1,\dots, Z_t$, where $Z_i=(X_i,Y_i)$. Let $\calS_t$ denote the set of all permutations of $t$ indices and let $\sigma\sim \mathrm{Unif}(\calS_t)$ be a random permutation of indices. Then:
\begin{equation*}
    \begin{aligned}
        (X_1,Y_1), \dots, (X_t,Y_t) &\Longrightarrow T = \widehat{\HSIC}_b\roundbrack{(X_1,Y_1), \dots, (X_t,Y_t)}\\
        (X_1,Y_{\sigma_m(1)}), \dots, (X_t,Y_{\sigma_m(t)}) , &\Longrightarrow T_m = \widehat{\HSIC}_b\roundbrack{(X_1,Y_{\sigma_m(1)}), \dots, (X_t,Y_{\sigma_m(t)})}, \quad m\in \curlybrack{1,\dots,M},
    \end{aligned}
\end{equation*}
where we use a biased estimator of HSIC:
\begin{equation*}
    \widehat{\HSIC}_b = \frac{1}{t^2}\sum_{i,j} K_{ij}L_{ij} + \frac{1}{t^4}\sum_{i,j,q,r} K_{ij}L_{qr} - \frac{2}{t^3} \sum_{i,j,q} K_{ij}L_{iq} = \frac{1}{t^2}\tr{KHLH}. 
\end{equation*}
For brevity, we use $K_{ij}=k(X_i,X_j)$, $L_{ij}=l(Y_i,Y_j)$ for $i,j\in\curlybrack{1,\dots,t}$. Next, we study batch HSIC under (a) \emph{fixed-time} and (b) \emph{continuous} monitoring.  We consider a simple case when $X$ and $Y$ are independent standard Gaussian random variables. We consider (re)conducting a test at 12 different sample sizes: $t\in \curlybrack{50,100,\dots,600}$:
\begin{enumerate}[label=(\alph*)]
    \item Under fixed-time monitoring, for each value of $t$, we sample a sequence $Z_1,\dots,Z_t$ (100 times for each $t$) and conduct batch-HSIC test. The goal is to confirm that batch-HSIC controls type I error by tracking the standard miscoverage rate.
    \item Under continuous monitoring, we sample new datapoints and re-conduct the test. We illustrate inflated type I error by tracking the \emph{cumulative miscoverage rate}, that is, the fraction of times, the test falsely rejects the independence null. 
\end{enumerate}
The results are presented in Figure~\ref{fig:inflated_type_one_error}. For Bonferroni correction, we decompose the error budget as:
    $\alpha = \sum_{i=1}^\infty \frac{\alpha}{i(i+1)}$,
that is, for $t$-th test we use threshold $\alpha_t = \alpha/(t(t+1))$ for testing.

\begin{figure}[!htb]
    \centering
    \includegraphics[width=0.425\linewidth]{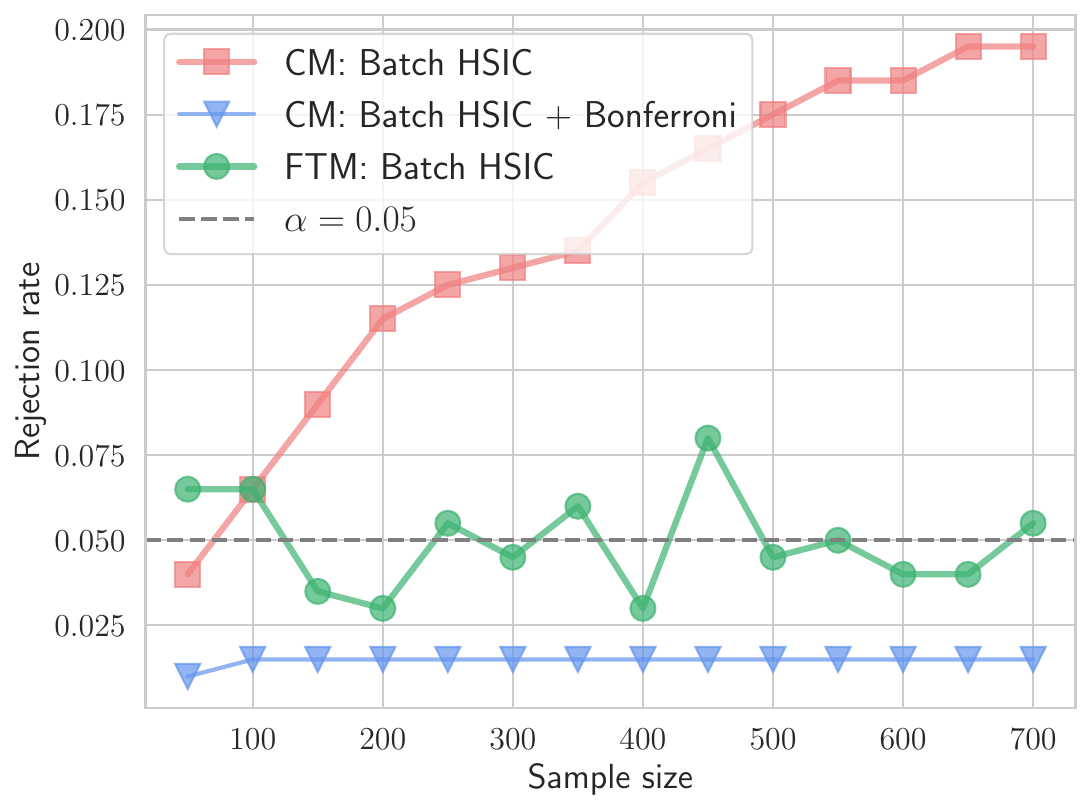}
    \caption{Inflated false alarm rate of batch HSIC under continuous monitoring (CM, red line with squares) for the case when $X$ and $Y$ are independent standard Gaussian random variables. Bonferroni correction (CM, blue line with triangles) restores type I error control. As expected, type I error is controlled at a specified level under fixed-time monitoring (FTM, green line with circles). }
    \label{fig:inflated_type_one_error}
\end{figure}

\subsection{Sequential Independence Testing via Sequential Two-Sample Testing}\label{appsubsec:two_sample_reduction}

First, we introduce the sequential two-sample testing problem. Suppose that we observe a stream of data: $(\tilde{X}_1,\tilde{Y}_1),(\tilde{X}_2,\tilde{Y}_2),\dots$, where $(\tilde{X}_t,\tilde{Y}_t)\simiid P\times Q$. Two-sample testing refers to testing:
\begin{equation*}
\begin{aligned}
    H_0: \ (\tilde{X}_t,\tilde{Y}_t)\simiid P\times Q \ \text{and} \  P=Q, \quad  \text{vs.} \quad H_1: \ (\tilde{X}_t,\tilde{Y}_t)\simiid P\times Q \ \text{and} \  P\neq Q.
\end{aligned}
\end{equation*}

In Figure~\ref{fig:intro_fig}, we compared our test against the approach based on the reduction of independence testing to two-sample testing. We used the sequential two-sample kernel MMD test of~\citet{shekhar2022one_two_sample} with the product kernel $\tilde{K}$ (that is, a product of Gaussian kernels) and the same set of hyperparameters as for our test for a fair comparison. To reduce sequential independence testing to any off-the-shelf sequential two-sample testing procedure, we convert the original sequence of points from $P_{XY}$ to a sequence of i.i.d.\ $(\tilde{X}_t,\tilde{Y}_t)$-pairs, where $\tilde{X}_t\sim P_{XY}$ and $\tilde{Y}_t\sim P_{X}\times P_Y$ respectively; see Figure~\ref{subfig:two_sample_iid}. At $t$-th round, we randomly choose one point as $\tilde{X}_t$, e.g., $(X_1,Y_1)$ for the first triple. Next, we obtain $\tilde{Y}_t$ by randomly matching $X$ and $Y$ from two other pairs, e.g., $(X_2,Y_3)$ or $(X_3,Y_2)$ for the first triple. In fact, the betting-based sequential two-sample test of~\citep{shekhar2022one_two_sample} allows removing the effect of randomization (i.e., throwing away one observation in each triple), by averaging payoffs evaluated on $(\tilde{X}_t,\tilde{Y}_t^{(1)})$ and $(\tilde{X}_t,\tilde{Y}_t^{(2)})$. Other approaches --- which do not require throwing data away --- are also available (Figures~\ref{subfig:two_sample_triple}) but those only yield an i.i.d.\ sequence only under the null. 

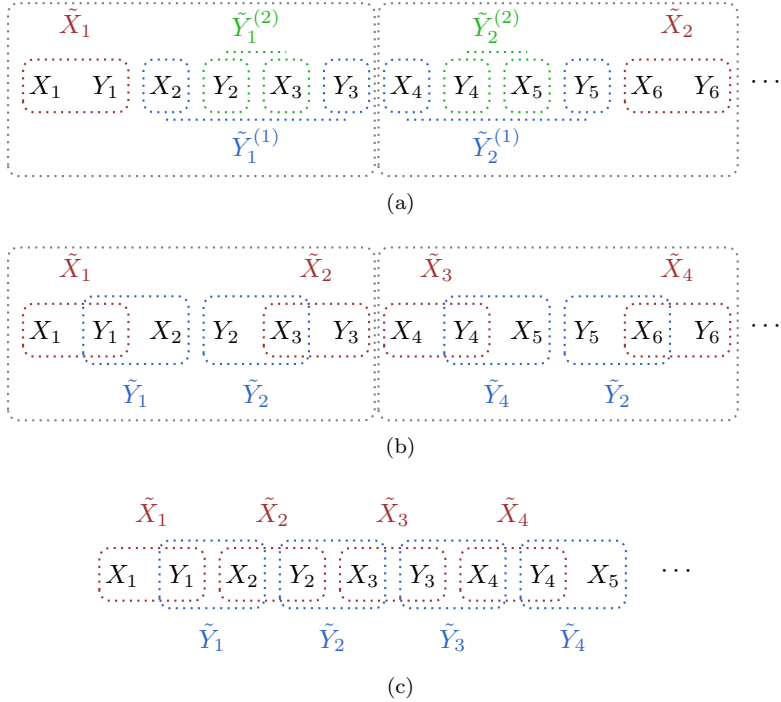
\begin{figure}[!htb]
\begin{center}
    \subfigure[]{\begin{tikzpicture}
\filldraw[black] (0,0) circle (0pt) node[anchor=north]{$X_1$};
\filldraw[black] (0.8,0) circle (0pt) node[anchor=north]{$Y_1$};
\filldraw[black] (1.6,0) circle (0pt) node[anchor=north]{$X_2$};
\filldraw[black] (2.4,0) circle (0pt) node[anchor=north]{$Y_2$};
\filldraw[black] (3.2,0) circle (0pt) node[anchor=north]{$X_3$};
\filldraw[black] (4,0) circle (0pt) node[anchor=north]{$Y_3$};
\filldraw[black] (4.8,0) circle (0pt) node[anchor=north]{$X_4$};
\filldraw[black] (5.6,0) circle (0pt) node[anchor=north]{$Y_4$};
\filldraw[black] (6.4,0) circle (0pt) node[anchor=north]{$X_5$};
\filldraw[black] (7.2,0) circle (0pt) node[anchor=north]{$Y_5$};
\filldraw[black] (8,0) circle (0pt) node[anchor=north]{$X_6$};
\filldraw[black] (8.8,0) circle (0pt) node[anchor=north]{$Y_6$};
\filldraw[black] (9.6,0) circle (0pt) node[anchor=north]{$\cdots$};
\filldraw[black] (0.4,0.9) circle (0pt) node[anchor=north]{$\textcolor{myred}{\tilde{X}_1}$};
\filldraw[black] (8.4,0.9) circle (0pt) node[anchor=north]{$\textcolor{myred}{\tilde{X}_2}$};
\filldraw[black] (2.8,0.9) circle (0pt) node[anchor=north]{$\textcolor{mygreen}{\tilde{Y}_1^{(2)}}$};
\filldraw[black] (6,0.9) circle (0pt) node[anchor=north]{$\textcolor{mygreen}{\tilde{Y}_2^{(2)}}$};
\filldraw[black] (2.8,-1.4) circle (0pt) node[anchor=south]{$\textcolor{myblue}{\tilde{Y}_1^{(1)}}$};
\filldraw[black] (6,-1.4) circle (0pt) node[anchor=south]{$\textcolor{myblue}{\tilde{Y}_2^{(1)}}$};
\draw[myred, dotted, thick, rounded corners] (-0.3,-0.6) rectangle (1.1,0.1);
\draw[myred, dotted, thick, rounded corners] (7.7,-0.6) rectangle (9.1,0.1);
\draw[myblue, dotted, thick, rounded corners] (1.3,-0.6) rectangle (1.9,0.1);
\draw[myblue, dotted, thick, rounded corners] (3.7,-0.6) rectangle (4.3,0.1);
\draw[myblue, dotted, thick, rounded corners] (4.5,-0.6) rectangle (5.1,0.1);
\draw[myblue, dotted, thick, rounded corners] (6.9,-0.6) rectangle (7.5,0.1);
\draw[mygreen, dotted, thick, rounded corners] (2.1,-0.6) rectangle (2.7,0.1);
\draw[mygreen, dotted, thick, rounded corners] (2.9,-0.6) rectangle (3.5,0.1);
\draw[mygreen, dotted, thick, rounded corners] (5.3,-0.6) rectangle (5.9,0.1);
\draw[mygreen, dotted, thick, rounded corners] (6.1,-0.6) rectangle (6.7,0.1);
\draw[gray, dotted, thick, rounded corners] (-0.5,-1.45) rectangle (4.375,0.85);
\draw[gray, dotted, thick, rounded corners] (4.425,-1.45) rectangle (9.2,0.85);
\draw[mygreen, dotted, thick] (2.4,0.2) -- (3.2,0.2);
\draw[mygreen, dotted, thick] (5.6,0.2) -- (6.4,0.2);
\draw[myblue, dotted, thick] (1.6,-0.7) -- (4,-0.7);
\draw[myblue, dotted, thick] (4.8,-0.7) -- (7.2,-0.7);
\end{tikzpicture}
\label{subfig:two_sample_iid}}
\subfigure[]{\begin{tikzpicture}
\filldraw[black] (0,0) circle (0pt) node[anchor=north]{$X_1$};
\filldraw[black] (0.8,0) circle (0pt) node[anchor=north]{$Y_1$};
\filldraw[black] (1.6,0) circle (0pt) node[anchor=north]{$X_2$};
\filldraw[black] (2.4,0) circle (0pt) node[anchor=north]{$Y_2$};
\filldraw[black] (3.2,0) circle (0pt) node[anchor=north]{$X_3$};
\filldraw[black] (4,0) circle (0pt) node[anchor=north]{$Y_3$};
\filldraw[black] (4.8,0) circle (0pt) node[anchor=north]{$X_4$};
\filldraw[black] (5.6,0) circle (0pt) node[anchor=north]{$Y_4$};
\filldraw[black] (6.4,0) circle (0pt) node[anchor=north]{$X_5$};
\filldraw[black] (7.2,0) circle (0pt) node[anchor=north]{$Y_5$};
\filldraw[black] (8,0) circle (0pt) node[anchor=north]{$X_6$};
\filldraw[black] (8.8,0) circle (0pt) node[anchor=north]{$Y_6$};
\filldraw[black] (9.6,0) circle (0pt) node[anchor=north]{$\cdots$};
\filldraw[black] (0.4,0.9) circle (0pt) node[anchor=north]{$\textcolor{myred}{\tilde{X}_1}$};
\filldraw[black] (3.6,0.9) circle (0pt) node[anchor=north]{$\textcolor{myred}{\tilde{X}_2}$};
\filldraw[black] (5.2,0.9) circle (0pt) node[anchor=north]{$\textcolor{myred}{\tilde{X}_3}$};
\filldraw[black] (8.4,0.9) circle (0pt) node[anchor=north]{$\textcolor{myred}{\tilde{X}_4}$};
\filldraw[black] (1.2,-1.4) circle (0pt) node[anchor=south]{$\textcolor{myblue}{\tilde{Y}_1}$};
\filldraw[black] (2.8,-1.4) circle (0pt) node[anchor=south]{$\textcolor{myblue}{\tilde{Y}_2}$};
\filldraw[black] (6,-1.4) circle (0pt) node[anchor=south]{$\textcolor{myblue}{\tilde{Y}_4}$};
\filldraw[black] (7.6,-1.4) circle (0pt) node[anchor=south]{$\textcolor{myblue}{\tilde{Y}_2}$};
\draw[myred, dotted, thick, rounded corners] (-0.3,-0.6) rectangle (1.1,0.1);
\draw[myred, dotted, thick, rounded corners] (2.9,-0.6) rectangle (4.3,0.1);
\draw[myred, dotted, thick, rounded corners] (4.5,-0.6) rectangle (5.9,0.1);
\draw[myred, dotted, thick, rounded corners] (7.7,-0.6) rectangle (9.1,0.1);
\draw[myblue, dotted, thick, rounded corners] (0.5,-0.7) rectangle (1.9,0.2);
\draw[myblue, dotted, thick, rounded corners] (2.1,-0.7) rectangle (3.5,0.2);
\draw[myblue, dotted, thick, rounded corners] (5.3,-0.7) rectangle (6.7,0.2);
\draw[myblue, dotted, thick, rounded corners] (6.9,-0.7) rectangle (8.3,0.2);
\draw[gray, dotted, thick, rounded corners] (-0.5,-1.45) rectangle (4.375,0.85);
\draw[gray, dotted, thick, rounded corners] (4.425,-1.45) rectangle (9.2,0.85);
\end{tikzpicture}
\label{subfig:two_sample_triple}}
\subfigure[]{\begin{tikzpicture}
\filldraw[black] (0,0) circle (0pt) node[anchor=north]{$X_1$};
\filldraw[black] (0.8,0) circle (0pt) node[anchor=north]{$Y_1$};
\filldraw[black] (1.6,0) circle (0pt) node[anchor=north]{$X_2$};
\filldraw[black] (2.4,0) circle (0pt) node[anchor=north]{$Y_2$};
\filldraw[black] (3.2,0) circle (0pt) node[anchor=north]{$X_3$};
\filldraw[black] (4,0) circle (0pt) node[anchor=north]{$Y_3$};
\filldraw[black] (4.8,0) circle (0pt) node[anchor=north]{$X_4$};
\filldraw[black] (5.6,0) circle (0pt) node[anchor=north]{$Y_4$};
\filldraw[black] (6.4,0) circle (0pt) node[anchor=north]{$X_5$};
\filldraw[black] (7.4,0) circle (0pt) node[anchor=north]{$\cdots$};
\filldraw[black] (0.4,0.9) circle (0pt) node[anchor=north]{$\textcolor{myred}{\tilde{X}_1}$};
\filldraw[black] (2,0.9) circle (0pt) node[anchor=north]{$\textcolor{myred}{\tilde{X}_2}$};
\filldraw[black] (3.6,0.9) circle (0pt) node[anchor=north]{$\textcolor{myred}{\tilde{X}_3}$};
\filldraw[black] (5.2,0.9) circle (0pt) node[anchor=north]{$\textcolor{myred}{\tilde{X}_4}$};
\filldraw[black] (1.2,-1.4) circle (0pt) node[anchor=south]{$\textcolor{myblue}{\tilde{Y}_1}$};
\filldraw[black] (2.8,-1.4) circle (0pt) node[anchor=south]{$\textcolor{myblue}{\tilde{Y}_2}$};
\filldraw[black] (4.4,-1.4) circle (0pt) node[anchor=south]{$\textcolor{myblue}{\tilde{Y}_3}$};
\filldraw[black] (6,-1.4) circle (0pt) node[anchor=south]{$\textcolor{myblue}{\tilde{Y}_4}$};
\draw[myred, dotted, thick, rounded corners] (-0.3,-0.6) rectangle (1.1,0.1);
\draw[myred, dotted, thick, rounded corners] (1.3,-0.6) rectangle (2.7,0.1);
\draw[myred, dotted, thick, rounded corners] (2.9,-0.6) rectangle (4.3,0.1);
\draw[myred, dotted, thick, rounded corners] (4.5,-0.6) rectangle (5.9,0.1);
\draw[myblue, dotted, thick, rounded corners] (0.5,-0.7) rectangle (1.9,0.2);
\draw[myblue, dotted, thick, rounded corners] (2.1,-0.7) rectangle (3.5,0.2);
\draw[myblue, dotted, thick, rounded corners] (3.7,-0.7) rectangle (5.1,0.2);
\draw[myblue, dotted, thick, rounded corners] (5.3,-0.7) rectangle (6.7,0.2);
\end{tikzpicture}
\label{subfig:two_sample_fail}}
    \caption{Reducing sequential independence testing to sequential two-sample testing. Processing as per (a) results in a sequence of i.i.d.\ observations both under the null and under the alternative (making the results about power valid). Processing data as per (b) gives an i.i.d.\ sequence only under the null. Reduction (b) is very similar to reduction (c). However, the latter makes $\tilde{X}_i$, $i\geq 2$, dependent on the past, and thus can not be used directly for considered sequential two-sample tests.}
    \label{subfig:two_sample_reductions}
\end{center}
\end{figure}

\paragraph{Additional Details of the Simulation Presented in Figure~\ref{fig:intro_fig}.}

We consider the Gaussian model: $Y_t=X_t\beta+\varepsilon_t$, where $X_t,\varepsilon_t\sim\mathcal{N}(0,1)$, $t\geq 1$. We consider 10 values of $\beta$: $\beta\in\curlybrack{0,0.04,\dots,0.36}$, and for each $\beta$ we repeat the simulation 100 times. In this simulation, we compare three approaches for testing independence (valid under continuous monitoring):
\begin{enumerate}
    \item \textbf{HSIC-based SKIT} proposed in this work (Algorithm~\ref{alg:ind_test_framework});
    \item \textbf{Batch HSIC} adapted to continuous monitoring via Bonferroni correction. We allow monitoring after processing every $n$, $n\in\curlybrack{10,100}$, new points from $P_{XY}$, that is, the permutation p-value (computed over 2500 randomly sampled permutations) is compared against rejection thresholds: $\alpha_n = \alpha/(n(n+1))$, $n=1,2,\dots$ 
    \item \textbf{Sequential independence testing via reduction sequential 2ST} as described above.  
\end{enumerate}

We use RBF kernel with the same set of kernel hyperparameters for all testing procedures: $\lambda_X = 1/4$, $\lambda_Y=1/(4(1+\beta^2)).$

\subsection{Comparison in the Batch Setting}\label{appsubsec:batch_comp}

Sequential tests are complementary to batch tests and are not intended to replace them, and hence comparing the two on equal footing is hard. To highlight this, consider two simple scenarios. If we have 2000 data points, and HSIC fails to reject, there is not much we can do to rescue the situation. But if SKIT fails to reject, an analyst can collect more data and continue testing, retaining type I error control. In contrast, with 2 million points, HSIC will take forever to run, especially due to permutations. But if the alternative is true and the signal is strong, then SKIT may reject within 200 samples and stop. In short, the ability of SKIT to continue collecting and analyzing data is helpful for hard problems, and the ability of SKIT to stop early is helpful for easy problems. There is no easy sense in which one can compare them apples to apples and there is no sense in which batch HSIC uniformly dominates SKIT or vice versa. In a real setting, if an analyst has a strong hunch that the null is false and has the ability to collect data and run HSIC, the question is how much data should be collected? The answer depends on the underlying data distribution, which is of course unknown. With SKIT, data can be collected and analyzed sequentially. Theorem~\ref{thm:hsic_skit} implies that SKIT will stop early on easy problems and later on harder problems, all without knowing anything about the problem in advance. If however, one has a fixed batch of data, no chance to collect more, and no computational constraints, then running HSIC makes more sense. 
    
To illustrate that batch HSIC can be superior to SKIT, we compare tests on a dataset with a prespecified sample size (500 observations from the Gaussian model) and track the empirical rejection rates of two tests. In Figure~\ref{fig:batch_comparison}, we show that HSIC actually has higher power than SKIT. However, for $\beta=0.1$ (where all tests have low power), Figure~\ref{subfig:gaus_sim_comp} shows that collecting just a bit more data (which is allowed) is needed for SKIT to reach perfect power. We also added a third method (D-SKIT) which removes the effect of the ordering of random variables under the assumptions that $\{(X_i,Y_i)\}_{i=1}^n$ are independent draws from $P_{XY}$. Let $\{\sigma_b\}_{b=1}^B$ define $B$ random permutations of $n$ indices, and let $\calK_n^b$ denote the wealth after betting on a sequence ordered according to $\sigma_b$. For each $b$, $\calK_n^b$ has expectation at most one, and hence (by linearity of expectation and Markov's inequality) $\indicator{\frac{1}{B}\sum_{i=1}^B \calK_n^b\geq 1/\alpha}$ is a valid level-$\alpha$ batch test. This test is a bit more stable: it improves SKIT's power on moderate-complexity setups at the cost of a slight power loss on more extreme ones.

\begin{figure}[!htb]
\centering
        \includegraphics[width=0.425\linewidth]{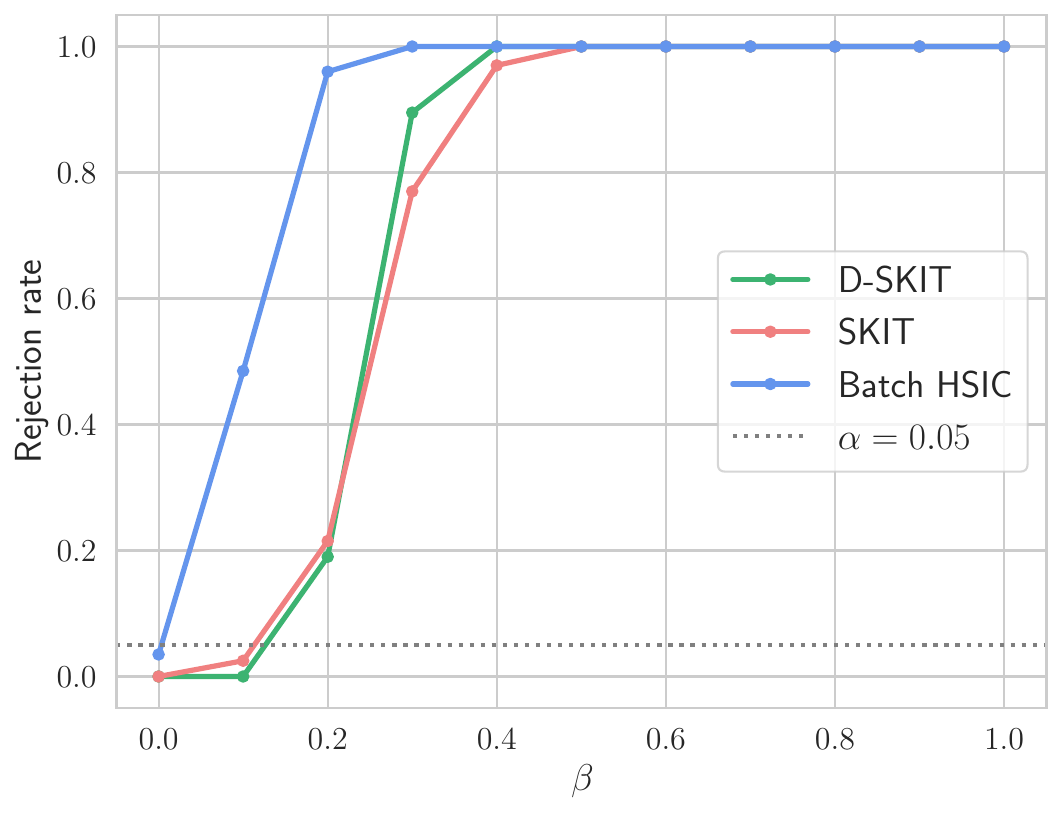}
        \caption{Comparison of SKIT and HSIC under Gaussian model in the batch setting. Non-surprisingly, batch HSIC performs best. D-SKIT improves over SKIT's power on moderate-complexity setups at the cost of a slight power loss on more extreme ones.}
        \label{fig:batch_comparison}
\end{figure}

\clearpage

\section{Proofs}\label{appsec:proofs}

Section~\ref{appsubsec:aux_res} contains auxiliary results needed to prove the results presented in this paper. In Section~\ref{appsubsec:proofs_sec_2}, we prove the results from Section~\ref{sec:seq_indep_testing}. In Secton~\ref{appsubsec:proofs_sec_3}, we prove the results from Section~\ref{sec:alt_measures}.

\subsection{Auxiliary Results}\label{appsubsec:aux_res}

\begin{theorem}[Ville's inequality~\citep{ville1939etude}]\label{thm:villes_ineq}
Suppose that $\roundbrack{\calM_t}_{t\geq 0}$ is a nonnegative supermartingale process adapted to a filtration $\curlybrack{\calF_t:t\geq 0}$. Then, for any $a>0$ it holds that:
\begin{equation*}
    \Prob\roundbrack{\exists t\geq 1:\calM_t\geq a}\leq \frac{\Exp{}{\calM_0}}{a}.
\end{equation*}
\end{theorem}

\begin{theorem}[Theorem 3 in~\citep{gretton2005measuring}]\label{thm:concen_hsic}
    Assume that $k$ and $l$ are bounded almost everywhere by 1, and are nonnegative. Then for $n>1$ and any $\delta\in (0,1)$, it holds with probability at least $1-\delta$ that:
    \begin{equation*}
        \abs{\mathrm{HSIC}(P_{XY};\calG,\calH)-\widehat{\mathrm{HSIC}}_b(P_{XY};\calG,\calH)}\leq \sqrt{\frac{\log (6/\delta)}{\alpha^2n}}+\frac{C}{n},
    \end{equation*}
    where $\alpha^2>0.24$ and $C$ are some absolute constants.
\end{theorem}



\subsection{Proofs for Section~\ref{sec:seq_indep_testing}}\label{appsubsec:proofs_sec_2}

In Section~\ref{appsubsubsec:sup_results}, we prove several intermediate results. The proofs of the main results are deferred to Section~\ref{appsubsubsec:sec_2_main}.

\subsubsection{Supporting Lemmas}\label{appsubsubsec:sup_results}

Before we state the first result, recall the definition of the empirical mean embeddings computed from the first $2(t-1)$ datapoints:
\begin{equation}\label{eq:emp_mean_embeddings}
    \begin{gathered}
        \hat{\mu}_{XY}^{(t)} = \frac{1}{2(t-1)} \sum_{i=1}^{2(t-1)} \varphi(X_{i})\otimes \psi(Y_{i}),\\
        \hat{\mu}_{X}^{(t)} = \frac{1}{2(t-1)} \sum_{i=1}^{2(t-1)} \varphi(X_{i}), \quad
        \hat{\mu}_{Y}^{(t)} = \frac{1}{2(t-1)} \sum_{i=1}^{2(t-1)}  \psi(Y_{i}),
    \end{gathered}
\end{equation}
where we highlight the dependence on the number of processed datapoints. We have the following result.
\begin{lemma}\label{lem:as_conv_norms}
    For the empirical~\eqref{eq:emp_mean_embeddings} and population~\eqref{eq:mean_emb_pop} mean embeddings, it holds that:
    \begin{equation}\label{eq:alm_sure_conv_of_norms}
    \norm{\calG\otimes\calH}{\hat{\mu}_{XY}^{(t)} - \hat{\mu}_X^{(t)}\otimes\hat{\mu}_Y^{(t)}}\overset{\mathrm{a.s.}}{\longrightarrow} \norm{\calG\otimes\calH}{\mu_{XY} - \mu_X\otimes\mu_Y}.
    \end{equation}
\end{lemma}
\begin{proof}

We have
\begin{equation*}
\begin{aligned}
    \norm{\calG\otimes\calH}{\mu_{XY} - \mu_X\otimes\mu_Y}^2 &= \mathrm{HSIC}(P_{XY};\calG,\calH),\\
    \norm{\calG\otimes\calH}{\hat{\mu}_{XY}^{(t)} - \hat{\mu}_X^{(t)}\otimes\hat{\mu}_Y^{(t)}}^2 &= \widehat{\mathrm{HSIC}}_b^{(t)}(P_{XY};\calG,\calH),
\end{aligned}
\end{equation*}
where the latter is a biased estimator of HSIC, computed from $2(t-1)$ datapoints from $P_{XY}$. From Theorem~\ref{thm:concen_hsic} and the Borel-Cantelli lemma, it follows that:
\begin{equation*}
    \norm{\calG\otimes\calH}{\hat{\mu}_{XY}^{(t)} - \hat{\mu}_X^{(t)}\otimes\hat{\mu}_Y^{(t)}}^2\overset{\mathrm{a.s.}}{\longrightarrow} \norm{\calG\otimes\calH}{\mu_{XY} - \mu_X\otimes\mu_Y}^2.
\end{equation*}
The result then follows from the continuous mapping theorem.
\end{proof}

\begin{lemma}\label{lem:alm_sure_conv_g_t}
    Suppose that $H_1$ in~\eqref{eq:alt_iid} is true. Then for the oracle~\eqref{eq:hsic_oracle_wf} and plug-in~\eqref{eq:hsic_plugin_wf} witness functions, it holds that:
    \begin{equation}\label{eq:target_ineq}
    \langle \hat{g}_t, g^\star \rangle_{\calG\otimes\calH} \overset{\mathrm{a.s.}}{\longrightarrow} 1.
\end{equation}
As a consequence, $\norm{\calG\otimes\calH}{\hat{g}_t - g^\star} \overset{\mathrm{a.s.}}{\longrightarrow} 0$.
\end{lemma}

\begin{proof}

    Suppose that the alternative in~\eqref{eq:alt_iid} happens to be true. Then since $k$ and $l$ are characteristic kernels, it follows that:
    \begin{equation*}
        \norm{\calG\otimes\calH}{\mu_{XY} - \mu_X\otimes\mu_Y}>0.
    \end{equation*}
    We aim to show that:
    \begin{equation*}
    \anglebrack{\frac{\hat{\mu}_{XY}^{(t)}-\hat{\mu}_X^{(t)}\otimes\hat{\mu}_Y^{(t)}}{\norm{\calG\otimes\calH}{\hat{\mu}_{XY}^{(t)}-\hat{\mu}_X^{(t)}\otimes\hat{\mu}_Y^{(t)}}}, \frac{\mu_{XY}-\mu_X\otimes\mu_Y}{\norm{\calG\otimes\calH}{\mu_{XY}-\mu_X\otimes \mu_Y}}}_{\calG\otimes\calH} \overset{\mathrm{a.s.}}{\longrightarrow}1.
    \end{equation*}
    From Lemma~\ref{lem:as_conv_norms}, we know that: $\norm{\calG\otimes\calH}{\hat{\mu}_{XY}^{(t)} - \hat{\mu}_X^{(t)}\otimes\hat{\mu}_Y^{(t)}}\overset{\mathrm{a.s.}}{\longrightarrow} \norm{\calG\otimes\calH}{\mu_{XY} - \mu_X\otimes\mu_Y}$. Hence it suffices to show that 
    \begin{equation}\label{eq:suffice}
        \anglebrack{\hat{\mu}_{XY}^{(t)}-\hat{\mu}_X^{(t)}\otimes\hat{\mu}_Y^{(t)}, \mu_{XY}-\mu_X\otimes \mu_Y}_{\calG\otimes\calH} \overset{\mathrm{a.s.}}{\longrightarrow}\norm{\calG\otimes\calH}{\mu_{XY} - \mu_X\otimes\mu_Y}^2.
    \end{equation}
    
    \noindent Recall that: $\mu_{XY}-\mu_X\otimes \mu_Y = \Exp{}{\varphi(\tilde{X})\otimes \psi(\tilde{Y})}-\Exp{}{\varphi(\tilde{X})}\otimes \Exp{}{\psi(\tilde{Y})}$. We have:
\begin{equation*}
    \hat{\mu}_{XY}^{(t)}-\hat{\mu}_X^{(t)}\otimes\hat{\mu}_Y^{(t)} = \roundbrack{1-\frac{1}{2(t-1)}}\roundbrack{\frac{1}{2(t-1)}\sum_{i=1}^{2(t-1)} \varphi(X_{i})\otimes \psi(Y_{i}) - \frac{1}{4(t-1)^2-2(t-1)}\sum_{\substack{j,k=1:\\j\neq k}}^{2(t-1)}\varphi(X_{j})\otimes \psi(Y_{k})}.
\end{equation*}
Further, it holds that:
\begin{equation*}
\begin{aligned}
    &\quad \anglebrack{\hat{\mu}_{XY}^{(t)}-\hat{\mu}_X^{(t)}\otimes\hat{\mu}_Y^{(t)},\mu_{XY}-\mu_X\otimes \mu_Y}_{\calG\otimes\calH} \\
    = & \quad\roundbrack{1-\frac{1}{2(t-1)}}\roundbrack{\frac{1}{2(t-1)}\sum_{i=1}^{2(t-1)} \Exp{\tilde{X},\tilde{Y}}{\langle \varphi(\tilde{X}),\varphi(X_{i})\rangle_{\calG}\langle \psi(\tilde{Y}), \psi(Y_{i})\rangle_{\calH}}}\\ 
    - & \quad\roundbrack{1-\frac{1}{2(t-1)}}\roundbrack{\frac{1}{4(t-1)^2-2(t-1)}\sum_{\substack{j,k=1:\\j\neq k}}^{2(t-1)}\Exp{\tilde{X}}{\langle \varphi(\tilde{X}),\varphi(X_{j})\rangle_{\calG}}\Exp{\tilde{Y}}{\langle \psi(\tilde{Y}),\psi(Y_{k})\rangle_{\calH}}},
\end{aligned}
\end{equation*}
    For any $(x,y)\in\calX\times\calY$, we have:
\begin{align*}
    \abs{\Exp{\tilde{X},\tilde{Y}}{\langle \varphi(\tilde{X}),\varphi(x)\rangle_{\calG}\langle \psi(\tilde{Y}), \psi(y)\rangle_{\calH}}} &\leq \Exp{\tilde{X},\tilde{Y}}{\abs{\langle \varphi(\tilde{X}),\varphi(x)\rangle_{\calG}\langle \psi(\tilde{Y}), \psi(y)\rangle_{\calH}}}\\
    &\leq \Exp{\tilde{X},\tilde{Y}}{\sqrt{k(\tilde{X},\tilde{X})k(x,x)l(\tilde{Y},\tilde{Y})k(y,y)}}\\
    &\leq 1,
\end{align*}
and similarly, for any $(x,y)\in\calX\times\calY$ it holds that:
\begin{equation*}
    \abs{\Exp{\tilde{X}}{\langle \varphi(\tilde{X}),\varphi(x)\rangle_{\calG}}\Exp{\tilde{Y}}{\langle \psi(\tilde{Y}),\psi(y)\rangle_{\calH}}}\leq 1.
\end{equation*}
Hence, by the SLLN, it follows that ($(X,Y),(\tilde{X},\tilde{Y})\simiid P_{XY}$):
\begin{equation*}
\begin{aligned}
    \frac{1}{2(t-1)}\sum_{i=1}^{2(t-1)} \Exp{\tilde{X},\tilde{Y}}{\langle \varphi(\tilde{X}),\varphi(X_{i})\rangle_{\calG}\langle \psi(\tilde{Y}), \psi(Y_{i})\rangle_{\calH}} \overset{\mathrm{a.s.}}{\longrightarrow} \quad &\Exp{X,Y,\tilde{X},\tilde{Y}}{\langle \varphi(\tilde{X}),\varphi(X)\rangle_{\calG}\langle \psi(\tilde{Y}), \psi(Y)\rangle_{\calH}} \\
    &= \anglebrack{\mu_{XY}, \mu_{XY}}_{\calG\otimes\calH}.
\end{aligned}
\end{equation*}
Similarly, by the SLLN for U-statistics with bounded kernel~\citep{hoeffding1961slln_ustat}, it follows that:
\begin{equation*}
\begin{aligned}
    &\quad \frac{1}{4(t-1)^2-2(t-1)}\sum_{\substack{j,k=1:\\j\neq k}}^{2(t-1)}\Exp{\tilde{X}}{\langle \varphi(\tilde{X}),\varphi(X_{j})\rangle_{\calG}}\Exp{\tilde{Y}}{\langle \psi(\tilde{Y}),\psi(Y_{k})\rangle_{\calH}} \\
    \overset{\mathrm{a.s.}}{\longrightarrow} &\quad \Exp{X,\tilde{X}}{\langle \varphi(\tilde{X}),\varphi(X)\rangle_{\calG}}\Exp{Y,\tilde{Y}}{\langle \psi(\tilde{Y}),\psi(Y)\rangle_{\calH}} \\
    = & \quad  \anglebrack {\mu_{X}\otimes \mu_{Y}, \mu_{X}\otimes \mu_{Y}}_{\calG\otimes\calH}.
\end{aligned}
\end{equation*}
Hence, we deduce that:
\begin{equation*}
\begin{aligned}
    \anglebrack{\hat{\mu}_{XY}^{(t)}-\hat{\mu}_X^{(t)}\otimes\hat{\mu}_Y^{(t)},\mu_{XY}-\mu_{X}\otimes \mu_{Y}}_{\calG\otimes\calH} \ \overset{\mathrm{a.s.}}{\longrightarrow} & \ \anglebrack{\mu_{XY}, \mu_{XY}}_{\calG\otimes\calH} -  \anglebrack{\mu_{X}\otimes \mu_{Y},\mu_{X}\otimes \mu_{Y}}_{\calG\otimes\calH}\\
    = & \ \anglebrack{\mu_{XY}-\mu_{X}\otimes \mu_{Y}, \mu_{XY}-\mu_{X}\otimes \mu_{Y}}_{\calG\otimes\calH} \\
    = & \  \norm{\calG\otimes\calH}{\mu_{XY} - \mu_X\otimes\mu_Y}^2.
\end{aligned}
\end{equation*}
Recalling~\eqref{eq:suffice}, the proof of~\eqref{eq:target_ineq} is complete. To establish the consequence, simply note that:
\begin{equation*}
    \norm{\calG\otimes\calH}{\hat{g}_t - g^\star} = \sqrt{2\roundbrack{1-\anglebrack{\hat{g}_t,g^\star}_{\calG\otimes\calH}}},
\end{equation*}
and hence the result follows.
\end{proof}

\begin{lemma}\label{lem:partial_avg}
Suppose that $\roundbrack{x_t}_{t\geq 1}$ is a sequence of numbers such that $\lim_{t\to\infty}x_t = x$. Then the corresponding sequence of partial averages also converges to $x$, that is, $\lim_{n\to\infty}\frac{1}{n}\sum_{t=1}^n x_t = x$. This also implies that if $\roundbrack{X_t}_{t\geq 1}$ is a sequence of random variables such that $X_t\overset{\mathrm{a.s.}}{\longrightarrow}X$, then $(\sum_{t=1}^{n} X_t)/n\overset{\mathrm{a.s.}}{\longrightarrow}X$.
\end{lemma}
\begin{proof}
    Fix any $\varepsilon>0$. Since $\roundbrack{x_t}_{t\geq 1}$ is converging, then $\exists M>0$:
    \begin{equation*}
        \abs{x_t-x}\leq M, \quad \forall t \geq 1.
    \end{equation*}
    Now, let $n_0$ be such that $\abs{x_t-x}\leq \varepsilon/2$ for all $n>n_0$. Further, choose any $n_1>n_0$: $Mn_0/n_1\leq \varepsilon/2$. Hence, for any $\tilde{n}>n_1$, it holds that:
    \begin{equation*}
        \begin{aligned}
            \abs{\frac{1}{\tilde{n}}\sum_{t=1}^{\tilde{n}}x_t-x} &\leq \abs{\frac{1}{\tilde{n}}\sum_{t=1}^{n_0}x_t-x} + \abs{\frac{1}{\tilde{n}}\sum_{t=n_0+1}^{\tilde{n}}x_t-x}\\
            &\leq \frac{1}{\tilde{n}}\sum_{t=1}^{n_0}\abs{x_t-x} + \frac{1}{\tilde{n}}\sum_{t=n_0+1}^{\tilde{n}}\abs{x_t-x}\\
            &\leq \frac{n_0}{\tilde{n}}M + \frac{\tilde{n}-n_0}{\tilde{n}}\frac{\varepsilon}{2} ~\leq ~\frac{\varepsilon}{2}+\frac{\varepsilon}{2}
            ~=~ \varepsilon,
        \end{aligned}
    \end{equation*}
    which implies the result.
\end{proof}

Before we state the next result, recall that HSIC-based payoffs are based on the predictable estimates $\{\hat g_i\}_{i\geq 1}$ of the oracle witness function $g^\star$ and have the following form:
\begin{equation}\label{eq:v_t_def}
\begin{aligned}
    f_i(Z_{2i-1},Z_{2i}) &= \frac{1}{2}\roundbrack{\hat g_i(Z_{2i-1})+\hat g_i(Z_{2i})}-\frac{1}{2}\roundbrack{\hat g_i(\tilde{Z}_{2i-1})+\hat g_i(\tilde{Z}_{2i})}, \quad i\geq 1.\\
    f^\star(Z_{2i-1},Z_{2i}) &= \frac{1}{2}\roundbrack{g^\star(Z_{2i-1})+ g^\star(Z_{2i})}-\frac{1}{2}\roundbrack{g^\star(\tilde{Z}_{2i-1})+g^\star(\tilde{Z}_{2i})}.
\end{aligned}
\end{equation}

\begin{lemma}\label{lem:mom_convergence}
Suppose that $H_1$ in~\eqref{eq:alt_iid} is true. Then it holds that:
\begin{align}
    \frac{1}{t}\sum_{i=1}^t f_i(Z_{2i-1},Z_{2i}) &\overset{\mathrm{a.s.}}{\longrightarrow} \Exp{}{f^\star(Z_{1},Z_{2})}, \label{eq:first_moment}\\
    \frac{1}{t}\sum_{i=1}^t \roundbrack{f_i(Z_{2i-1},Z_{2i})}^2 &\overset{\mathrm{a.s.}}{\longrightarrow} \Exp{}{(f^\star(Z_{1},Z_{2}))^2}. \label{eq:second_moment}
\end{align}
\end{lemma}

\begin{proof}
We start by proving~\eqref{eq:first_moment}. Note that:
\begin{equation*}
\begin{aligned}
    f_i(Z_{2i-1},Z_{2i})  &=  \frac{1}{2}\roundbrack{\hat{g}_i(Z_{2i-1})+\hat{g}_t(Z_{2i})}-\frac{1}{2}\roundbrack{\hat{g}_i(\tilde{Z}_{2i-1})+\hat{g}_i(\tilde{Z}_{2i})} \\
    &= \frac{1}{2}\anglebrack{\hat{g}_i, \roundbrack{\varphi(X_{2i})-\varphi(X_{2i-1})}\otimes \roundbrack{\psi(Y_{2i})- \psi(Y_{2i-1})} }_{\calG\otimes\calH}.
\end{aligned}
\end{equation*}
Next, observe that:
\begin{equation*}
\begin{aligned}
    \abs{\frac{1}{t}\sum_{i=1}^t f_i(Z_{2i-1},Z_{2i}) - \Exp{}{f^\star(Z_{1},Z_{2})}} &\leq \abs{\frac{1}{t}\sum_{i=1}^t f_i(Z_{2i-1},Z_{2i}) - \frac{1}{t}\sum_{i=1}^t f^\star(Z_{2i-1},Z_{2i})} \\
    &+ \underbrace{\abs{\frac{1}{t}\sum_{i=1}^t f^\star(Z_{2i-1},Z_{2i}) - \Exp{}{f^\star(Z_{1},Z_{2})}}}_{\overset{\mathrm{a.s.}}{\longrightarrow} \ 0},
\end{aligned}
\end{equation*}
where the second term converges almost surely to 0 by the SLLN. For the first term, we have that:
\begin{equation*}
    \abs{\frac{1}{t}\sum_{i=1}^t f_i(Z_{2i-1},Z_{2i}) - \frac{1}{t}\sum_{i=1}^t f^\star(Z_{2i-1},Z_{2i})} \leq \frac{1}{t}\sum_{i=1}^t \abs{f_i(Z_{2i-1},Z_{2i}) - f^\star(Z_{2i-1},Z_{2i})}.
\end{equation*}
Finally, note that:
\begin{equation}\label{eq:der_chain}
\begin{aligned}
    \abs{f_i(Z_{2i-1},Z_{2i}) - f^\star(Z_{2i-1},Z_{2i})} &=\frac{1}{2} \abs{\anglebrack{\hat{g}_i-g^\star, \roundbrack{\varphi(X_{2i})-\varphi(X_{2i-1})}\otimes \roundbrack{\psi(Y_{2i})- \psi(Y_{2i-1})} }_{\calG\otimes\calH}}\\
    &\leq \norm{\calG\otimes\calH}{\hat{g}_i-g^\star} \overset{\mathrm{a.s.}}{\longrightarrow} 0,
\end{aligned}
\end{equation}
where the convergence result is due to Lemma~\ref{lem:alm_sure_conv_g_t}. The result~\eqref{eq:first_moment} then follows after invoking Lemma~\ref{lem:partial_avg}. Next, we prove~\eqref{eq:second_moment}. Note that:
\begin{equation*}
\begin{aligned}
    \frac{1}{t}\sum_{i=1}^t (f_i(Z_{2i-1},Z_{2i}))^2 &= \frac{1}{t}\sum_{i=1}^t (f_i(Z_{2i-1},Z_{2i})-f^\star(Z_{2i-1},Z_{2i})+f^\star(Z_{2i-1},Z_{2i}))^2\\
    &= \underbrace{\frac{1}{t}\sum_{i=1}^t(f_i(Z_{2i-1},Z_{2i})-f^\star(Z_{2i-1},Z_{2i}))^2}_{\overset{\mathrm{a.s.}}{\longrightarrow} \ 0}\\
    & + \frac{2}{t}\sum_{i=1}^t (f^\star(Z_{2i-1},Z_{2i}))(f_i(Z_{2i-1},Z_{2i})-f^\star(Z_{2i-1},Z_{2i}))\\
    &+ \underbrace{\frac{1}{t}\sum_{i=1}^t (f^\star(Z_{2i-1},Z_{2i}))^2}_{\overset{\mathrm{a.s.}}{\longrightarrow} \ \Exp{}{(f^\star(Z_{1},Z_{2}))^2}} ,
\end{aligned}
\end{equation*}
where the first convergence result is due to~\eqref{eq:der_chain} and Lemma~\ref{lem:partial_avg} and the second convergence result is due to the SLLN. Using~\eqref{eq:der_chain} and Lemma~\ref{lem:partial_avg}, we deduce that:
\begin{equation*}
    \abs{\frac{2}{t}\sum_{i=1}^t (f^\star(Z_{2i-1},Z_{2i}))(f_i(Z_{2i-1},Z_{2i})-f^\star(Z_{2i-1},Z_{2i}))}\leq 2\cdot \frac{1}{t}\sum_{i=1}^t \abs{f_i(Z_{2i-1},Z_{2i})-f^\star(Z_{2i-1},Z_{2i})} \overset{\mathrm{a.s.}}{\longrightarrow} 0,
\end{equation*}
and hence we conclude that the convergence~\eqref{eq:second_moment} holds.

\end{proof}

\subsubsection{Main Results}\label{appsubsubsec:sec_2_main}

\thmoracletest*
\begin{proof}
\begin{enumerate}
    \item Under $H_0$ in~\eqref{eq:null_iid}, we have that:
    \begin{equation*}
    (X_{2t-1},Y_{2t-1})\overset{d}{=} (X_{2t},Y_{2t})\overset{d}{=} (X_{2t-1},Y_{2t})\overset{d}{=} (X_{2t},Y_{2t-1}),
\end{equation*}
and hence, the first part of the Proposition trivially follows from the linearity of expectation. Under distribution drift, we use that at least one of the marginal distributions does not change at each round. For example, suppose that at round $t$, it holds that: $P_{X}^{2t-1}=P_{X}^{2t}$. For the stream of independent observations, we have: $X_{2t}\indep Y_{2t-1}$ and $X_{2t-1}\indep Y_{2t}$. Further, under the $H_0$ in~\eqref{eq:null_noniid_drift}, it holds that: $X_{2t-1}\indep Y_{2t-1}$ and $X_{2t}\indep Y_{2t}$. Hence, we have:
\begin{equation*}
    (X_{2t-1},Y_{2t-1})\overset{d}{=} (X_{2t},Y_{2t-1}) \quad \text{and}\quad (X_{2t-1},Y_{2t})\overset{d}{=} (X_{2t},Y_{2t}), 
\end{equation*}
and hence, we get the result using linearity of expectation.

\item Under the i.i.d. setting, we have 
\begin{equation*}
    \Exp{}{f^\star(Z_{2t-1},Z_{2t})\mid \calF_{t-1}} =\Exp{}{f^\star(Z_{1},Z_{2})}= s \cdot m(P_{XY};\calC),
\end{equation*}
and hence the result follows from the fact that the functional class $\calC$ satisfies the characteristic condition~\eqref{eq:charac_cond}.
\item Let $W:=f^\star(Z_{1},Z_{2})$, and consider $\Exp{H_1}{\log(1+\lambda W)}$. We know that $\Exp{H_1}{W}>0$. We use the following inequality~\citep[Equation (4.12)]{fan2015exp_ineq}: for any $y\geq -1$ and $\lambda\in[0,1)$, it holds:
\begin{equation*}
    \log(1+\lambda y) \geq \lambda y + y^2 \roundbrack{\log(1-\lambda)+\lambda}
\end{equation*}
Hence
\begin{equation*}
\begin{aligned}
    \Exp{}{\log(1+\lambda W)} &\geq \lambda \Exp{}{W} + \Exp{}{W^2} \roundbrack{\log(1-\lambda)+\lambda}.
\end{aligned}
\end{equation*}
Finally, using that $\log(1-x)+x\geq -x^2/(2(1-x))$ for $x\in [0,1)$, we get:
\begin{equation*}
\begin{aligned}
    \Exp{H_1}{\log(1+\lambda^\star W)} \geq \frac{(\Exp{H_1}{W})^2/2}{\Exp{H_1}{W}+\Exp{H_1}{W^2}} > 0,
\end{aligned}
\end{equation*}
where recall that:
\begin{equation*}
    \lambda^\star = \frac{\Exp{}{W}}{\Exp{}{W}+\Exp{}{W^2}}\in (0,1).
\end{equation*}
The wealth process corresponding to the oracle test satisfies:
\begin{equation*}
    \calK_t = \prod_{i=1}^{t} (1+\lambda^\star f^\star(Z_{2i-1},Z_{2i})) = \exp\roundbrack{t\cdot \frac{1}{t}\sum_{i=1}^t\log(1+\lambda^\star f^\star(Z_{2i-1},Z_{2i}))}.
\end{equation*}
By the Strong Law of Large Numbers (SLLN), we have:
\begin{equation*}
    \frac{1}{t}\sum_{i=1}^t\log(1+\lambda^\star f^\star(Z_{2i-1},Z_{2i})) \convas \Exp{}{\log(1+\lambda^\star W)}>0.
\end{equation*}
Hence, we get that $\calK_t \convas +\infty$, and hence, the oracle test is consistent.
\end{enumerate}
\end{proof}

\thmhsicskitconsistency*
\begin{remark}
    While it will be clear from the proof that the i.i.d.\ assumption is sufficient but not necessary for asymptotic power one, the more relaxed sufficient conditions are slightly technical to state and thus omitted.
\end{remark}

\begin{proof}
    \begin{enumerate}
    \item First, let us show that the predictable estimates of the oracle payoff function are bounded when the scaling factor $s=1/2$ is used. Recall that:
    \begin{equation}\label{eq:payoff_proof_rep}
\begin{aligned}
    f_{t}((x',y'),(x,y)) &= \frac{1}{2}\roundbrack{\hat{g}_t(x',y')-\hat{g}_t(x',y)+\hat{g}_t(x,y)-\hat{g}_t(x,y')}\\
    &= \frac{1}{2}\anglebrack{\hat{g}_t,\varphi(x')\otimes\psi(y')-\varphi(x')\otimes\psi(y)+\varphi(x)\otimes\psi(y)-\varphi(x)\otimes\psi(y')}_{\calG\otimes\calH}\\
    &= \frac{1}{2}\anglebrack{\hat{g}_t,\roundbrack{\varphi(x')-\varphi(x)}\otimes \roundbrack{\psi(y')-\psi(y)}}_{\calG\otimes\calH}.
\end{aligned}
\end{equation}
Note that:
\begin{equation*}
\begin{aligned}
\abs{f_{t}((x',y'),(x,y))} & \leq  \frac{1}{2}\norm{\calG\otimes\calH}{\hat{g}_t}\norm{\calG\otimes\calH}{(\varphi(x')-\varphi(x))\otimes(\psi(y')-\psi(y))}\\
     &\leq \frac{1}{2}\norm{\calG\otimes\calH}{(\varphi(x')-\varphi(x))\otimes(\psi(y')-\psi(y))}\\
     & = \frac{1}{2}\norm{\calG}{\varphi(x')-\varphi(x)}\cdot \norm{\calH}{\psi(y')-\psi(y)}\\
     &= \frac{1}{2}\sqrt{2(1-k(x',x))}\cdot \sqrt{2(1-l(y',y))}\\
     & =1.
\end{aligned}
\end{equation*}
and hence, $f_{t}((x',y'),(x,y))\leq [-1,1]$. Next, we show that constructed payoff function yields a fair bet. Indeed, we have that:
\begin{equation*}
    \Exp{}{f_{t}(Z_{2t-1},Z_{2t})\mid \calF_{t-1}} = \anglebrack{\hat{g}_t,\mu_{XY}-\mu_X\otimes\mu_Y}_{\calG\otimes\calH},
\end{equation*}
and in particular, the above implies that $\Exp{H_0}{f_{t}(Z_{2t-1},Z_{2t})\mid \calF_{t-1}} = 0$ for $H_0$ in~\eqref{eq:null_iid}. For $H_0$ in~\eqref{eq:null_noniid_drift}, it is easy to see that the result holds using the form~\eqref{eq:payoff_proof_rep}. We use that $X_{2t-1}\indep Y_{2t-1}$, $X_{2t}\indep Y_{2t}$, $X_{2t}\indep Y_{2t-1}$, $X_{2t-1}\indep Y_{2t}$, and the fact that at least one of the marginal distributions does not change.
Next, we show that for all strategies for selecting betting fractions that are considered in this work, the resulting wealth process is a nonnegative martingale. In case aGRAPA/ONS strategies are used, the resulting wealth process is clearly a nonnegative martingale since betting fractions are predictable. The mixed wealth process $\roundbrack{\calK^{\mathrm{mixed}}_t}_{t\geq 1}$ is a nonnegative martingale under the null $H_0$, and hence
\begin{equation*}
    \begin{aligned}
        \Exp{H_0}{\calK^{\mathrm{mixed}}_t\mid \calF_{t-1}} &= \Exp{}{\int_0^1\calK_{t-1}^\lambda(1+\lambda f_t(Z_{2t-1},Z_{2t}))\nu(\lambda)d\lambda\mid \calF_{t-1} }\\
        &= \int_0^1\calK_{t-1}^\lambda\Exp{H_0}{1+\lambda f_t(Z_{2t-1},Z_{2t})\mid \calF_{t-1}} \nu(\lambda)d\lambda\\
        &= \int_0^1\calK_{t-1}^\lambda \nu(\lambda)d\lambda\\
        &= \calK^{\mathrm{mixed}}_{t-1},
    \end{aligned}
\end{equation*}
where the interchange of conditional expectation and integration is justified by the conditional monotone convergence theorem. The assertion of the Theorem then follows directly from Ville's inequality (Proposition~\ref{thm:villes_ineq}) when $a=1/\alpha$.

    \item Next, we establish the consistency of HSIC-based SKIT with ONS betting strategy.  
Under the ONS betting strategy, for any sequence of outcomes $\roundbrack{f_i}_{i\geq 1}$, $f_i\in[-1,1]$, $i\geq 1$, it holds that (see the proof of Theorem 1 in~\citep{cutkosky2018bb_reductions}):
\begin{equation}\label{eq:ons_regret}
    \log \calK_t(\lambda_0)-\log \calK_t = O\roundbrack{\log\roundbrack{\sum_{i=1}^t f_i^2}},
\end{equation}
where $\calK_t(\lambda_0)$ is the wealth of any constant betting strategy $\lambda_0\in[-1/2,1/2]$ and $\calK_t$ is the wealth corresponding to the ONS betting strategy. It follows that the wealth process corresponding to the ONS betting strategy satisfies
\begin{equation}\label{eq:ons_lower_bound}
    \frac{\log \calK_t}{t} \geq \frac{\log \calK_t(\lambda_0)}{t} - C\cdot \frac{\log t}{t},
\end{equation}
for some absolute constant $C>0$. Next, let us consider:
\begin{equation*}
    \lambda_0 = \frac{1}{2} \roundbrack{\roundbrack{\frac{\sum_{i=1}^t f_i}{\sum_{i=1}^t f_i^2} \wedge 1} \vee 0}.
\end{equation*}
We obtain:
\begin{equation}\label{eq:ons_general_lb}
\begin{aligned}
    \frac{\log\calK_t(\lambda_0)}{t} &= \frac{1}{t}\sum_{i=1}^t \log(1+\lambda_0 f_i) \\
    &\overset{\mathrm{(a)}}{\geq} \frac{1}{t}\sum_{i=1}^t (\lambda_0 f_i - \lambda_0^2 f_i^2)\\
    &= \roundbrack{\frac{\frac{1}{t}\sum_{i=1}^t f_i}{4}\vee 0 }\cdot \roundbrack{\frac{\frac{1}{t}\sum_{i=1}^t f_i}{\frac{1}{t}\sum_{i=1}^t f_i^2} \wedge 1},
\end{aligned}
\end{equation}
where in (a) we used\footnote{A slightly better constant for the growth rate (0.3 in place of 1/4) can be obtained by using the inequality: $\log(1+x)\geq x-\frac{5}{6}x^2$, that holds $\forall x\in [-0.5,0.5]$.} that $\log(1+x)\geq x-x^2$ for $x\in [-1/2,1/2]$. From Lemma~\ref{lem:mom_convergence}, it follows for $f_i=f_i(Z_{2i-1},Z_{2i})$ that:
\begin{equation}
    \frac{\frac{1}{t}\sum_{i=1}^t f_i(Z_{2i-1},Z_{2i})}{4}\cdot \roundbrack{\frac{\frac{1}{t}\sum_{i=1}^t f_i(Z_{2i-1},Z_{2i})}{\frac{1}{t}\sum_{i=1}^t (f_i(Z_{2i-1},Z_{2i}))^2} \wedge 1} \overset{\mathrm{a.s.}}{\longrightarrow} \frac{\Exp{}{f^\star(Z_{1},Z_{2})}}{4}\cdot \roundbrack{\frac{\Exp{}{f^\star(Z_{1},Z_{2})}}{\Exp{}{(f^\star(Z_{1},Z_{2}))^2}} \wedge 1}.
\end{equation}
Further, note that: 
\begin{equation*}
    \Exp{}{f^\star(Z_{1},Z_{2})} =\norm{\calG\otimes\calH}{\mu_{XY} - \mu_X\otimes\mu_Y} = \sqrt{\mathrm{HSIC}(P_{XY};\calG,\calH)},
\end{equation*}
which is positive if the $H_1$ is true. Hence, using~\eqref{eq:ons_lower_bound}, we deduce that the growth rate of the ONS wealth process satisfies
\begin{equation}
    \liminf_{t\to\infty} \frac{\log \calK_t}{t} \geq \frac{\Exp{}{f^\star(Z_{1},Z_{2})}}{4}\cdot \roundbrack{\frac{\Exp{}{f^\star(Z_{1},Z_{2})}}{\Exp{}{(f^\star(Z_{1},Z_{2}))^2}} \wedge 1}.
\end{equation}
We conclude that the test is consistent, that is, if $H_1$ is true, then $\Prob(\tau<\infty)=1$.

\end{enumerate}
\end{proof}

\proporaclebounds*
    
    \begin{proof}
    We start by establishing the upper bound in~\eqref{eq:oracle_lw_bounds}. 
    The fact that $S^\star\leq \Exp{}{f^\star(Z_{1},Z_{2})}/2$ trivially follows from $\Exp{}{\log(1+\lambda f^\star(Z_{1},Z_{2}))}\leq \lambda \Exp{}{f^\star(Z_{1},Z_{2})}\leq \Exp{}{f^\star(Z_{1},Z_{2})}/2$. Since for any $x\in[-0.5,0.5]$, it holds that: $ \log(1+x)\leq x-3x^2/8$, we know that:
    \begin{equation}\label{eq:ub}
    \begin{aligned}
        S^\star &\leq \max_{\lambda\in [-0.5,0.5]} \roundbrack{\lambda \Exp{}{f^\star(Z_{1},Z_{2})}-\frac{3}{8}\lambda^2\Exp{}{(f^\star(Z_{1},Z_{2}))^2}},
    \end{aligned}
    \end{equation}
and by solving the maximization problem, we get the upper bound:
\begin{equation}\label{eq:orig_up_bound}
    S^\star \leq \frac{2}{3} \frac{(\Exp{}{f^\star(Z_1,Z_2)})^2}{\Exp{}{(f^\star(Z_1,Z_2))^2}},
\end{equation}
assuming $(\Exp{}{f^\star(Z_1,Z_2)})^2/\Exp{}{(f^\star(Z_1,Z_2))^2}\leq 3/8$. On the other hand, it always holds that: $S^\star \leq \Exp{}{f^\star(Z_1,Z_2)}/2$. To obtain the claimed bound, we multiply the RHS of~\eqref{eq:orig_up_bound} by two, which completes the proof of~\eqref{eq:oracle_lw_bounds}.

        \end{proof}

\thmnoniid*

\begin{proof}

Recall that at round $t$, the payoff function has form:
\begin{align*}
    f_t((X_{2t-1},Y_{2t-1}),(X_{2t},Y_{2t})) = \frac{1}{2}\langle \hat{g}_t, (\varphi(X_{2t})-\varphi(X_{2t-1})) \otimes (\psi(Y_{2t})-\psi(Y_{2t-1})) \rangle_{\calG\otimes\calH}.
\end{align*}
Let $\calD_t = \curlybrack{(X_i, Y_i)}_{i\leq 2(t-1)}$. To establish validity, we need to show that under $H_0$ in~\eqref{eq:null_noniid},
\begin{equation}\label{eq:suf_cond_non_iid}
    \Exp{}{f_t((X_{2t-1},Y_{2t-1}),(X_{2t},Y_{2t}))\mid \calD_t}=0,
\end{equation}
and hence it suffices to show that:
\begin{equation*}
    \Exp{}{(\varphi(X_{2t})-\varphi(X_{2t-1})) \otimes (\psi(Y_{2t})-\psi(Y_{2t-1}))\mid \calD_t}=0.
\end{equation*}
Due to independence under the null $H_0$, we have:
\begin{align*}
    \Exp{}{\varphi(X_{2t-1})\otimes \psi(Y_{2t-1})\mid \calD_t} &= \Exp{}{\varphi(X_{2t-1})\mid \calD_t}\otimes \Exp{}{\psi(Y_{2t-1})\mid \calD_t}=:\mu_X^{2t-1} \otimes \mu_Y^{2t-1},\\
    \Exp{}{\varphi(X_{2t})\otimes \psi(Y_{2t})\mid \calD_t} &= \Exp{}{\varphi(X_{2t})\mid \calD_t}\otimes \Exp{}{\psi(Y_{2t})\mid \calD_t}=: \mu_X^{2t} \otimes \mu_Y^{2t},
\end{align*}
Consider one of the cross-terms $\varphi(X_{2t})\otimes \psi(Y_{2t-1})$. We have the following:
\begin{align*}
    \Exp{}{\varphi(X_{2t})\otimes \psi(Y_{2t-1})\mid \calD_t} &\overset{a}{=} \Exp{}{\Exp{}{\varphi(X_{2t})\otimes \psi(Y_{2t-1})\mid X_{2t-1},\calD_t}\mid \calD_t}\\
    &\overset{b}{=} \Exp{}{\Exp{}{\varphi(X_{2t})\mid X_{2t-1},\calD_t}\otimes \Exp{}{ \psi(Y_{2t-1})\mid X_{2t-1},\calD_t}\mid \calD_t}\\
    &\overset{c}{=} \Exp{}{\Exp{}{\varphi(X_{2t})\mid X_{2t-1},\calD_t}\otimes \Exp{}{ \psi(Y_{2t-1})\mid \calD_t}\mid \calD_t}\\
    &\overset{d}{=} \Exp{}{\Exp{}{\varphi(X_{2t})\mid X_{2t-1},\calD_t}\mid \calD_t}\otimes \Exp{}{ \psi(Y_{2t-1})\mid \calD_t}\\
    &\overset{e}{=} \Exp{}{\varphi(X_{2t})\mid \calD_t}\otimes \Exp{}{ \psi(Y_{2t-1})\mid \calD_t}\\
    &\overset{f}{=} \mu_X^{2t} \otimes \mu_Y^{2t-1}.
\end{align*}
In the above, (a) uses the law of iterated expectations and conditioning on $X_{2t-1}$, (b) uses the assumption~\eqref{eq:cross_links} about conditional independence, (c) uses the independence null assumption~\eqref{eq:null_iid}, (d) uses that $\Exp{}{ \psi(Y_{2t-1})\mid \calD_t}$ is $\sigma(\calD_t)$-measurable, (e) uses the law of iterated expectations, and (f) uses the definitions of the mean embeddings of conditional distributions. An analogous argument can be used to deduce:
\begin{equation*}
    \Exp{}{\varphi(X_{2t-1})\otimes \psi(Y_{2t})\mid\calD_t} = \mu_X^{2t-1} \otimes \mu_Y^{2t}.
\end{equation*}
We get that:
\begin{align*}
    \Exp{}{(\varphi(X_{2t})- \varphi(X_{2t-1}))\otimes (\psi(Y_{2t})- \psi(Y_{2t-1}))\mid\calD_t} &= \mu_X^{2t-1} \otimes \mu_Y^{2t-1}+\mu_X^{2t} \otimes \mu_Y^{2t}-\mu_X^{2t-1} \otimes \mu_Y^{2t}-\mu_X^{2t} \otimes \mu_Y^{2t-1}\\
     &= (\mu_X^{2t}-\mu_X^{2t-1}) \otimes (\mu_Y^{2t}-\mu_Y^{2t-1}),
\end{align*}
and hence, if either $(X_{2t-1}, X_{2t})$ or $(Y_{2t-1}, X_{2t})$ are exchangeable conditional on $\calD_t$, it follows that either $\mu_X^{2t}=\mu_X^{2t-1}$ or $\mu_Y^{2t}=\mu_Y^{2t-1}$ respectively. This, in turn, implies that~\eqref{eq:suf_cond_non_iid} holds, and hence, the result follows.

\end{proof}

\subsection{Proofs for Section~\ref{sec:alt_measures}}\label{appsubsec:proofs_sec_3}

\thmvaliditycocokcc*

\begin{proof}
 It suffices to show that the proposed payoff functions are bounded. The rest of the proof follows will follow the same steps as the proof of Theorem~\ref{thm:hsic_skit} (for a stream of independent observations) or Theorem~\ref{thm:non_iid} (for time-varying independence null), and we omit the details. Note that:
\begin{equation*}
\begin{aligned}
    \abs{\hat{h}_{t}(y')-\hat{h}_{t}(y)} &= \abs{\langle \hat{h}_{t}, \psi(y')\rangle_{\calH}-\langle \hat{h}_{t}, \psi(y)\rangle_{\calH}}\\
    &= \abs{\langle \hat{h}_{t}, \psi(y')- \psi(y)\rangle_{\calH}}\\
    &\leq \norm{\calH}{\hat{h}_{t}}\norm{\calH}{\psi(y')- \psi(y)}\\
    &\leq\norm{\calH}{\psi(y')- \psi(y)}\\
    &=\sqrt{2(1-l(y,y'))}\\
    &\leq \sqrt{2},
\end{aligned}
\end{equation*}
where we used that $\norm{\calH}{\hat{h}_{t}}\leq 1$ due to normalization. Analogous bound holds for $\abs{\hat{g}_{t}(x')-\hat{g}_{t}(x)}$. We conclude that any predictable estimate of the oracle payoff function for COCO (or KCC) satisfies
\begin{equation*}
    \abs{f_t ((x',y'),(x,y))}\leq 1,
\end{equation*}
as proposed. The fact that the payoff function is fair trivially follows from the definition. Regarding the existence of the oracle payoff, whose mean is positive under $H_1$ in~\eqref{eq:alt_iid}, note that if $k$ and $l$ are characteristic kernels, then COCO and KCC satisfy the characteristic condition~\eqref{eq:charac_cond}; see \citet{bach2003kernel_cca,gretton2005kernel_coco,gretton05a}. Hence, the result follows from Theorem~\ref{thm:oracle_test}. This completes the proof.

\end{proof}

\subsection{Proofs for Section~\ref{sec:alt_bet}}\label{appsubsec:proofs_alt_bet}

\thmvaliditysym*

\begin{proof}
For any $t\geq 1$, we have that the payoffs defined in~\eqref{eq:tanh_bet}, \eqref{eq:rank_bet}, and~\eqref{eq:pred_bet} are bounded: $f_{t}(w)\in [-1,1]$, $\forall w\in \Real$. Due to Proposition~\ref{prop:symmetry}, we know that, under the null, $W_t$ is a random variable that is symmetric around zero (conditional on $\calF_{t-1}$). Hence, for the composition approach, it trivially follows that $\Exp{H_0}{f^{\mathrm{odd}}_t(W_t)\mid \calF_{t-1}}=0$ since a composition with an odd function is used. For the rank and predictive approaches, we use the fact that, under the null, $\mathrm{sign}(W_t)\indep \abs{W_t}\mid \calF_{t-1}$. Since, $\Exp{H_0}{\mathrm{sign}(W_t)\mid \calF_{t-1}}=0$, it then follows that $\Exp{H_0}{f^{\mathrm{rank}}_t(W_t)\mid \calF_{t-1}}=0$. Using that $\mathrm{sign}(W_t)\indep \abs{W_t}\mid \calF_{t-1}$ and by conditioning on the sign of $W_t$, we get:
\begin{equation*}
\begin{aligned}
    \Exp{H_0}{\ell_t(W_t)\mid\calF_{t-1}} &= \frac{1}{2}\Prob_{H_0}\roundbrack{p_t(\abs{W_t})\geq 1/2}+\frac{1}{2}\Prob_{H_0}\roundbrack{p_t(\abs{W_t})< 1/2} =\frac{1}{2}.
\end{aligned}
\end{equation*}
Hence $\Exp{H_0}{1-2\ell_t(W_t)\mid\calF_{t-1}}=0$. The rest of the proof regarding the validity of the symmetry-based SKITs follows the same steps as the proof of Theorem~\ref{thm:hsic_skit}, and we omit the details.

\end{proof}

\clearpage

\section{Selecting Betting Fractions}\label{appsec:betting_fractions}

As alluded to in Remark~\ref{rmk:kelly_bet}, sticking to a single fixed betting fraction, $\lambda_t=\lambda\in[0,1]$, $t\geq 1$, may result in a wealth process that either has a sub-optimal growth rate under the alternative or tends to zero almost surely (see Figure~\ref{fig:two_seeds}). \emph{Mixing} over different betting fractions is a simple approach that often works well in practice. Given a fine grid of values: $\Lambda = \curlybrack{\lambda^{(1)},\dots,\lambda^{(J)}}$, e.g., uniformly spaced values on the unit interval, consider 
\begin{equation}\label{eq:mix_method_wealth}
    \calK_t^{\mathrm{mixed}} = \frac{1}{\abs{\Lambda}} \sum_{\lambda^{(j)}\in\Lambda} \calK_t(\lambda^{(j)}),
\end{equation}
where $\roundbrack{\calK_t(\lambda^{(j)})}_{t\geq 0}$ is a wealth process corresponding to a constant-betting strategy with betting fraction $\lambda^{(j)}$ \footnote{Practically, it is advisable to start with a coarse grid (small $J$) at small $t$ and occasionally add another grid point, so that the grid becomes finer over time. Whenever a grid point is added, it is like adding another stock to a portfolio, and the wealth must be appropriately redistributed; we omit the details for brevity.}.

\begin{figure}[!htb]
\begin{center}
        \subfigure[]{
                \includegraphics[width=0.425\linewidth]{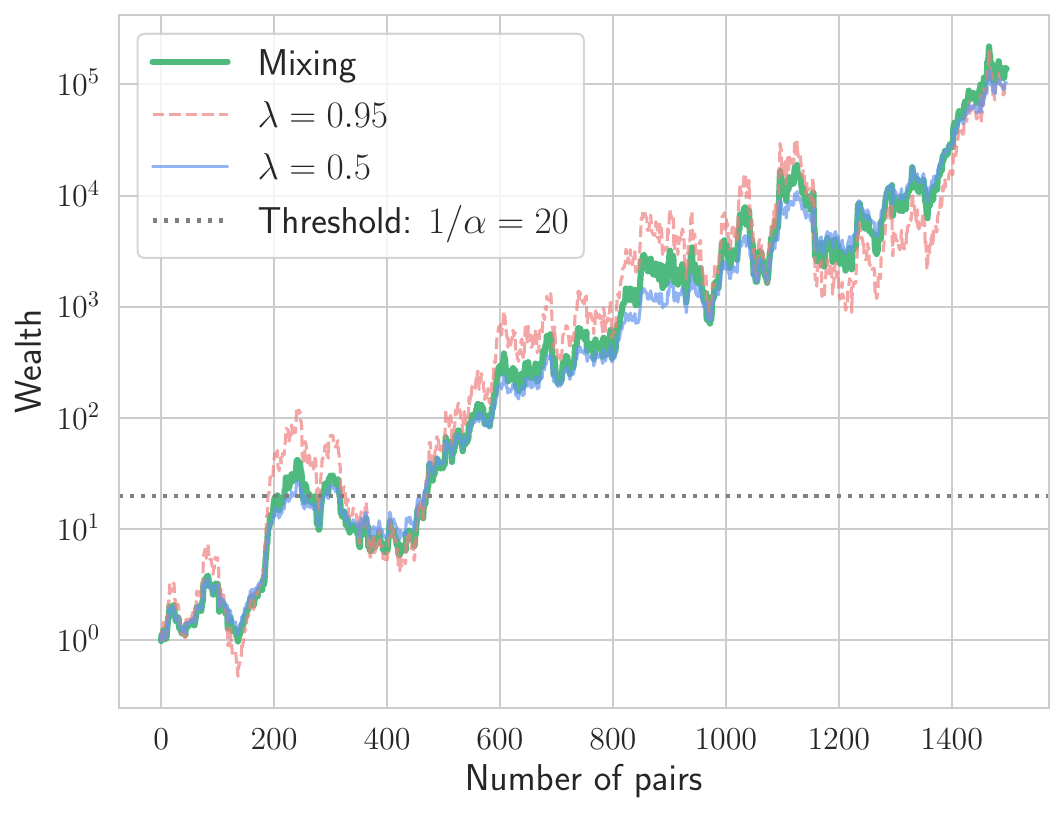}}
        \subfigure[]{
            \includegraphics[width=0.425\linewidth]{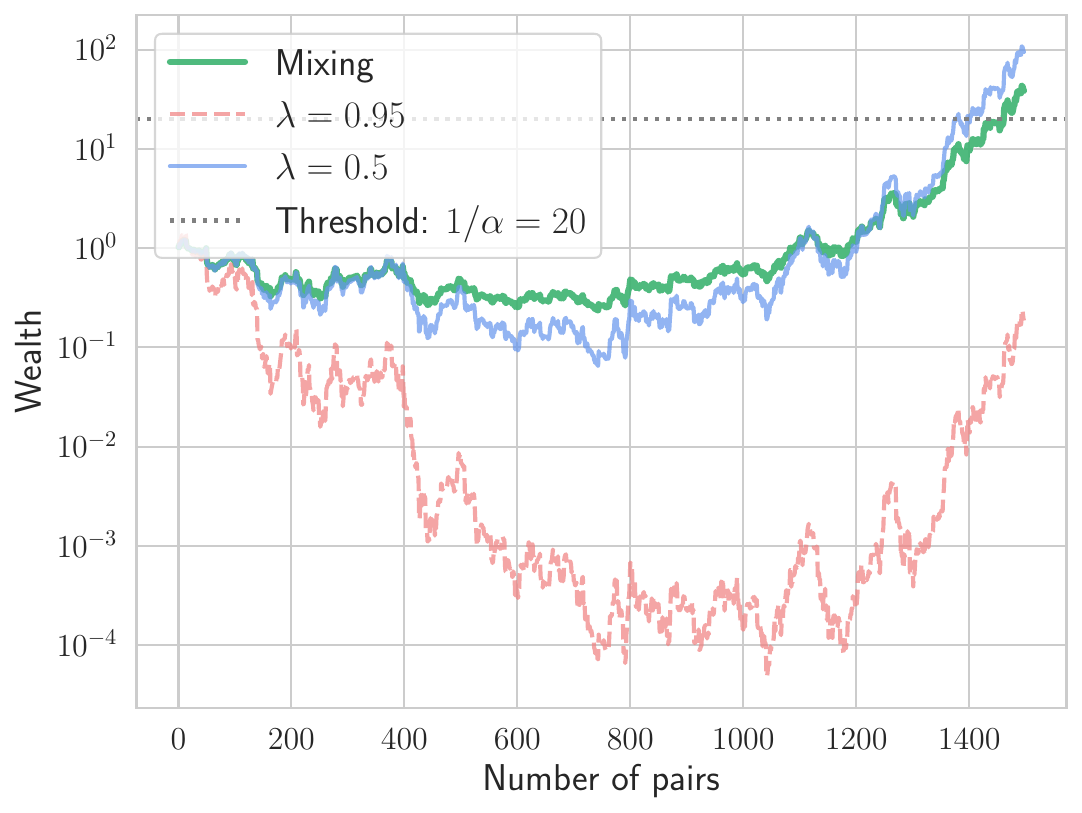}}
        \caption{SKIT with HSIC payoff function on two particular realizations of streams of dependent data: $Y_t=0.1\cdot X_t+\varepsilon_t$, $X_t,\varepsilon_t\sim \mathcal{N}(0,1)$. For both cases, we consider a mixed wealth process for $\Lambda = \curlybrack{0.05,0.1,\dots,0.95}$. We observe that the mixed wealth process follows closely the best of constant-betting strategies with $\lambda\in\curlybrack{0.5,0.95}$.}
        \label{fig:two_seeds}
\end{center}
\end{figure}

While mixing often works well in practice, it introduces additional tuning hyperparameters, e.g., grid size. We consider two compelling approaches for the selection of betting fractions in a predictable way, meaning that $\lambda_t$ depends only on $\curlybrack{(X_i, Y_i)}_{i\leq 2(t-1)}$. In addition to the ONS strategy (Algorithm~\ref{alg:ons}), we also consider \emph{aGRAPA} strategy (Algorithm~\ref{alg:aGRAPA}). The idea that effective betting strategies are ones that maximize a gambler's expected log capital dates back to early works of~\citet{kelly1956new_interp} and~\citet{breiman1962optimal_gambling}. Assuming that the same betting fraction is used, the log capital after round $(t-1)$ is
\begin{equation*}
\begin{aligned}
    \log \calK_{t-1}(\lambda) = \sum_{i=1}^{t-1} \log\roundbrack{1+\lambda f_{i}(Z_{2i-1},Z_{2i})}.
\end{aligned}
\end{equation*}

\begin{algorithm}[!htb]
\caption{aGRAPA strategy for selecting betting fractions}\label{alg:aGRAPA}
\begin{algorithmic}
\State \textbf{Input:} sequence of payoffs $\roundbrack{f_t(Z_{2t-1},Z_{2t})}_{t\geq 1}$, $\lambda^{\mathrm{aGRAPA}}_1 = 0$, $\mu_0^{(1)}=0$, $\mu_0^{(2)}=1$, $c=0.9$.
\For {$t=1,2,\dots$}
\State Set $\mu^{(1)}_t = \mu^{(1)}_{t-1}+f_t(Z_{2t-1},Z_{2t})$;
\State Set $\mu^{(2)}_t = \mu^{(2)}_{t-1} + (f_t(Z_{2t-1},Z_{2t}))^2$;
\State Set $\lambda^{\mathrm{aGRAPA}}_{t+1} = c \wedge\roundbrack{0\vee\roundbrack{\mu^{(1)}_t/\mu^{(2)}_t}}$;
\EndFor
\end{algorithmic}
\end{algorithm}

Following \citet{ws2020est_means}, we set the derivative to zero and use Taylor's expansion 
to get
\begin{equation*}
    \lambda^{\mathrm{aGRAPA}}_{t} =  \roundbrack{\roundbrack{\frac{\sum_{i=1}^{t-1} f_{i}(Z_{2i-1},Z_{2i})}{\sum_{i=1}^{t-1} \roundbrack{f_{i}(Z_{2i-1},Z_{2i})}^2}}\vee 0}\wedge c.
\end{equation*}
Truncation at zero is inspired by the fact that $\Exp{H_1}{f^\star(Z_{2t-1},Z_{2t})\mid\calF_{t-1}}>0$, whereas truncation at $c\in(0,1]$ (e.g., $c=0.9$) is necessary to guarantee that the wealth process is indeed nonnegative.

\clearpage

\section{Omitted Details for Sections~\ref{sec:seq_indep_testing} and~\ref{sec:alt_measures}}\label{appsec:ommited_der}

In this section, we complement the material presented in the main paper by deriving
the forms of the witness functions for the dependence criteria considered in this work.



\paragraph{Oracle Witness Function for HSIC.} Let us derive the form of the oracle witness function for HSIC. Note that:
\begin{equation*}
    \begin{aligned}
        &\sup_{g:\|g\|_{\calG\otimes\calH}\leq 1}\squarebrack{\Exp{P_{XY}}{g(X,Y)}-\Exp{P_X\times P_Y}{g(X',Y')}} \\
        =\quad & \sup_{g:\|g\|_{\calG\otimes\calH}\leq 1}\squarebrack{\Exp{P_{XY}}{\langle g, \varphi(X)\otimes \psi(Y)\rangle_{\calG\otimes\calH}}-\Exp{P_X\times P_Y}{\langle g, \varphi(X')\otimes \psi(Y')\rangle_{\calG\otimes\calH}}}\\ 
       =\quad & \sup_{g:\|g\|_{\calG\otimes\calH}\leq 1}\squarebrack{\langle g, \Exp{P_{XY}}{\varphi(X)\otimes \psi(Y)}\rangle_{\calG\otimes\calH}-\langle g,\Exp{P_X\times P_Y}{ \varphi(X')\otimes \psi(Y')}\rangle_{\calG\otimes\calH}}\\
       =\quad & \sup_{g:\|g\|_{\calG\otimes\calH}\leq 1}\squarebrack{\langle g, \mu_{XY}\rangle_{\calG}-\langle g, \mu_{X}\otimes\mu_{Y}\rangle_{\calG\otimes\calH}}\\
       =\quad &\sup_{g:\|g\|_{\calG\otimes\calH}\leq 1}\langle g,\mu_{XY}-\mu_{X}\otimes\mu_{Y}\rangle_{\calG\otimes\calH},
    \end{aligned}
\end{equation*}
from which it is easy to derive the oracle witness function for HSIC.

\begin{remark}\label{rmk:operator_fn_relation}
Note that in~\eqref{eq:hsic_plugin_wf} the witness function is defined as an operator: $\hat{g}_t:\calX\times\calY\to\Real$. To clarify, for any $z=(x,y)\in\calX\times\calY$, we have
\begin{equation*}
\begin{aligned}
    (\hat{\mu}_{XY}-\hat{\mu}_X\otimes\hat{\mu}_Y)(z) 
    = \ & \frac{1}{2(t-1)} \sum_{i=1}^{2(t-1)} k(X_i,x) l(Y_i,y) -  \ \roundbrack{\frac{1}{2(t-1)} \sum_{i=1}^{2(t-1)} k(X_i,x)} \cdot \roundbrack{\frac{1}{2(t-1)}\sum_{i=1}^{2(t-1)} l(Y_i,y)},
\end{aligned}
\end{equation*}
and the denominator in~\eqref{eq:hsic_plugin_wf} can be expressed in terms of kernel matrices $K,L\in\Real^{2(t-1)\times 2(t-1)}$ with entries $K_{ij} = k(X_i,X_j)$, $L_{ij} = l(Y_i,Y_j)$, $i,j\in\curlybrack{1,\dots,2(t-1)}$, as:
\begin{equation*}
    \norm{\calG\otimes\calH}{\hat{\mu}_{XY}-\hat{\mu}_X\otimes\hat{\mu}_Y} = \frac{1}{2(t-1)}\sqrt{\mathrm{tr}(KHLH)},
\end{equation*}
where $H=\textbf{I}_{2(t-1)}-(1/(2(t-1))\textbf{1}\textbf{1}^\top$ is the centering projection matrix.
\end{remark}

\begin{remark}
    While the empirical witness functions for COCO/KCC~\eqref{eq:coco_witness_fns} are defined as operators, we use those as functions in the definition of the corresponding payoff function. To clarify, for any $x\in\calX$ and $y\in\calY$, we have
    \begin{equation*}
    \begin{aligned}
        \hat{g}_t(x) &= \sum_{i=1}^{2(t-1)} \alpha_i \roundbrack{k(X_i,x)-\frac{1}{2(t-1)}\sum_{j=1}^{2(t-1)} k(X_j,x)},\\ 
        \hat{h}_t(y) &= \sum_{i=1}^{2(t-1)} \beta_i \roundbrack{l(Y_i,y)-\frac{1}{2(t-1)}\sum_{j=1}^{2(t-1)} l(Y_j,y)}.
    \end{aligned}
\end{equation*}
\end{remark}

\paragraph{Minibatched Payoff Function for HSIC.} The minibatched payoff function at round $t$ has the following form:
    \begin{equation*}
        f_t(Z_{b(t-1)+1},\dots,Z_{bt})= \frac{1}{b}\sum_{i=1}^b \hat{g}_t (X_{b(t-1)+i},Y_{b(t-1)+i}) - \frac{1}{b(b-1)} \sum_{\substack{i,j=1 \\ i\neq j}}^{b} \hat{g}_t (X_{b(t-1)+i},Y_{b(t-1)+j}).
    \end{equation*}
Note that:
    \begin{align*}
        f_t(Z_{b(t-1)+1},\dots,Z_{bt}) &= \frac{1}{b}\sum_{i=1}^b \langle \hat{g}_t, \varphi(X_{b(t-1)+i})\otimes\psi(Y_{b(t-1)+i})\rangle_{\calG\otimes\calH}\\
        &- \frac{1}{b(b-1)} \sum_{\substack{i,j=1 \\ i\neq j}}^{b} \langle \hat{g}_t, \varphi(X_{b(t-1)+i})\otimes\psi(Y_{b(t-1)+j})\rangle_{\calG\otimes\calH}\\
        &=\anglebrack{\hat{g}_t, \frac{1}{2b(b-1)}\sum_{\substack{i,j=1\\ i\neq j}}^b \roundbrack{\varphi(X_{b(t-1)+i})-\varphi(X_{b(t-1)+j})}\otimes\roundbrack{\psi(Y_{b(t-1)+i}) -\psi(Y_{b(t-1)+j})}}_{\calG\otimes\calH}.
    \end{align*}
    Let $\calF_{t-1}' = \sigma(\curlybrack{(X_{i},Y_{i})}_{i\leq b(t-1)})$. We have that:
    \begin{equation*}
        \Exp{}{f_t(Z_{b(t-1)+1},\dots,Z_{bt})\mid \calF_{t-1}'} = \langle \hat{g}_t, \mu_{XY}-\mu_X\otimes\mu_Y\rangle_{\calG\otimes\calH},
    \end{equation*}
    and in particular, $\Exp{H_0}{f_t(Z_{b(t-1)+1},\dots,Z_{bt})\mid \calF_{t-1}'}=0$ if the null $H_0$ in~\eqref{eq:null_iid} is true. It suffices to show that the payoff is bounded. Since $\norm{\calG\otimes\calH}{\hat{g}_t}=1$, we can easily deduce that:
    \begin{align*}
        \abs{f_t(Z_{b(t-1)+1},\dots,Z_{bt})} &\leq \frac{1}{2b(b-1)}\sum_{\substack{i,j=1\\ i\neq j}}^b \norm{\calG\otimes\calH}{\roundbrack{\varphi(X_{b(t-1)+i})-\varphi(X_{b(t-1)+j})}\otimes\roundbrack{\psi(Y_{b(t-1)+i}) -\psi(Y_{b(t-1)+j})}}\\
        &= \frac{1}{2b(b-1)}\sum_{\substack{i,j=1\\ i\neq j}}^b \norm{\calG}{\varphi(X_{b(t-1)+i})-\varphi(X_{b(t-1)+j})}\norm{\calH}{\psi(Y_{b(t-1)+i}) -\psi(Y_{b(t-1)+j})}\\
        &= \frac{1}{2b(b-1)}\sum_{\substack{i,j=1\\ i\neq j}}^b \sqrt{2(1-k(X_{b(t-1)+i},X_{b(t-1)+j}))}\sqrt{2(1-l(Y_{b(t-1)+i},Y_{b(t-1)+j}))}\\
        &\leq 1.
    \end{align*}
Hence, we conclude that the wealth process constructed using a minibatched version of the payoff function is also a nonnegative martingale.

\begin{example}\label{example:dist_drift_violation}
For $t\geq 1$, consider
\begin{equation*}
    (X_t,Y_t) = \roundbrack{\frac{V_t+1-1/t}{2}, \frac{V_t'+1- 1/t}{2}},
\end{equation*}
where $V_t,V_t'\simiid \mathrm{Ber}(1/2)$. Note that $\calX=\calY\subseteq [0,1]$, which means that a pair of linear kernels, $k(x,x')=xx'$ and $l(y,y')=yy'$ are nonnegative and bounded by one on $\calX$ and $\calY$ respectively. Note that for a linear kernel,
\begin{equation*}
    \hat{g}_t(x,y) = \hat{g}_t \cdot x\cdot y.
\end{equation*}
Hence,
\begin{align*}
    f_t ((X_{2t-1},Y_{2t-1}),(X_{2t},Y_{2t})) 
    &=\frac{\hat{g}_t}{2}\roundbrack{X_{2t}-X_{2t-1}}\roundbrack{Y_{2t}-Y_{2t-1}}\\
    &=\frac{\hat{g}_t}{8}\roundbrack{V_{2t}-V_{2t-1}+\frac{1}{2t(2t-1)}}\roundbrack{V'_{2t}-V'_{2t-1}+\frac{1}{2t(2t-1)}}.
\end{align*}
In particular, $\Exp{}{f_t ((X_{2t-1},Y_{2t-1}),(X_{2t},Y_{2t}))\mid \calF_{t-1}}\neq 0$,
implying that the wealth process $(\calK_t)_{t\geq 0}$ is no longer a nonnegative martingale.
\end{example}

\paragraph{Witness Functions for COCO.} Let $\Phi$ and $\Psi$ be a pair of matrices whose columns represent embeddings of $X_1,\dots,X_{2(t-1)}$ and $Y_1,\dots,Y_{2(t-1)}$, that is, $\varphi(X_i)=k(X_i,\cdot)$ and $\psi(Y_i)=l(Y_i,\cdot)$ for $i=1,\dots,2(t-1)$. Recall that
\begin{equation*}
    \begin{aligned}
        \hat{g} &= \sum_{i=1}^{2(t-1)} \alpha_i \roundbrack{\varphi(X_i)-\frac{1}{2(t-1)}\sum_{j=1}^{2(t-1)} \varphi(X_j)}=\Phi H \alpha,\\
        \hat{h} &= \sum_{i=1}^{2(t-1)} \beta_i \roundbrack{\psi(Y_i)-\frac{1}{2(t-1)}\sum_{j=1}^{2(t-1)} \psi(Y_j)} = \Psi H \beta,
    \end{aligned}
\end{equation*}
where $H=\textbf{I}_{2(t-1)}-\frac{1}{2(t-1)}11^\top$ is the centering projection matrix. We have
\begin{equation*}
\begin{aligned}
    \langle g, \hat{C}_{XY} h\rangle_{\calG} &= \frac{1}{2(t-1)} (\alpha^\top H \Phi^\top) (\Phi H\Psi^\top) (\Psi H\beta) = \frac{1}{2(t-1)} \alpha^\top HKHLH\beta = \frac{1}{2(t-1)} \alpha^\top \tilde{K}\tilde{L}\beta, \\
    \norm{\calG}{g}^2 &= \alpha^\top \tilde{K} \alpha, \\
    \norm{\calH}{h}^2 &= \beta^\top \tilde{L} \beta,
\end{aligned}
\end{equation*}
where $\tilde{K}:=HKH$ and $\tilde{L}:=HLH$ are centered kernel matrices. Hence, the maximization problem in~\eqref{eq:coco_cross_cov} can be expressed as:
\begin{equation}\label{eq:coco_max_prob}
   \begin{aligned}
    \max_{\alpha,\beta} \quad & \frac{1}{2(t-1)}\alpha^\top \tilde{K}\tilde{L}\beta\\
    \text{subject to} \quad & \alpha^\top \tilde{K} \alpha = 1,\quad  \beta^\top \tilde{L} \beta = 1.\\
\end{aligned}    
\end{equation}
After introducing Lagrange multipliers, it can then be shown that $\alpha$ and $\beta$, which solve~\eqref{eq:coco_max_prob}, exactly correspond to the generalized eigenvalue problem~\eqref{eq:gen_eig_coco}.

\paragraph{Witness Functions for KCC.} Introduce empirical covariance operators:
\begin{equation*}
\begin{aligned}
    \hat{C}_{X} &= \frac{1}{2(t-1)}\sum_{i=1}^{2(t-1)} \varphi(X_i)\otimes\varphi(X_i) - \roundbrack{\frac{1}{2(t-1)}\sum_{i=1}^{2(t-1)} \varphi(X_i)}\otimes\roundbrack{\frac{1}{2(t-1)}\sum_{i=1}^n \varphi(X_i)}=\frac{1}{2(t-1)}\Phi H\Phi^\top, \\
    \hat{C}_{Y} &= \frac{1}{n}\sum_{i=1}^{2(t-1)} \psi(Y_i)\otimes\psi(Y_i) - \roundbrack{\frac{1}{2(t-1)}\sum_{i=1}^{2(t-1)} \psi(Y_i)}\otimes\roundbrack{\frac{1}{2(t-1)}\sum_{i=1}^{2(t-1)} \psi(Y_i)} = \frac{1}{2(t-1)}\Psi H\Psi^\top.
\end{aligned}
\end{equation*}
Then the empirical variance terms can be expressed as:
\begin{equation*}
\begin{aligned}
    \EmpVar{}{g(X)} &= \langle g, \hat{C}_{X} g\rangle_{\calG}= \frac{1}{2(t-1)} (\alpha^\top H \Phi^\top) (\Phi H\Phi^\top) (\Phi H\alpha) = \frac{1}{2(t-1)}\alpha^\top \tilde{K}^2 \alpha, \\
    \EmpVar{}{h(Y)} &= \langle h, \hat{C}_{Y} h\rangle_{\calH}= \frac{1}{2(t-1)} (\beta^\top H \Psi^\top) (\Psi H\Psi^\top) (\Psi H\beta) = \frac{1}{2(t-1)}\beta^\top \tilde{L}^2 \beta.
\end{aligned}
\end{equation*}
Thus, an empirical estimator of the kernel canonical correlation~\eqref{eq:kcc_def} can be obtained by solving:
\begin{equation*}
\begin{aligned}
    \max_{\alpha,\beta} \quad & \frac{1}{2(t-1)}\alpha^\top \tilde{K}\tilde{L}\beta\\
    \text{subject to} \quad & \frac{1}{2(t-1)}\alpha^\top \tilde{K}^2 \alpha+ \kappa_1 \alpha^\top \tilde{K} \alpha = 1,\\
    & \frac{1}{2(t-1)}\beta^\top \tilde{L}^2 \beta+ \kappa_2 \beta^\top \tilde{L} \beta = 1.\\
\end{aligned}   
\end{equation*}
After introducing Lagrange multipliers, it can then be shown that $\alpha$ and $\beta$, which solve~\eqref{eq:kcc_def}, correspond to the generalized eigenvalue problem:
\begin{align*}
    \begin{pmatrix}
    0 & \frac{1}{2(t-1)}\tilde{K}\tilde{L} \\ 
    \frac{1}{2(t-1)}\tilde{L}\tilde{K} & 0
    \end{pmatrix}
    \begin{pmatrix}
    \alpha \\ \beta
    \end{pmatrix} &= \gamma \begin{pmatrix}
    \kappa_1\tilde{K}+\frac{1}{2(t-1)}\tilde{K}^2 & 0 \\ 
    0 & \kappa_2\tilde{L}+\frac{1}{2(t-1)}\tilde{L}^2
    \end{pmatrix} \begin{pmatrix}
    \alpha \\ \beta
    \end{pmatrix},
\end{align*}

\clearpage




\section{Additional Simulations}

This section contains: (a) additional experiments on synthetic dataset and (b) data visualizations of the datasets used in this paper.

\subsection{Test of Instantaneous Dependence}\label{appsubsec:inst_dep}

In Figure~\ref{fig:ins_dep_vis}, we demonstrate it is hard to visually tell the difference between independence and dependence under distribution drift setting~\eqref{eq:ind_test_generalization_indep}. See Example~\ref{ex:dist_drift_ex} for details.

\begin{figure}[!htb]
\begin{center}
        \subfigure[Independent noise.]{
            \includegraphics[width=0.45\linewidth]{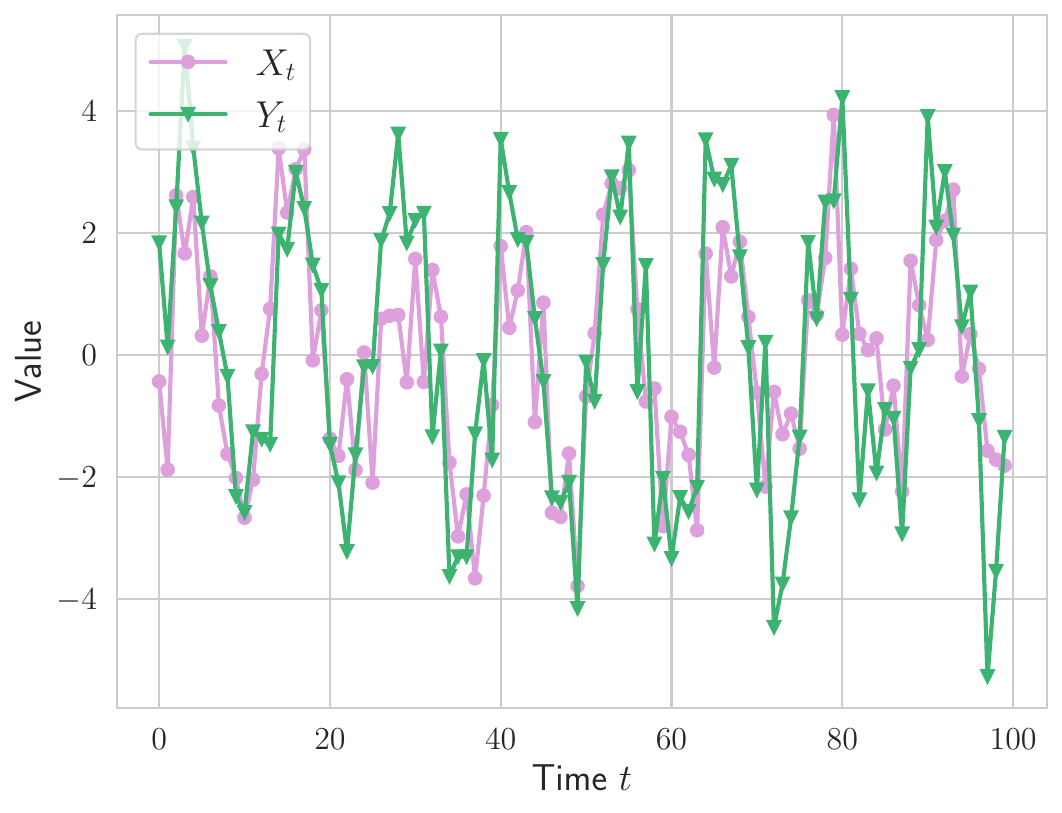}
            \label{subfig:drift_vis_null}}
        \subfigure[Dependent noise.]{
            \includegraphics[width=0.45\linewidth]{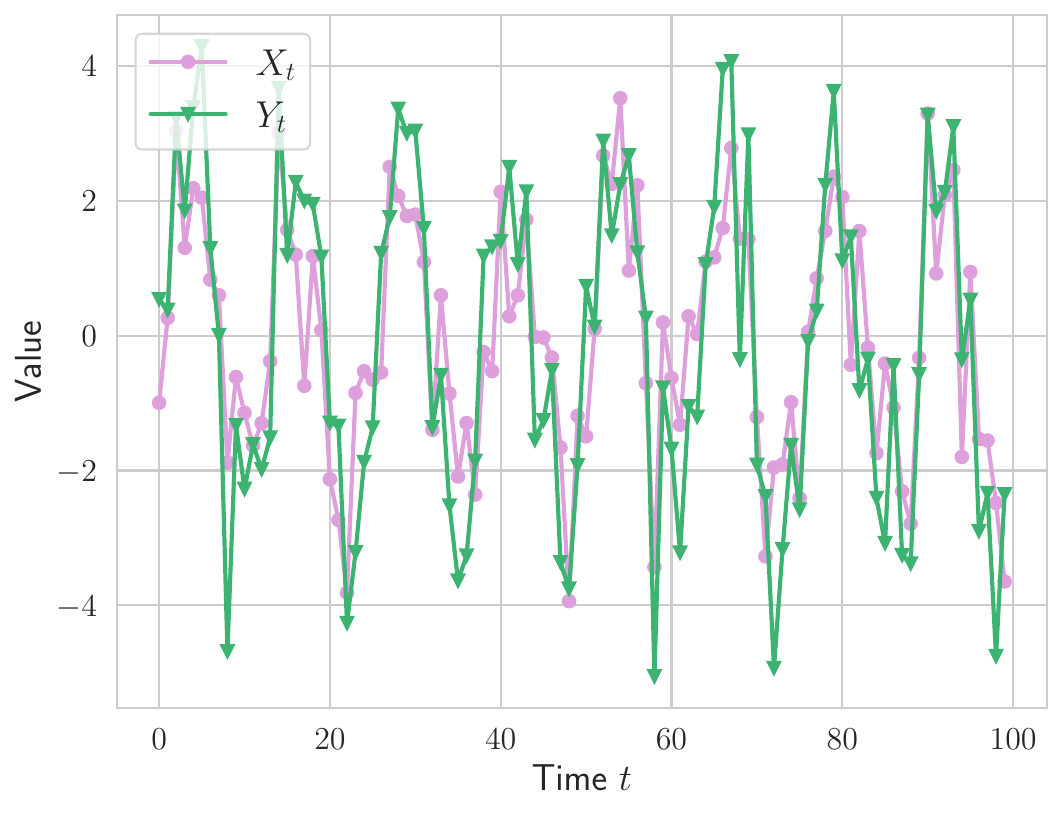}
            \label{subfig:drift_vis_alt}}
        \caption{Sample of independent (subplot (a)) and dependent ($\rho=0.5$, subplot(b)) data according to~\eqref{eq:dist_drift_example}. The purpose of visualizing raw data is to demonstrate that dependence is hard to detect visually, and dependence refers to more than temporal correlation which may be present due to cyclical trends.}
        \label{fig:ins_dep_vis}
\end{center}
\end{figure}

\subsection{Distribution Drift}\label{appsubsec:dist_drift}

In this section, we consider the linear Gaussian model with an underlying distribution drift:
\begin{equation*}
        Y_t =  X_t\beta_t + \varepsilon_t,\quad X_t,\varepsilon_t \sim \mathcal{N}(0,1), \quad t\geq 1,
\end{equation*}
that is, in contrast to the Gaussian linear model (Section~\ref{sec:alt_measures}), $\beta_t$ changes over time. We gradually increase it from $\beta_t=0$ to $\beta_t=0.1$ in increments of 0.02, that is:
\begin{equation*}
\underbrace{\beta_0, \ \dots, \ \beta_{b-1}}_{=0}, \ \underbrace{\beta_b,\ \dots, \ \beta_{2b-1}}_{=0.02}, \ \dots, \ \underbrace{ \beta_{5b-1},\ \dots}_{=0.1}
\end{equation*}
and, starting with $\beta_{5b}$, we keep it equal to 0.1. We consider $b\in \curlybrack{100, 200, 400}$ as possible block sizes. Note that there is a transition from independence (first $b$ datapoints in a stream) to dependence. In Figure~\ref{subfig:gaus_drif}, we show that our test performs well under the distribution drift setting and consistently detects dependence. 
\begin{figure}[!htb]
\centering
        \includegraphics[width=0.425\linewidth]{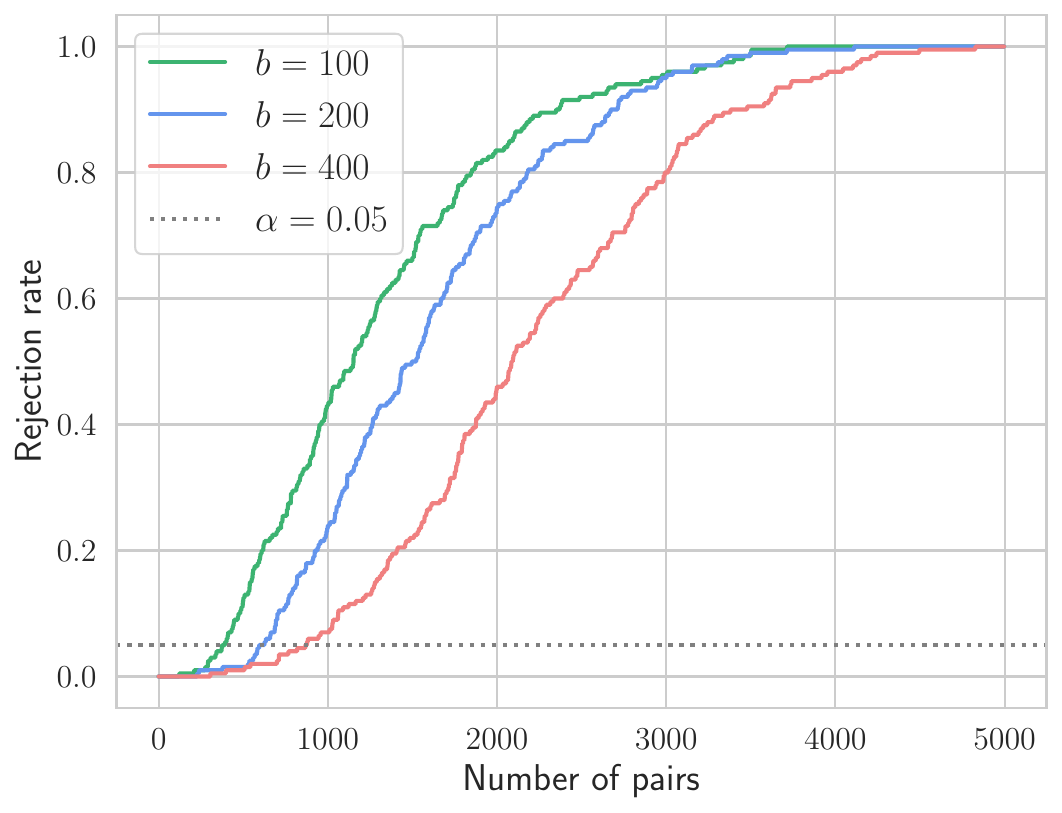}
        \caption{Rejection rate of sequential independence test under distribution drift setting. Focusing on the non-i.i.d.\ time-varying setting, we confirm that our test has high power under the alternative.}
        \label{subfig:gaus_drif}
\end{figure}

\subsection{Symmetry-based Payoff Functions}\label{appsubsec:sym_bet}

In this section, we complement the comparison presented in Section~\ref{sec:alt_bet} between the rank- and composition-based betting strategies (since those require minimal tuning) used with ONS or aGRAPA criteria for selecting betting fractions. We also increase the monitoring horizon to 20000 datapoints. In Figure~\ref{subfig:sym_gaus}, we consider the Gaussian linear model, but in contrast to the setting considered in Section~\ref{sec:alt_bet}, we focus on harder testing settings by considering $\beta\in [0,0.3]$. In Figure~\ref{subfig:sym_sphere}, we compare composition- and rank-based approaches when data are sampled from the spherical model. In both cases, composition and rank-based approaches are similar; none of the payoffs uniformly dominates the other. We also observe that selecting betting fractions via aGRAPA criterion tends to result in a bit more powerful testing procedure. 

\begin{figure}[!htb]
\begin{center}
        \subfigure[]{              
                \includegraphics[width=0.425\linewidth]{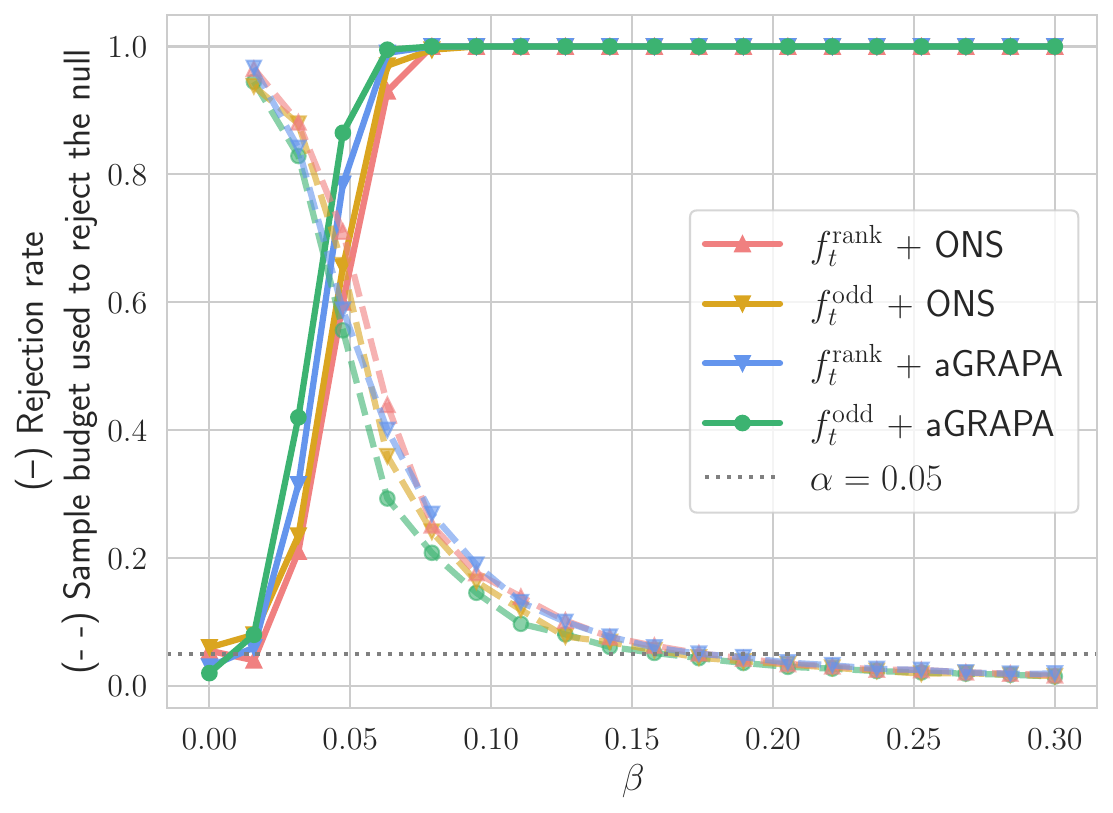}
                \label{subfig:sym_gaus}}
        \subfigure[]{
                \includegraphics[width=0.425\linewidth]{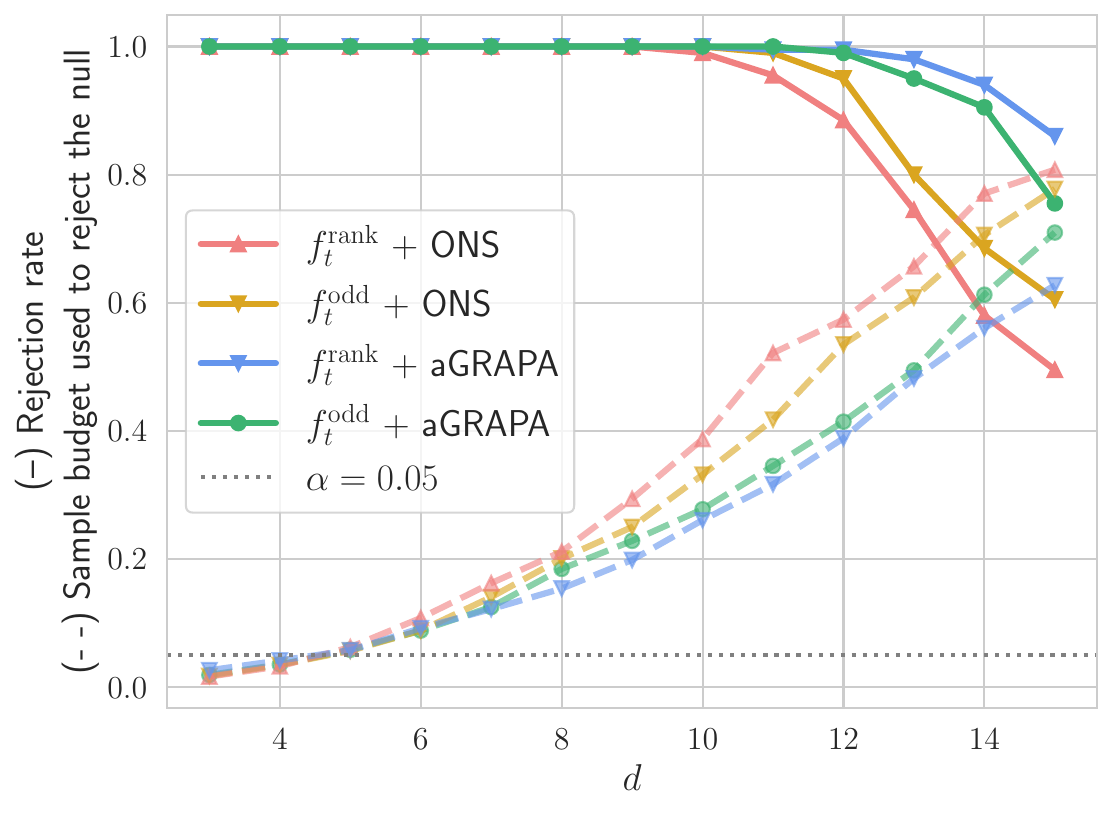}
                \label{subfig:sym_sphere}}
        \caption{(a) Comparison of symmetry-based betting strategies under the Gaussian model. The betting strategy based on composition with an odd function performs only slightly better than the rank-based strategy. (b) SKIT with composition- and rank-based betting strategies under the spherical model. None of the betting strategies uniformly dominates the other. aGRAPA criterion for selecting betting fractions tends to result in a bit more powerful testing procedure.}
        \label{fig:ons_grapa_power_app}
\end{center}
\end{figure}

\subsection{Hard-to-detect Dependence}\label{appsubsec:hard_to_det_vis}

\emph{Hard-to-detect dependence.} Consider the joint density $p(x,y)$ of the form:
\begin{equation}\label{eq:hard_to_detect}
    \frac{1}{4\pi^2} \roundbrack{1+\sin(wx)\sin(wy)}\cdot \indicator{(x,y)\in[-\pi,\pi]^2}.
\end{equation}
With the null case corresponding to $w=0$, the testing problem becomes harder with growing $w$. In Figure~\ref{fig:hard_to_det_vis}, we illustrate the densities and a data sample for the hard-to-detect setting~\eqref{eq:hard_to_detect}.

\begin{figure}[!htb]
\begin{center}
        \subfigure[$w=0$.]{
                \includegraphics[width=0.3\linewidth]{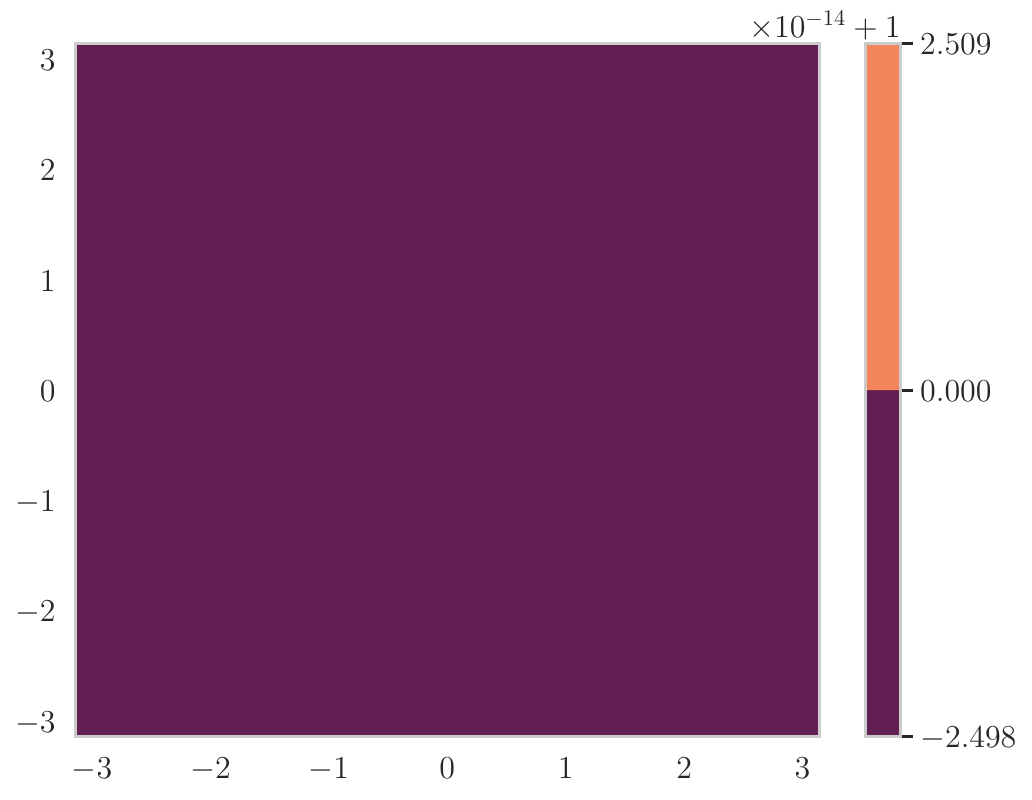}}
        \subfigure[$w=2$.]{
        \includegraphics[width=0.3\linewidth]{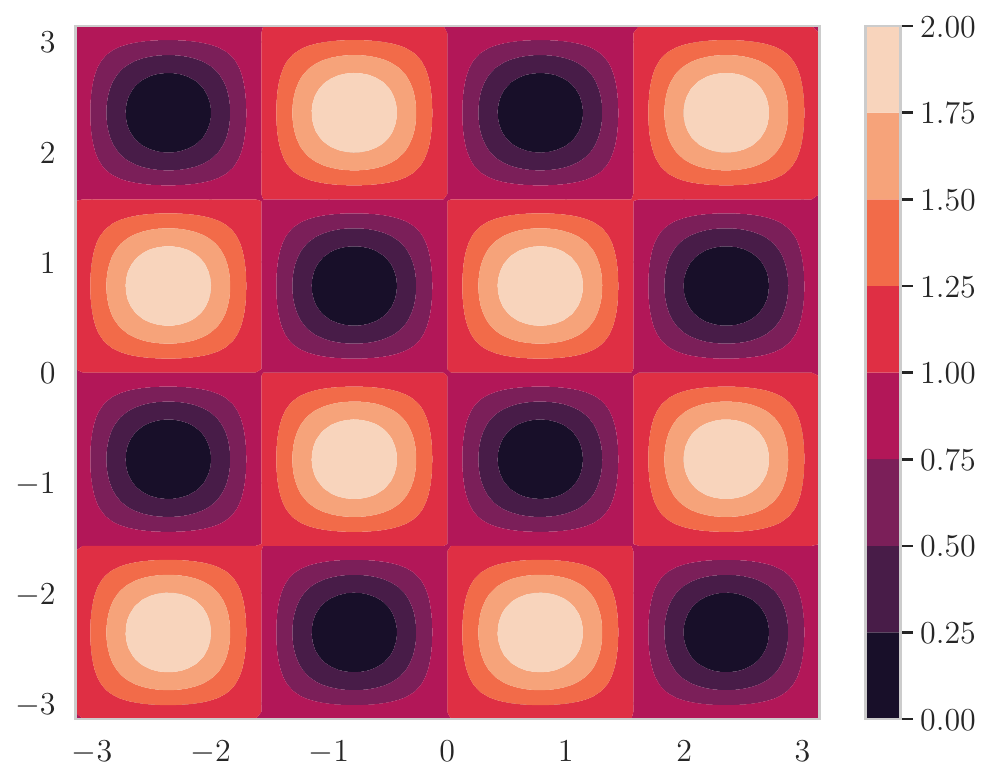}}
        \subfigure[$w=5$]{
                \includegraphics[width=0.3\linewidth]{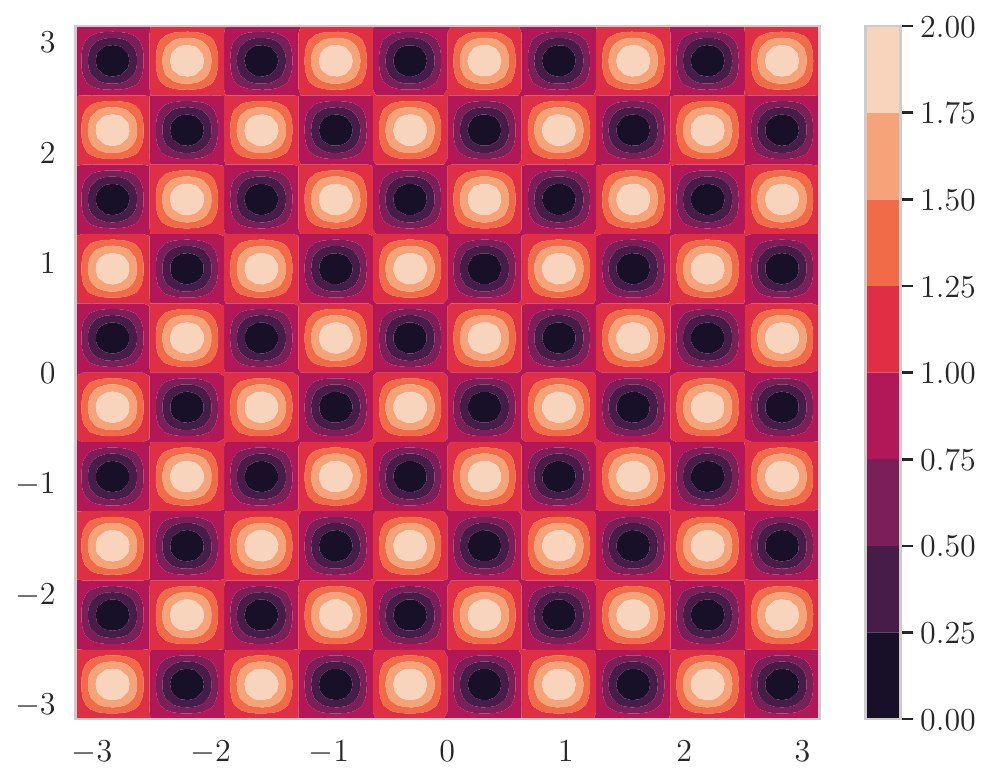}}
        \subfigure[$w=0$.]{
                \includegraphics[width=0.3\linewidth]{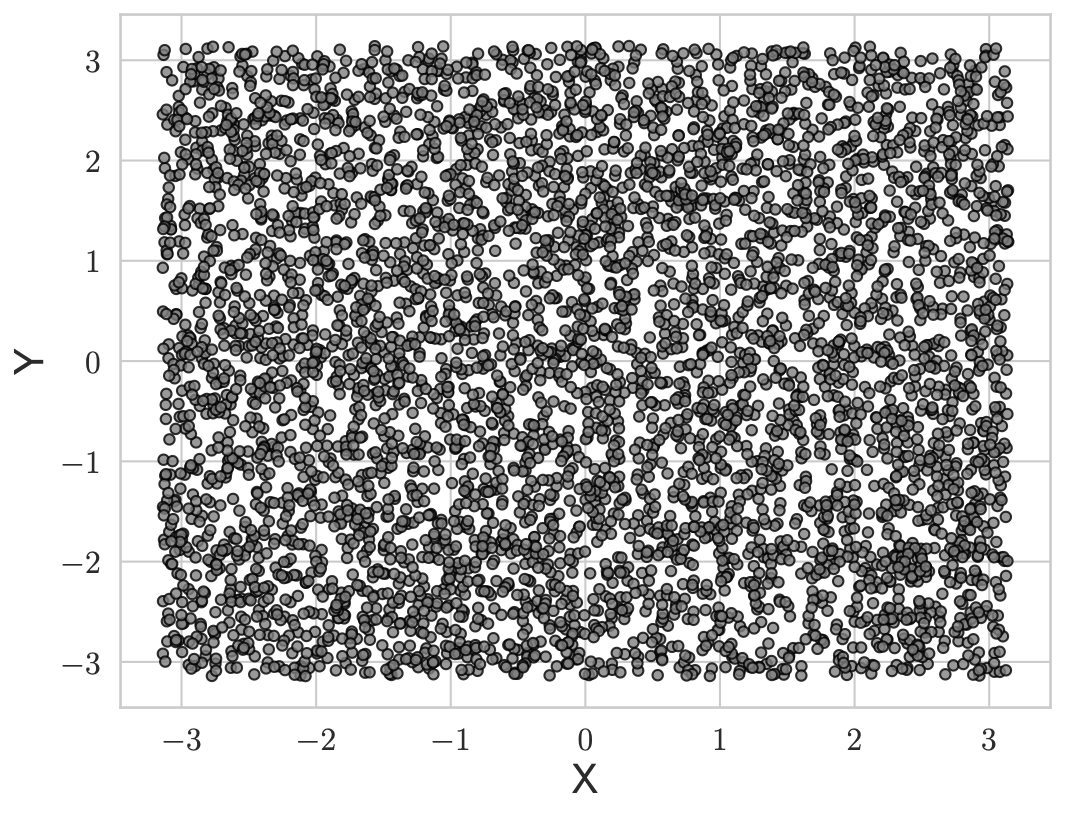}}
        \subfigure[$w=2$.]{
            \includegraphics[width=0.3\linewidth]{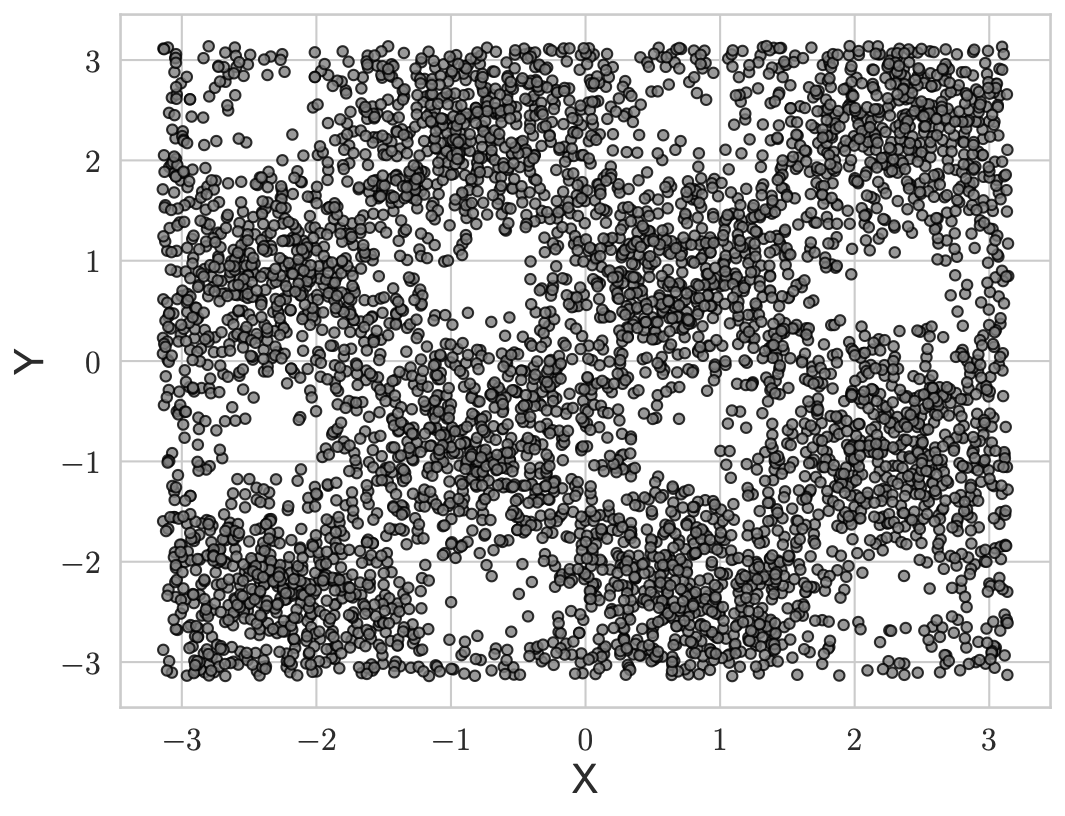}}
        \subfigure[$w=5$]{
                \includegraphics[width=0.3\linewidth]{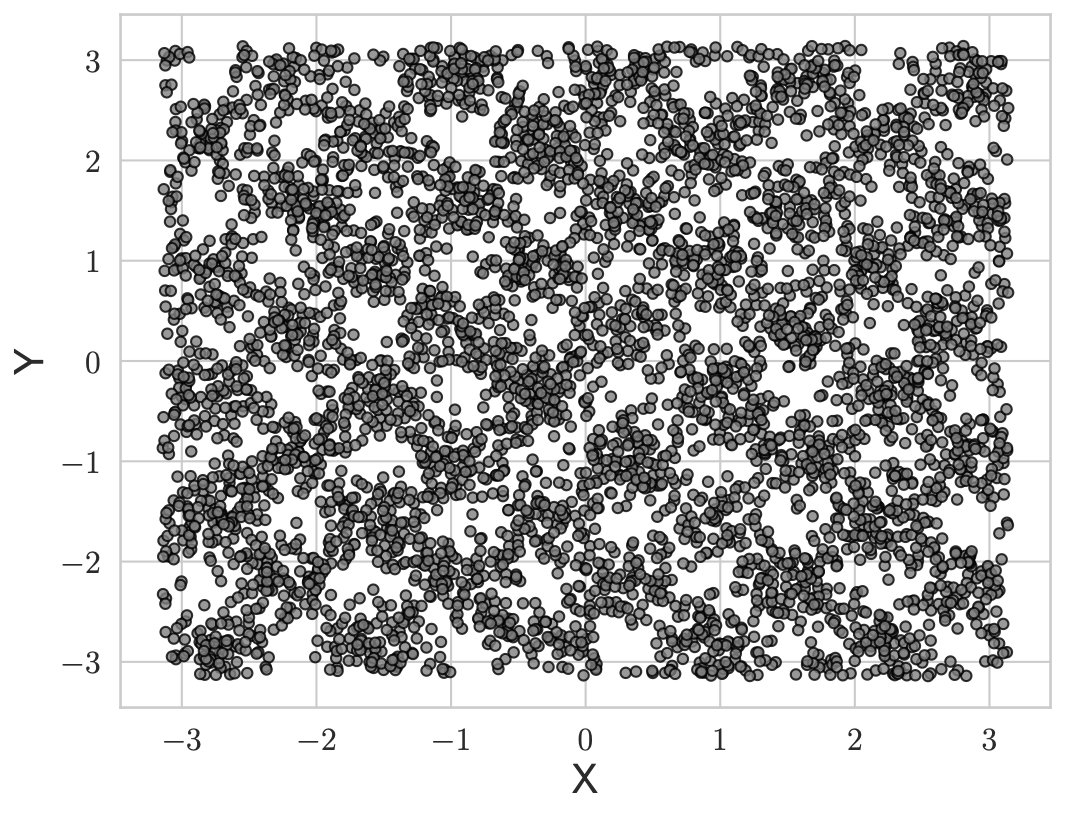}}
        \caption{Visualization of the densities (top) and a dataset of size 5000 (bottom) sampled from the corresponding distribution.}
        \label{fig:hard_to_det_vis}
\end{center}
\end{figure}

We use $\lambda_X = \lambda_Y = 3/(4\pi^2)$ as RBF kernel hyperparameters. For visualization purposes, we stop monitoring after observing 20000 datapoints from $P_{XY}$, and if a SKIT does not reject $H_0$ by that time, we assume that the null is retained. The results are aggregated over 200 runs for each value of $w$. In Figure~\ref{subfig:hard_to_sim_comp}, where the null case corresponds to $w=0$, we confirm that SKITs have time-uniform type I error control. The average rejection rate starts to drop for $w\geq 3$, meaning that observing 20000 points from $P_{XY}$ does not suffice to detect dependence. 

\begin{figure}[!htb]
\begin{center}
 \includegraphics[width=0.45\linewidth]{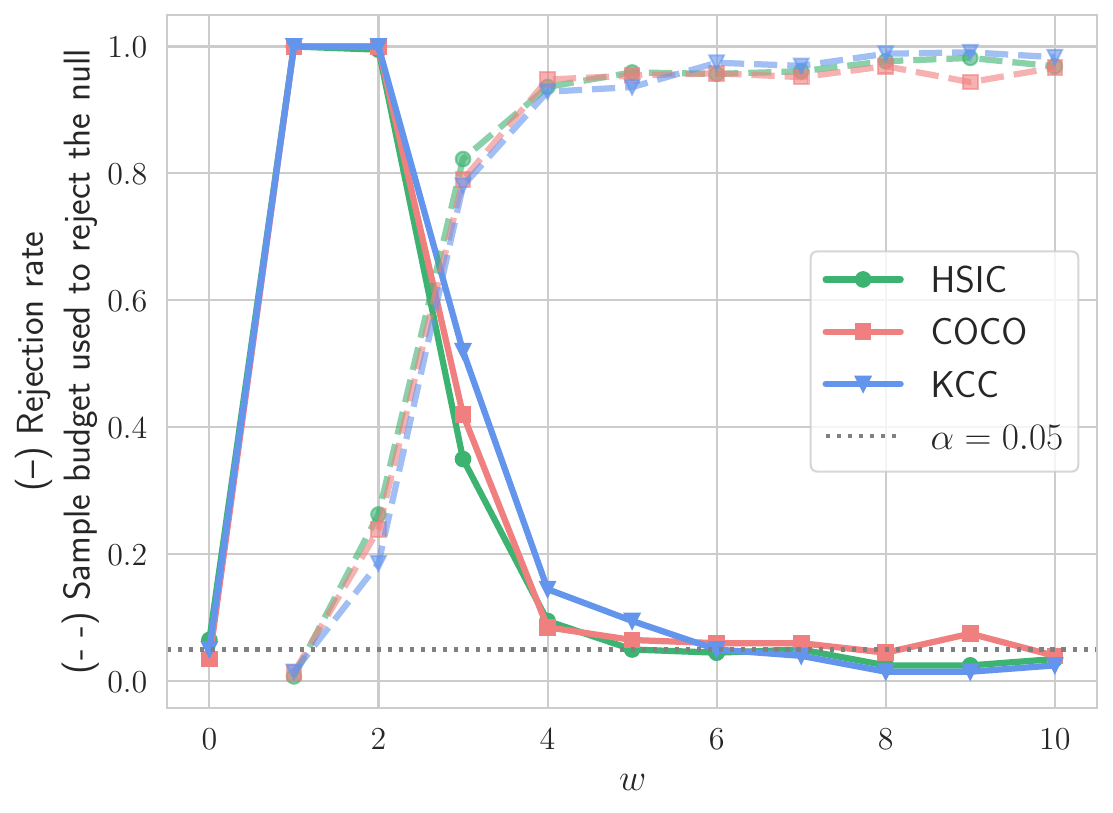}
        \caption{Rejection rate (solid) and fraction of samples used before the null hypothesis was rejected (dashed) for hard-to-detect dependence model. By inspecting the rejection rate for $w=0$ (independence holds), we confirm that the type I error is controlled. Further, SKIT is adaptive to the complexity of a problem (larger $w$ corresponds to a harder setting).}
        \label{subfig:hard_to_sim_comp}
\end{center}
\end{figure}

\subsection{Additional Results for Real Data}\label{appsubsec:real_data}

In Figure~\ref{fig:real_data_seasonal}, we illustrate that the average daily temperature in selected cities share similar seasonal patterns. We repeat the same experiment as in Section~\ref{sec:alt_bet}, but for four cities in South Africa: Cape Town (CT), Port Elizabeth (PE), Durban (DRN), and Bloemfontein (BFN). In Figures~\ref{subfig:europe_martingales} and~\ref{subfig:south_africa_martingales}, we illustrate the resulting wealth processes for each pair of cities and for each region. Finally, we illustrate the pairs of cities for which the null has been rejected in Figure~\ref{subfig:south_africa_map}.

\begin{figure}[!htb]
\begin{center}
 \subfigure[]{              
            \includegraphics[width=0.425\linewidth]{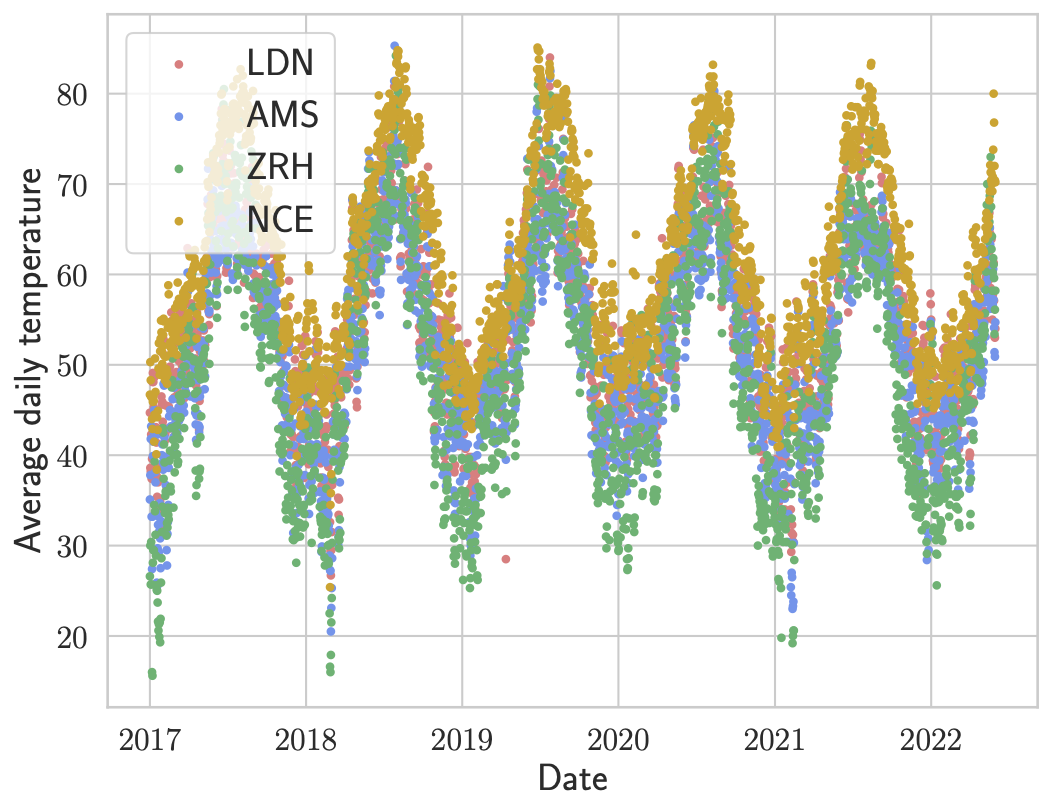}
            \label{subfig:europe_seasonal}}
\subfigure[]{         
            \includegraphics[width=0.425\linewidth]{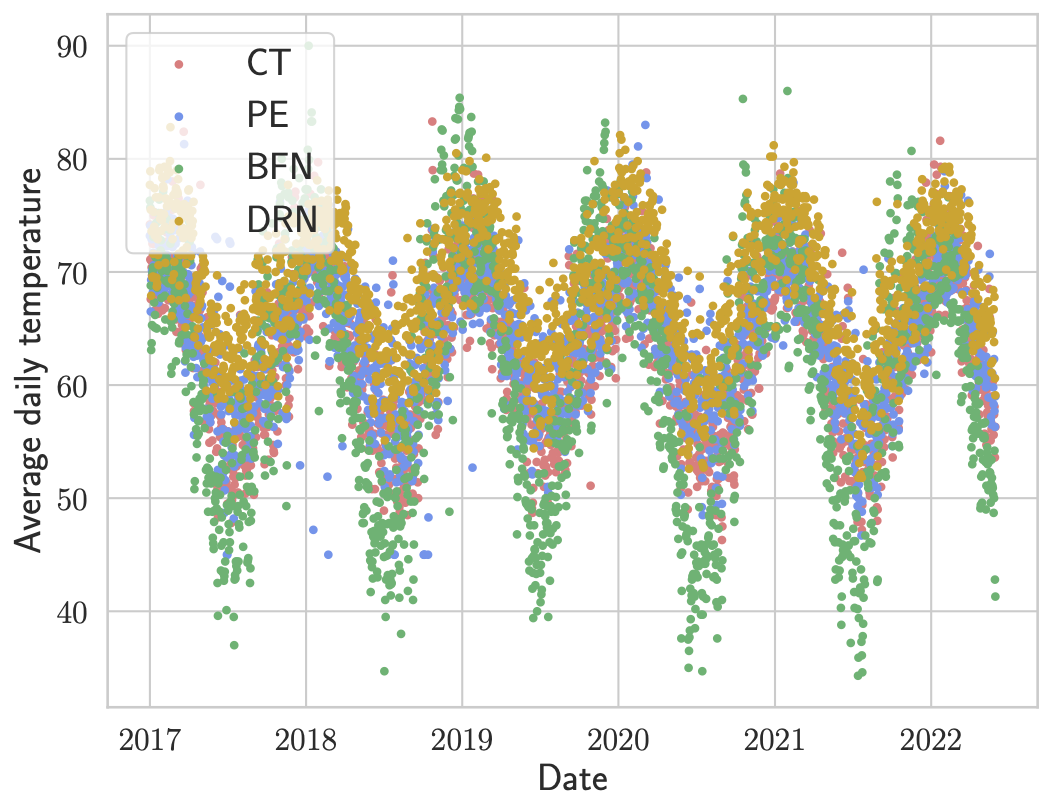}
            \label{subfig:south_africa_seasonal}}
                \subfigure[]{          
            \includegraphics[width=0.425\linewidth]{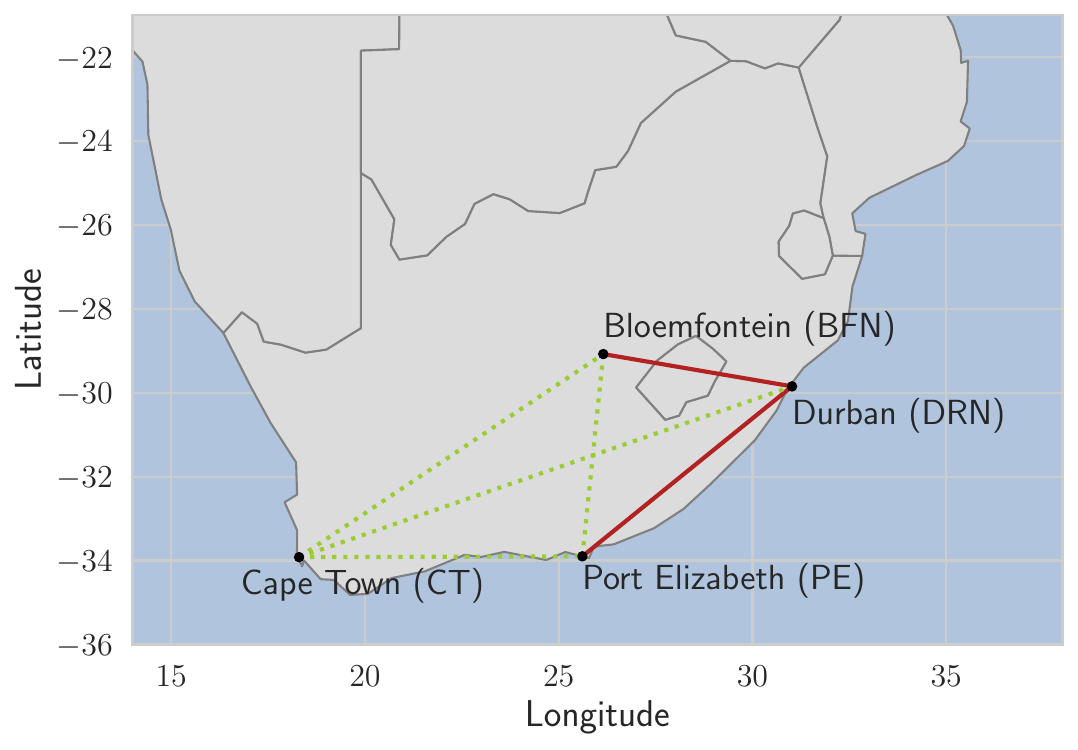}
            \label{subfig:south_africa_map}}
            \subfigure[]{               
            \includegraphics[width=0.425\linewidth]{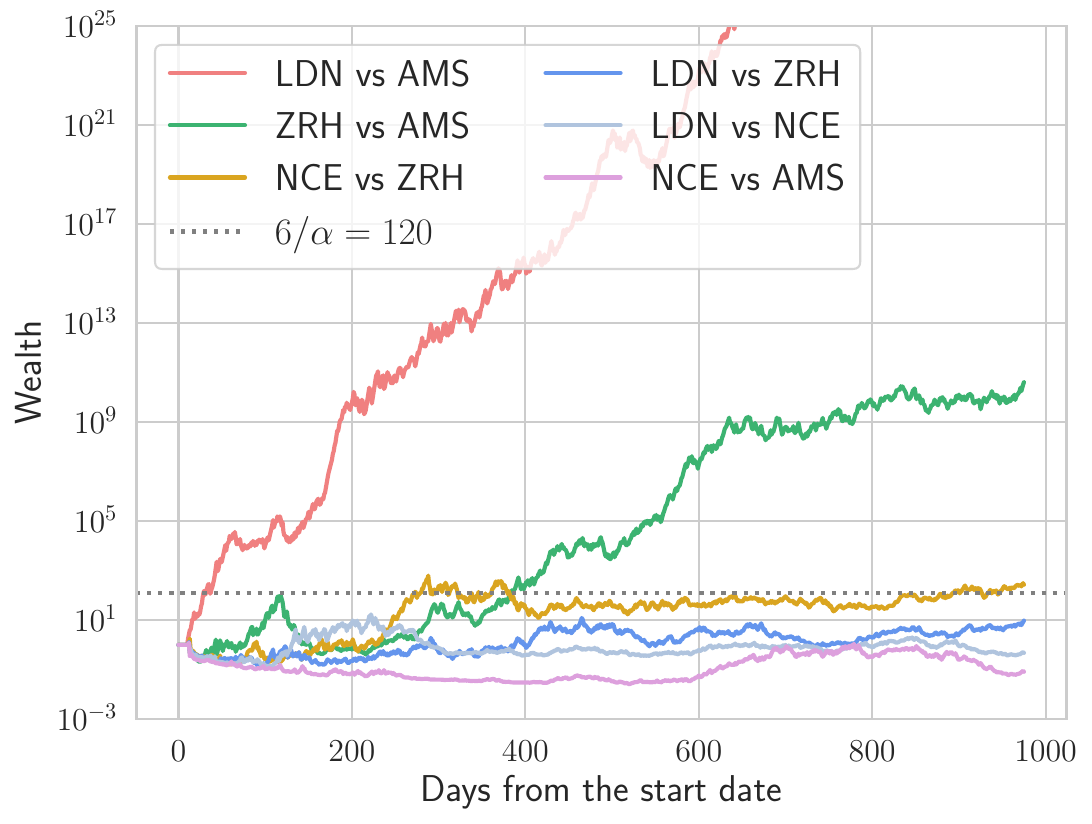}
            \label{subfig:europe_martingales}}
    \subfigure[]{             
            \includegraphics[width=0.425\linewidth]{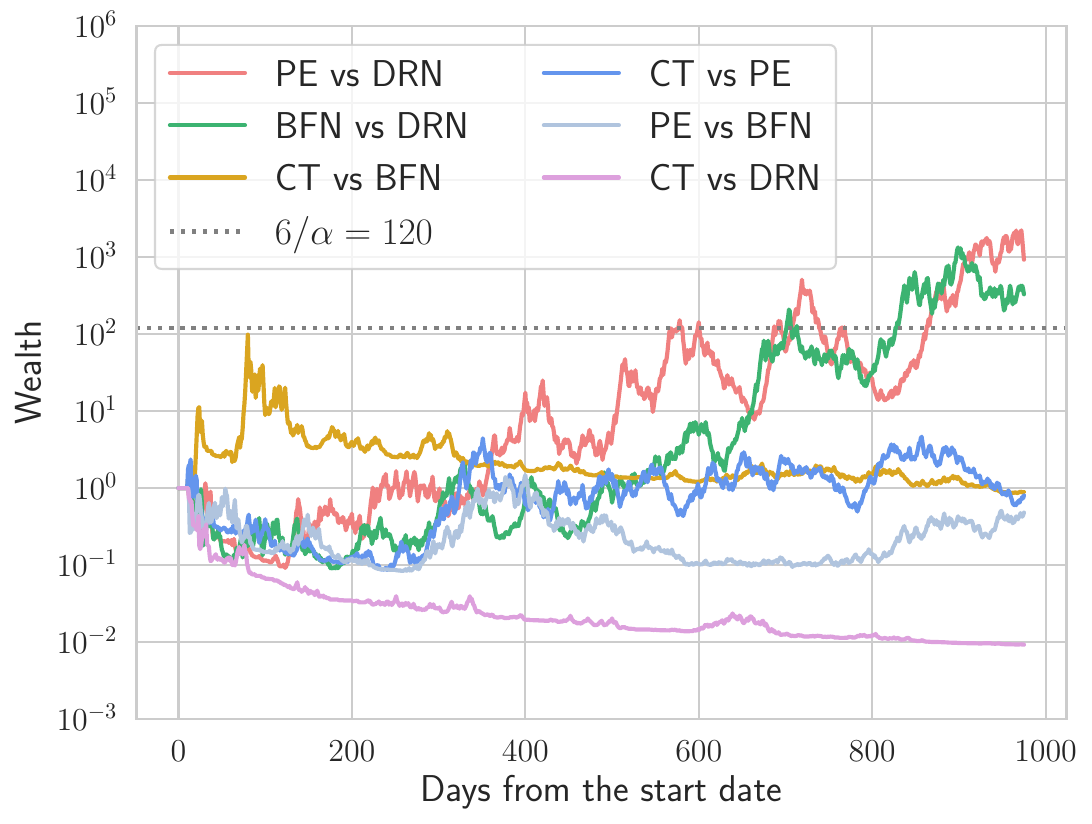}
            \label{subfig:south_africa_martingales}}
    \caption{Temperatures for selected cities in Europe (subplot (a)) and South Africa (subplot (b)) share similar seasonal patterns. Map (subplot (c)) where solid red lines connect those cities for which the null is rejected. SKIT supports our conjecture about dependent temperature fluctuations for geographically close cities. For completeness, we also plot wealth processes for SKIT used on weather data for Europe (subplot (d)) and South Africa (subplot (e)).  }
    \label{fig:real_data_seasonal}
\end{center}
\end{figure}


\subsection{Experiment with MNIST data}\label{appubsec:mnist_exp}

In this section, we analyze the performance of SKIT on high-dimensional real data. This experiment is based on MNIST dataset~\citep{lecun1998mnist} where pairs of digits are observed at each step; under the null one sees digits $(a,b)$ where $a$ and $b$ are uniformly randomly chosen, but under the alternative one sees $(a, a')$, i.e., two different images of the same digit. To estimate kernel hyperparameters, we deploy the median heuristic using 20 pairs of images. 

We illustrate the results in Figure~\ref{subfig:mnist_vis}. Under the null, our test does not reject more often than the required 5\%, but its power increases with sample size under the alternative, reaching power one after processing $\approx 500$ pairs of digits (points from $P_{XY}$) on average.

\begin{figure}[!htb]
        \centering
        \includegraphics[width=0.425\textwidth]{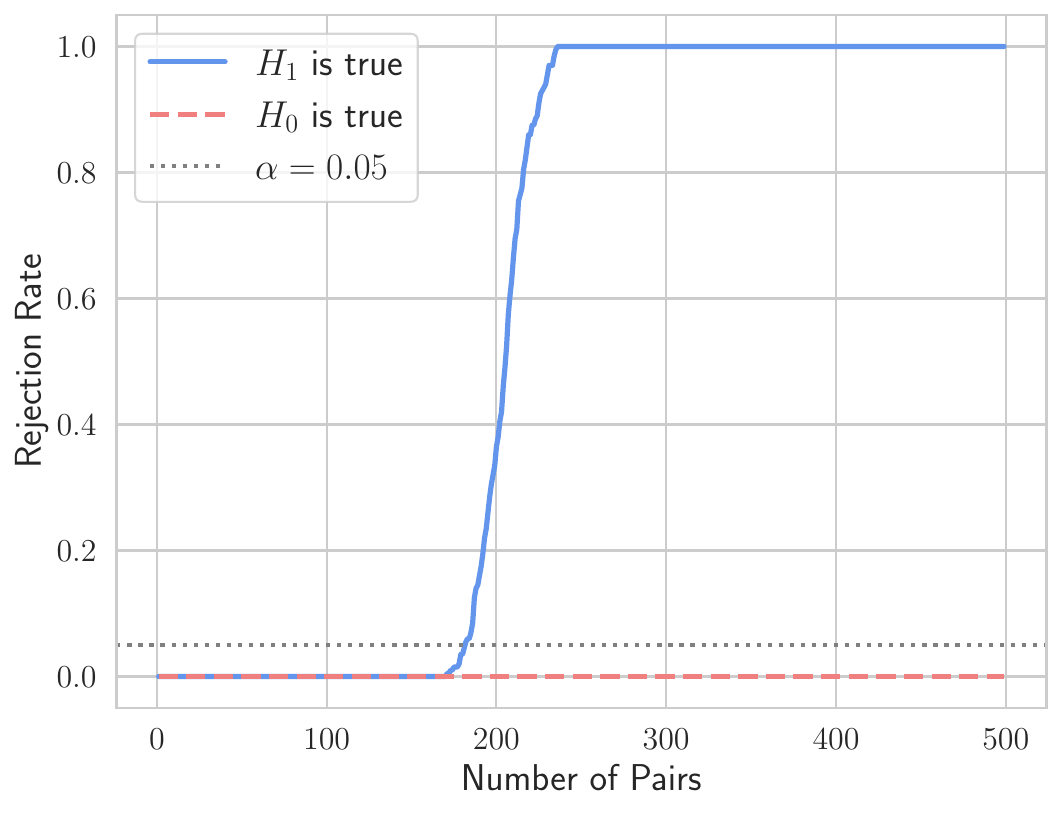}
        \caption{Rejection rate for SKIT on MNIST data. Under the null (red dashed line), our test does not reject more often than the required 5\%, but its power increases with sample size under the alternative (blue solid line). Each pair corresponds to two points from $P_{XY}$, and hence, SKIT reaches power one after processing $\approx 500$ pairs of images on average.}
        \label{subfig:mnist_vis}
    \end{figure}

\clearpage

\section{Scaling Sequential Testing Procedures}

Updating the wealth process at each round requires evaluating the payoff function at a new pair of observations (and hence computing the witness function corresponding to a chosen dependence criterion). In this section, we provide details about the ways of reducing the computational complexity of this step, which are necessary to scale the proposed sequential testing frameworks to moderately large sample sizes. Note that the proposed implementation of COCO allows updating kernel hyperparameters on the fly. In contrast, linear-time updates for HSIC require fixing kernel hyperparameters in advance.

\subsection{Incomplete/Pivoted Cholesky Decomposition for COCO and KCC}\label{appsubsec:coco_updates}

Suppose that we want to evaluate COCO payoff function on the next pair of points $(X_{2t-1},Y_{2t-1}),(X_{2t},Y_{2t})$. In order to do so, we need to compute $g_{1,t}$ and $g_{2,t}$, that is solve the generalized eigenvalue problem. Note that solving generalized eigenvalue problem at each iteration could be computationally prohibitive. One simple way is to use a random subsample of datapoints when performing witness function estimation, e.g., once the sample size $n$ exceeds $n_s$, e.g., $n_s=25$, we randomly subsample (without replacement) a sample of size $n_s$ to estimate witness functions. Alternatively, a common approach is to reduce computational burden through incomplete Cholesky decomposition. The idea is to use the fact that kernel matrices tend to demonstrate rapid spectrum decay, and thus low-rank approximations can be used to scale the procedures. Suppose that $K \approx G_1G_1^T$ and $L \approx G_2G_2^T$ where $G_i$'s are lower triangular matrices of size $n\times M$ ($M$ depends on the preset approximation error level). After computing Cholesky decomposition, we center both matrices via left multiplication by $H$ and compute SVDs of $HG_1$ and $HG_2$, that is, $HG_1=U_1\Lambda_1V_1^\top$ and $HG_2=U_2\Lambda_2V_2^\top$. We have:
\begin{equation*}
    \begin{aligned}
        \tilde{K} \approx U_1\Lambda_1^2U_1^\top, \quad 
        \tilde{L} \approx U_2\Lambda_2^2U_2^\top.
    \end{aligned}
\end{equation*}
Our goal is to find the largest eigenvalue/eigenvector pair for $Ax=\gamma Bx$ for a PD matrix $B$. Since:
\begin{equation*}
    Ax=\gamma Bx \Longleftrightarrow B^{-1/2}AB^{-1/2} (B^{1/2} x)=\gamma (B^{1/2}x),
\end{equation*}
it suffices to leading eigenvalue/eigenvector pair for:
\begin{equation*}
    B^{-1/2}AB^{-1/2}y=\gamma y.
\end{equation*}
Then $x=B^{-1/2}y$ is a generalized eigenvector for the initial problem. 

\paragraph{COCO.} For COCO, we have:
\begin{equation*}
\begin{aligned}
    B&= \begin{pmatrix}
    \tilde{K} & 0 \\ 
    0 & \tilde{L}
    \end{pmatrix}  \approx \begin{pmatrix}
    U_1\Lambda_1^2U_1^\top & 0 \\ 
    0 & U_2\Lambda_2^2U_2^\top
    \end{pmatrix} = \begin{pmatrix}
    U_1 & 0 \\ 
    0 & U_2 \end{pmatrix}\begin{pmatrix}
    \Lambda_1^2 & 0 \\ 
    0 & \Lambda_2^2
    \end{pmatrix}\begin{pmatrix}
    U_1 & 0 \\ 
    0 & U_2
    \end{pmatrix}^\top \\
    \Longrightarrow B^{-1/2} & \approx \begin{pmatrix}
    U_1 & 0 \\ 
    0 & U_2 \end{pmatrix}\begin{pmatrix}
    \Lambda_1^{-1} & 0 \\ 
    0 & \Lambda_2^{-1}
    \end{pmatrix}\begin{pmatrix}
    U_1 & 0 \\ 
    0 & U_2
    \end{pmatrix}^\top=: \mathcal{U}\Lambda^{-1}\mathcal{U}^\top.
\end{aligned}
\end{equation*}
We also have:
\begin{equation*}
\begin{aligned}
    A &\approx \begin{pmatrix}
    0 & \frac{1}{n}U_1\Lambda_1^2U_1^\top U_2\Lambda_2^2U_2^\top \\ 
    \frac{1}{n} U_2\Lambda_2^2U_2^\top U_1\Lambda_1^2U_1^\top & 0
    \end{pmatrix} \\
    &= \begin{pmatrix}
    U_1 & 0 \\ 
    0& U_2
    \end{pmatrix}\begin{pmatrix}
    0 & \frac{1}{n}\Lambda_1^2U_1^\top U_2\Lambda_2^2 \\ 
    \frac{1}{n} \Lambda_2^2U_2^\top U_1\Lambda_1^2 & 0
    \end{pmatrix}\begin{pmatrix}
    U_1 & 0 \\ 
    0& U_2
    \end{pmatrix}^\top.
\end{aligned}
\end{equation*}
Thus we have:
\begin{equation*}
\begin{aligned}
    B^{-1/2}AB^{-1/2} &\approx  \begin{pmatrix}
    U_1 & 0  \\ 
     0&  U_2
    \end{pmatrix}\begin{pmatrix}
    0 & \frac{1}{n}\Lambda_1U_1^\top U_2\Lambda_2 \\ 
    \frac{1}{n} \Lambda_2U_2^\top U_1\Lambda_1 & 0
    \end{pmatrix}\begin{pmatrix}
    U_1 & 0  \\ 
     0&  U_2
    \end{pmatrix}^\top.
\end{aligned}
\end{equation*}
Hence, we only need to compute the leading eigenvector (say, $z^*$) for:
\begin{equation*}
    \begin{pmatrix}
    0 & \frac{1}{n}\Lambda_1U_1^\top U_2\Lambda_2 \\ 
    \frac{1}{n} \Lambda_2U_2^\top U_1\Lambda_1 & 0
    \end{pmatrix} \in \Real^{(M_1+M_2)\times (M_1+M_2)}.
\end{equation*}
It implies that the leading eigenvector for $B^{-1/2}AB^{-1/2}$ is then $\mathcal{U}z^*$, and the solution for the generalized eigenvalue problem is given by:
\begin{equation*}
    \mathcal{U} \Lambda^{-1} z^* = \begin{pmatrix}
    U_1 \Lambda_1^{-1} z^*_1\\
    U_2 \Lambda_2^{-1} z^*_2
    \end{pmatrix}=:\begin{pmatrix}
        \alpha_0\\
        \beta_0
    \end{pmatrix}.
\end{equation*}
Next, we need to normalize this vector of coefficients appropriately, i.e., we need to guarantee that $\norm{2}{\tilde{K}^{1/2} \alpha} = 1$ and $\norm{2}{\tilde{L}^{1/2} \beta} = 1$, and thus re-normalizing naively is quadratic in $n$. Instead, note that in order to compute incomplete Cholesky decomposition, we choose a tolerance parameter $\delta$ so that: $\norm{*}{PKP^\top-G_1G_1^\top}=\norm{*}{K-G_1G_1^\top}\leq \delta$ (nuclear norm). Let $\Delta=K-G_1G_1^\top$. We know that:
\begin{equation*}
     \alpha^\top \tilde{K} \alpha =  \alpha^\top HK H \alpha = \alpha^\top H (\Delta+G_1G_1^\top) H \alpha = \alpha^\top H \Delta H \alpha+\alpha^\top H G_1G_1^\top H \alpha
\end{equation*}
First, note that $\alpha^\top H \Delta H \alpha\leq \delta \|H\alpha\|_2^2$. Next,
\begin{equation*}
    G_1^\top H = V_1 \Lambda_1 U_1^\top.
\end{equation*}
Given an initial vector of parameters $\alpha_0$ and $\beta_0$, vectors of coefficients can be normalized in linear time using
\begin{equation*}
\begin{aligned}
    \alpha &= \frac{\alpha_0}{\sqrt{\norm{2}{G_1^\top H\alpha_0}^2+\delta\norm{2}{H\alpha_0}^2}}=\frac{U_1\Lambda_1^{-1}z_1^*}{\sqrt{\norm{2}{V_1z_1^*}^2+\delta\norm{2}{H\alpha_0}^2}}=\frac{U_1\Lambda_1^{-1}z_1^*}{\sqrt{\norm{2}{z_1^*}^2+\delta\norm{2}{H\alpha_0}^2}}, \\ 
    \beta &= \frac{\beta_0}{\sqrt{\norm{2}{G_2^\top H\beta_0}^2+\delta\norm{2}{H\beta_0}^2}}=\frac{U_2\Lambda_2^{-1}z_2^*}{\sqrt{\norm{2}{V_2 z_2^*}^2+\delta\norm{2}{H\beta_0}^2}}=\frac{U_2\Lambda_2^{-1}z_2^*}{\sqrt{\norm{2}{ z_2^*}^2+\delta\norm{2}{H\beta_0}^2}}.
\end{aligned}
\end{equation*}
For small $\delta$, we essentially normalize by $\alpha_0^\top \tilde{K} \alpha_0$ and $\beta_0^\top \tilde{L} \beta_0$ as expected. It also makes sense to use $\delta=n\cdot \delta_0$. Still, re-estimating the witness functions after processing $2t$, $t\geq 1$ points is computationally intensive. In contrast to HSIC, for which there are no clear benefits of skipping certain estimation steps, for COCO we estimate the witness functions after processing $2t^2$, $t\geq 1$ points.

\paragraph{KCC.} For KCC, we have:
\begin{equation*}
\begin{aligned}
    B &= \begin{pmatrix}
    \kappa_1\tilde{K}+\frac{1}{n}\tilde{K}^2 & 0 \\ 
    0 & \kappa_2\tilde{L}+\frac{1}{n}\tilde{L}^2
    \end{pmatrix} \\
    &\approx \begin{pmatrix}
    \kappa_1U_1\Lambda_1^2U_1^\top+\frac{1}{n}U_1\Lambda_1^4U_1^\top & 0 \\ 
    0 & \kappa_2U_2\Lambda_2^2U_2^\top+\frac{1}{n}U_2\Lambda_2^4U_2^\top
    \end{pmatrix}\\
    &= \begin{pmatrix}
    U_1\Lambda_1^2\roundbrack{\kappa_1\textbf{I}_n+\Lambda_1^2\frac{1}{n}}U_1^\top  & 0 \\ 
    0 & U_2\Lambda_2^2\roundbrack{\kappa_2\textbf{I}_n+\Lambda_2^2\frac{1}{n}}U_2^\top 
    \end{pmatrix}\\
    &= \calU \begin{pmatrix}
    \Lambda_1^2\roundbrack{\kappa_1\textbf{I}_n+\Lambda_1^2\frac{1}{n}}  & 0 \\ 
    0 & \Lambda_2^2\roundbrack{\kappa_2\textbf{I}_n+\Lambda_2^2\frac{1}{n}}
    \end{pmatrix} \calU^\top,
\end{aligned}
\end{equation*}
which implies that:
\begin{equation*}
    B^{-1/2} \approx \calU \begin{pmatrix}
    \roundbrack{\kappa_1\textbf{I}_n+\Lambda_1^2\frac{1}{n}}^{-1/2}\Lambda_1^{-1}  & 0 \\ 
    0 & \roundbrack{\kappa_2\textbf{I}_n+\Lambda_2^2\frac{1}{n}}^{-1/2}\Lambda_2^{-1}
    \end{pmatrix} \calU^\top.
\end{equation*}
Recall that:
\begin{equation*}
\begin{aligned}
    A &\approx \calU\begin{pmatrix}
    0 & \frac{1}{n}\Lambda_1^2U_1^\top U_2\Lambda_2^2 \\ 
    \frac{1}{n} \Lambda_2^2U_2^\top U_1\Lambda_1^2 & 0
    \end{pmatrix}\calU^\top.
\end{aligned}
\end{equation*}
Thus,
\begin{equation*}
\begin{aligned}
    B^{-1/2}AB^{-1/2} &\approx  \calU \begin{pmatrix}
    0 & M_*\\ 
    M_*^\top & 0
    \end{pmatrix} \calU^\top, \\
    \text{where} \quad M_* &= \frac{1}{n}\roundbrack{\kappa_1\textbf{I}_n+\Lambda_1^2\frac{1}{n}}^{-1/2}\Lambda_1U_1^\top U_2\Lambda_2\roundbrack{\kappa_2\textbf{I}_n+\Lambda_2^2\frac{1}{n}}^{-1/2}.
\end{aligned}
\end{equation*}
Equivalently,
\begin{equation*}
    M_* = \frac{1}{n}\rho_{\kappa_1}(\Lambda_1) \Lambda_1U_1^\top U_2\Lambda_2\rho_{\kappa_2}(\Lambda_2), \quad \text{where} \quad \rho_\kappa(x) = \frac{1}{\sqrt{x^2/n+\kappa}}.
\end{equation*}
Hence, we only need to compute the leading eigenvector (say, $z^*$) for:
\begin{equation*}
\begin{pmatrix}
    0 & M_*\\ 
    M_*^\top & 0
    \end{pmatrix}
\end{equation*}
It implies that the leading eigenvector for $B^{-1/2}AB^{-1/2}$ is then $\mathcal{U}z^*$. For the initial generalized eigenvalue problem, an approximate solution (due to using low-rank approximations of kernel matrices) is given by:
\begin{equation*}
    B^{-1/2}\mathcal{U}z^* = \begin{pmatrix}
        U_1  \rho_{\kappa_1}(\Lambda_1)\Lambda_1^{-1} z^*_1\\
        U_2  \rho_{\kappa_2}(\Lambda_2)\Lambda_2^{-1} z^*_2\\
    \end{pmatrix}= \begin{pmatrix}
        U_1  \Lambda_1^{-1}\rho_{\kappa_1}(\Lambda_1) z^*_1\\
        U_2  \Lambda_2^{-1}\rho_{\kappa_2}(\Lambda_2) z^*_2\\
    \end{pmatrix}=:\begin{pmatrix}
        \alpha_0\\
        \beta_0
    \end{pmatrix}.
\end{equation*}
Next, we need to normalize this vector of coefficients appropriately, i.e., we need to guarantee that $\norm{2}{\tilde{K}^{1/2} \alpha} = 1$ and $\norm{2}{\tilde{L}^{1/2} \beta} = 1$, and thus re-normalizing naively is quadratic in $n$. Instead, note that in order to compute incomplete Cholesky decomposition, we choose a tolerance parameter $\delta$ so that: $\norm{*}{PKP^\top-G_1G_1^\top}=\norm{*}{K-G_1G_1^\top}\leq \delta$ (nuclear norm). Let $\Delta=K-G_1G_1^\top$. We know that:
\begin{equation*}
     \alpha^\top \tilde{K} \alpha =  \alpha^\top HK H \alpha = \alpha^\top H (\Delta+G_1G_1^\top) H \alpha = \alpha^\top H \Delta H \alpha+\alpha^\top H G_1G_1^\top H \alpha
\end{equation*}
First, note that $\alpha^\top H \Delta H \alpha\leq \delta \|H\alpha\|_2^2$. Next,
\begin{equation*}
    G_1^\top H = V_1 \Lambda_1 U_1^\top.
\end{equation*}
Given an initial vector of parameters $\alpha_0$ and $\beta_0$, vectors of coefficients can be normalized in linear time using
\begin{equation*}
\begin{aligned}
    \alpha &= \frac{\alpha_0}{\sqrt{\norm{2}{G_1^\top H\alpha_0}^2+\delta\norm{2}{H\alpha_0}^2}}=\frac{U_1  \rho_{\kappa_1}(\Lambda_1)\Lambda_1^{-1} z^*_1}{\sqrt{\norm{2}{V_1  \rho_{\kappa_1}(\Lambda_1) z^*_1}^2+\delta\norm{2}{H\alpha_0}^2}}=\frac{U_1  \rho_{\kappa_1}(\Lambda_1)\Lambda_1^{-1} z^*_1}{\sqrt{\norm{2}{\rho_{\kappa_1}(\Lambda_1)z_1^*}^2+\delta\norm{2}{H\alpha_0}^2}}, \\ 
    \beta &= \frac{\beta_0}{\sqrt{\norm{2}{G_2^\top H\beta_0}^2+\delta\norm{2}{H\beta_0}^2}}=\frac{U_2  \rho_{\kappa_2}(\Lambda_2)\Lambda_2^{-1} z^*_2}{\sqrt{\norm{2}{V_2 \rho_{\kappa_2}(\Lambda_2)z_2^*}^2+\delta\norm{2}{H\beta_0}^2}}=\frac{U_2  \rho_{\kappa_2}(\Lambda_2)\Lambda_2^{-1} z^*_2}{\sqrt{\norm{2}{ \rho_{\kappa_2}(\Lambda_2)z_2^*}^2+\delta\norm{2}{H\beta_0}^2}}.
\end{aligned}
\end{equation*}

\subsection{Linear-time Updates of the HSIC Payoff Function}\label{appsubsec:hsic_updates}

Suppose that we want to evaluate HSIC payoff function on the next pair of points $(X_{2t+1},Y_{2t+1}),(X_{2t+2},Y_{2t+2})$. In order to do so, we need to compute: $\hat{g}_t(X_{2t+2},Y_{2t+2})$. It is clear that the computational of evaluating $\hat{\mu}_{XY} (x,y)$ and $(\hat{\mu}_X\otimes \hat{\mu}_Y) (x,y)$ on a given pair $(x,y)$ is linear in $t$. However, we also need to compute the normalization constant:
\begin{equation}\label{eq:norm_const}
    \norm{\calG\otimes\calH}{\hat{\mu}_{XY}-\hat{\mu}_X\otimes \hat{\mu}_Y}.
\end{equation}
Recall that:
\begin{equation*}
    \norm{\calG\otimes\calH}{\hat{\mu}_{XY}^{(t)}-\hat{\mu}_X^{(t)}\otimes \hat{\mu}_Y^{(t)}}^2 = \frac{1}{(2t)^2}\tr{K^{(t)}H^{(t)}L^{(t)}H^{(t)}},
\end{equation*}
where $K^{(t)}$ and $L^{(t)}$ are kernel matrices corresponding to the first $2t$ pairs, $H^{(t)} := \textbf{I}_{2t}-\frac{1}{2t}\textbf{1}_{2t} \textbf{1}_{2t}^\top$. Instead of computing the normalization constant naively, we next establish a more efficient way of computing~\eqref{eq:norm_const} in time linear in $t$ by caching certain values. Introduce:
    \begin{equation*}
        \begin{aligned}
            \Delta^{(t)}_1 &= \sum_{i,j = 1}^{2t} K_{ij}L_{ij} = \tr{K^{(t)}L^{(t)}},\\
            \Delta^{(t)}_2 &= \sum_{i,j}^{2t} K_{ij} = \textbf{1}_{2t}^\top K^{(t)} \textbf{1}_{2t},\\
            \Delta^{(t)}_3 &= \sum_{i,j}^{2t} L_{ij} = \textbf{1}_{2t}^\top L^{(t)} \textbf{1}_{2t},\\
            \Delta^{(t)}_4 &= \sum_{i=1}^{2t} \sum_{j,q=1}^{2t} K_{ij}L_{iq} = \textbf{1}_{2t}^\top K^{(t)}L^{(t)} \textbf{1}_{2t}.
        \end{aligned}
    \end{equation*}
We have:
\begin{equation*}
    \norm{\calG\otimes\calH}{\hat{\mu}_{XY}^{(t+1)}-\hat{\mu}_X^{(t+1)}\otimes \hat{\mu}_Y^{(t+1)}}^2 = \frac{1}{(2t+2)^2}\Delta^{(t+1)}_1 + \frac{1}{(2t+2)^4}\Delta^{(t+1)}_2\cdot\Delta^{(t)}_3-\frac{2}{(2t+2)^3}\Delta^{(t+1)}_4.
\end{equation*}
Next, we show how to speed up computations via caching certain intermediate values. Kernel matrices have the following structure:
\begin{equation*}
    K^{(t+1)} = \begin{pmatrix} K^{(t)} & K_{\cdot,2t+1} & K_{\cdot,2t+2} \\
    K_{\cdot,2t+1}^\top & K_{2t+1,2t+1} & K_{2t+1,2t+2}\\
    K_{\cdot,2t+2}^\top &  K_{2t+2,2t+1} & K_{2t+2,2t+2}
    \end{pmatrix}, \quad L^{(t+1)} = \begin{pmatrix} L^{(t)} & L_{\cdot,2t+1} & L_{\cdot,2t+2} \\
    L_{\cdot,2t+1}^\top & L_{2t+1,2t+1} & L_{2t+1,2t+2}\\
    L_{\cdot,2t+2}^\top &  L_{2t+2,2t+1} & L_{2t+2,2t+2}
    \end{pmatrix},
\end{equation*}
where $K_{\cdot,2t+1},K_{\cdot,2t+2},L_{\cdot,2t+1},L_{\cdot,2t+2}\in \Real^{2t}$ contain kernel function evaluations: 
\begin{equation*}
    K_{\cdot,m} = \begin{pmatrix}
    k(X_1,X_{m})\\
    \vdots\\
    k(X_{2t},X_{m})
    \end{pmatrix}, \quad L_{\cdot,m} = \begin{pmatrix}
    l(Y_1,Y_{m})\\
    \vdots\\
    l(Y_{2t},Y_{m})
    \end{pmatrix}, \quad m\in\curlybrack{2t+1,2t+2}.
\end{equation*}
First, it is easy to derive that:
\begin{equation*}
\begin{aligned}
    \tr{K^{(t+1)}L^{(t+1)}} &= \tr{K^{(t)} L^{(t)}} + 2(L_{\cdot,2t+1}^\top K_{\cdot,2t+1})+ 2(L_{\cdot,2t+2}^\top K_{\cdot,2t+2})+ \\
    & + K_{2t+1,2t+1}L_{2t+1,2t+1}+K_{2t+2,2t+2}L_{2t+2,2t+2}\\
    &+K_{2t+1,2t+2} L_{2t+2,2t+1} + K_{2t+2,2t+1}L_{2t+1,2t+2}.
\end{aligned}
\end{equation*}
Thus, if the value $\tr{K^{(t)} L^{(t)}}$ is cached, then $\tr{K^{(t+1)} L^{(t+1)}}$ can be computed in linear time. Note that:
\begin{equation*}
    \begin{aligned}
        K^{(t+1)} \textbf{1}_{2t+2} = \begin{pmatrix}
        K^{(t)} \textbf{1}_{2t} +  k_{\cdot,2t+1} + k_{\cdot,2t+2} \\
        K_{\cdot,2t+1}^\top \textbf{1}_{2t} + K_{2t+1,2t+1}+ K_{2t+1,2t+2}\\
        K_{\cdot,2t+2}^\top \textbf{1}_{2t} + K_{2t+2,2t+1}+ K_{2t+2,2t+2}
        \end{pmatrix},
    \end{aligned}
\end{equation*}
which can be computed in linear time if $K^{(t)} \textbf{1}_{2t}$ is stored (similar result holds for $L^{(t+1)} \textbf{1}_{2t+2}$). It thus follows that $\textbf{1}_{2t+2}^\top K^{(t+1)} \textbf{1}_{2t+2}$, $\textbf{1}_{2t+2}^\top L^{(t+1)} \textbf{1}_{2t+2}$ and $\textbf{1}_{2t+2}^\top K^{(t+1)} L^{(t+1)} \textbf{1}_{2t+2}$ can all be computed in linear time. To sum up, we need to cache $\tr{K^{(t)}L^{(t)}}$, $K^{(t)} \textbf{1}_{2t}$, $L^{(t)} \textbf{1}_{2t}$ to compute the normalization constant in linear time.

\end{document}

%% file: macro.tex
\usepackage{amsmath} 
\usepackage{amsthm}
\usepackage{bbm}
\usepackage{xspace}
\usepackage{caption}
\usepackage{amssymb}
\usepackage{mathtools}
\usepackage{enumitem}
\usepackage{algorithm}
\usepackage[noend]{algpseudocode}
\usepackage{graphicx}
\usepackage{xcolor}
\usepackage{nicefrac}
\usepackage{array}
\usepackage[toc,page]{appendix}
\usepackage{titlesec}
\usepackage{hyperref}  
\usepackage{accents}
\usepackage{float}
\usepackage{tcolorbox}

\newtheorem{theorem}{Theorem}
\newtheorem{proposition}{Proposition}

\newtheorem{lemma}[theorem]{Lemma}
\theoremstyle{definition}

\theoremstyle{remark}
\newtheorem{remark}{Remark}
\theoremstyle{definition}
\newtheorem{assumption}{Assumption}
\theoremstyle{definition}
\newtheorem{example}{Example}

\allowdisplaybreaks

\newcommand{\simiid}{\overset{\mathrm{iid}}{\sim}}
\newcommand{\convas}{\overset{\mathrm{a.s.}}{\to}}

\newcommand{\HSIC}{\mathrm{HSIC}}

\newcommand{\calM}{\mathcal{M}}
\newcommand{\calS}{\mathcal{S}}
\newcommand{\calZ}{\mathcal{Z}}
\newcommand{\calX}{\mathcal{X}}
\newcommand{\calY}{\mathcal{Y}}
\newcommand{\calF}{\mathcal{F}}
\newcommand{\calG}{\mathcal{G}}

\newcommand{\calC}{\mathcal{C}}
\newcommand{\calD}{\mathcal{D}}
\newcommand{\calK}{\mathcal{K}}

\newcommand{\calH}{\mathcal{H}}

\newcommand{\calU}{\mathcal{U}}

\newcommand{\quant}{\mathbf{Q}}

\newcommand{\tr}[1]{\text{tr}\roundbrack{#1}}

\newcommand{\abs}[1]{\left\lvert#1\right\rvert}
\newcommand{\norm}[2]{\left\lVert#2\right\rVert_{#1}}

\DeclareMathOperator*{\argmax}{arg\,max}
\newcommand{\Real}{\mathbb{R}}

\newcommand{\indicator}[1]{\mathbbm{1}\curlybrack{#1}}

\newcommand{\roundbrack}[1]{\left( #1 \right)}
\newcommand{\curlybrack}[1]{\left\lbrace #1 \right\rbrace}
\newcommand{\squarebrack}[1]{\left\lbrack #1 \right\rbrack}
\newcommand{\anglebrack}[1]{\left\langle #1 \right\rangle}

\newcommand{\Prob}{\mathbb{P}}

\newcommand{\Exp}[2]{\mathbb{E}_{#1}\left\lbrack#2\right\rbrack}
\newcommand{\Var}[2]{\mathbb{V}_{#1}\left\lbrack#2\right\rbrack}
\newcommand{\EmpVar}[2]{\hat{\mathbb{V}}_{#1}\left\lbrack#2\right\rbrack}
\newcommand{\Cov}[1]{\mathrm{Cov}\left(#1\right)}

\newcommand{\indep}{\perp \!\!\! \perp}

\newcommand{\sign}[1]{\mathrm{sign}\left(#1\right)}

\definecolor{myred}{RGB}{160, 50, 50}
\definecolor{myblue}{RGB}{50, 100, 200}
\definecolor{mygreen}{RGB}{40, 180, 40}